\documentclass{article}

\usepackage{arxiv}

\usepackage[utf8]{inputenc} 
\usepackage[T1]{fontenc}    
\usepackage{environ}
\usepackage{bbold}
\usepackage[normalem]{ulem}
\usepackage{enumerate}
\usepackage{array, makecell}
\usepackage{array,multirow} 

\usepackage[usenames,dvipsnames]{xcolor}
\definecolor{Gred}{RGB}{219, 50, 54}
\definecolor{Ggreen}{RGB}{60, 186, 84}
\definecolor{Gblue}{RGB}{72, 133, 237}
\definecolor{Gyellow}{RGB}{247, 178, 16}
\definecolor{ToCgreen}{RGB}{0, 128, 0}
\definecolor{myGold}{RGB}{231,141,20}
\definecolor{myBlue}{rgb}{0.19,0.41,.65}
\definecolor{myPurple}{RGB}{175,0,124}
\definecolor{niceRed}{RGB}{153,0,0}
\definecolor{niceRed}{RGB}{190,38,38}
\definecolor{blueGrotto}{HTML}{059DC0}
\definecolor{royalBlue}{HTML}{057DCD}
\definecolor{navyBlueP}{HTML}{0B579C}
\definecolor{limeGreen}{HTML}{81B622}
\definecolor{nicePink}{RGB}{247,83,148}

\usepackage{amsmath}
\usepackage{amssymb}
\usepackage{mathtools}
\usepackage{amsthm}

\title{Estimating Ising Models in Total Variation Distance}

\author{ Constantinos Daskalakis\\
	EECS Department\\
	MIT\\
	\texttt{costis@csail.mit.edu}
    \And
    Vardis Kandiros \\
    Data Science Institute\\
	Columbia University\\
	\texttt{ak5484@columbia.edu} \\
	\And
	Rui Yao\\
    EECS Department\\
	MIT\\
	\texttt{rayyao@mit.edu}
}

\def\compactify{\itemsep=0pt \topsep=0pt \partopsep=0pt \parsep=0pt}
\let\latexusecounter=\usecounter

\definecolor{myC}{rgb}{0, 255, 255}
\definecolor{myY}{rgb}{204, 204, 0}
\definecolor{myM}{rgb}{255, 0, 255}
\definecolor{secinhead}{RGB}{249,196,95}
\definecolor{lgray}{gray}{0.8}

\newtheorem{theorem}{Theorem} 
\newtheorem*{theorem*}{Theorem} 
\newtheorem*{proposition*}{Proposition} 
\newtheorem{lemma}[theorem]{Lemma}

\newtheorem{corollary}[theorem]{Corollary}

\newtheorem{definition}[theorem]{Definition}
\newtheorem{remark}[theorem]{Remark}

\newtheorem{assumption}[theorem]{Assumption}

\renewcommand{\Pr}{\mathop{\bf Pr\/}}
\newcommand{\E}{\mathop{\bf E\/}}

\newcommand{\Var}{\mathop{\bf Var\/}}

\newcommand{\calN}{\mathcal{N}}

\def\<{\langle}
\def\>{\rangle}

\DeclareMathOperator*{\argmax}{argmax}

\def\dtv{\mathrm{TV}}

\usepackage{amsthm,amsmath,amssymb,amsfonts,amssymb,mathtools}
\usepackage{xcolor}
\usepackage{empheq}
\usepackage{dsfont}
\usepackage{mathrsfs}
\usepackage{graphicx}
\usepackage{enumitem}
\usepackage{subfig}

\usepackage{tikz}
\usetikzlibrary{patterns}
\usetikzlibrary{arrows,shapes,automata,backgrounds,petri,positioning}
\usetikzlibrary{shadows}
\usetikzlibrary{calc}
\usetikzlibrary{spy}
\usepackage{pgf,pgfplots}
\usetikzlibrary{angles, quotes}
\usepackage{pgfmath,pgffor}
\pgfplotsset{compat=1.17}
\usepackage{framed}

\definecolor[named]{ACMBlue}{cmyk}{1,0.1,0,0.1}
\definecolor[named]{ACMYellow}{cmyk}{0,0.16,1,0}
\definecolor[named]{ACMOrange}{cmyk}{0,0.42,1,0.01}
\definecolor[named]{ACMRed}{cmyk}{0,0.90,0.86,0}
\definecolor[named]{ACMLightBlue}{cmyk}{0.49,0.01,0,0}
\definecolor[named]{ACMGreen}{cmyk}{0.20,0,1,0.19}
\definecolor[named]{ACMPurple}{cmyk}{0.55,1,0,0.15}
\definecolor[named]{ACMDarkBlue}{cmyk}{1,0.58,0,0.21}

\DeclareMathOperator{\sech}{sech}

\newcommand{\lp}{\left}
\newcommand{\rp}{\right}

\def\dkl{\mathrm{KL}}

\newcommand{\R}{\mathbb{R}}

\newcommand{\lr}[1]{\mathopen{}\left(#1\right)}

\newcommand{\lrnorm}[1]{\mathopen{}\left\|#1\right\|}

\renewcommand{\shorttitle}{Estimating Ising Models in Total Variation Distance}

\newcommand{\mat}[1]{\boldsymbol{#1}}
\newcommand{\indicator}[1]{\mathbf{1}\{#1\}}
\newcommand{\flip}[2]{\overline{#1^{(#2)}}}
\newcommand{\Ent}[1]{\text{Ent}_\mu\lp(#1\rp)}

\usepackage{cmap}
\usepackage[T1]{fontenc}
\usepackage{bm}

\usepackage{amsmath,amsfonts,amssymb,amsbsy,amsthm}
\usepackage{bbm}
\usepackage{thm-restate}

\usepackage[pagebackref]{hyperref}
\hypersetup{
	colorlinks = true,
	urlcolor = {royalBlue},
	linkcolor = {niceRed},
	citecolor = {royalBlue}
}
\usepackage[nameinlink]{cleveref}

\usepackage{multirow}
\usepackage{array}

\begin{document}

\maketitle

\begin{abstract}
	We consider the problem of estimating Ising models over $n$ variables in Total Variation (TV) distance, given $l$ independent samples from the model. While the statistical complexity of the problem is well-understood~\cite{devroye2020minimax}, identifying computationally and statistically efficient algorithms has been challenging. In particular, remarkable progress has occurred in several settings, such as when the underlying graph is a tree \cite{daskalakis2021sample,bhattacharyya2021near}, when the entries of the interaction matrix follow a Gaussian distribution~\cite{gaitonde2024unified,chandrasekaran2024learning}, or when the bulk of its eigenvalues lie in a small interval \cite{anari2024universality,koehler2024efficiently}, but no unified framework for polynomial-time estimation in TV exists so far.

    Our main contribution is a unified analysis of the Maximum Pseudo-Likelihood Estimator (MPLE) for two general classes of Ising models. The first class includes models that have bounded operator norm and satisfy the Modified Log-Sobolev Inequality (MLSI), a functional inequality that was introduced to study the convergence of the associated Glauber dynamics to stationarity. In the second class of models, the interaction matrix has bounded infinity norm (or \emph{bounded width}), which is the most common assumption in the literature for structure learning of Ising models. We show how our general results for these classes yield polynomial-time algorithms and optimal or near-optimal sample complexity guarantees in a variety of settings. Our proofs employ a variety of tools from tensorization inequalities to measure decompositions and concentration bounds.
\end{abstract}

\section{Introduction}

Undirected graphical models are a widely used framework for capturing conditional independence structure in a high-dimensional distribution. One of the earliest and most prominent instances of these models are {\em Ising models}, a family of distributions over $n$ binary variables specified by a symmetric \emph{interaction matrix} $J^* \in \R^{n \times n}$, with zero diagonal, and a vector of \emph{external fields} $h\in \R^n$. In terms of these parameters, a probability distribution is defined over $\{-1,1\}^n$, assigning to each vector $x\in \{-1,1\}^n$ probability
\begin{equation}\label{eq:ising_model}
\Pr_{J^*,h}[x] = \frac{1}{Z_{J^*,h}} \cdot \exp\lp(\frac{1}{2} x^\top J^* x + h^\top x \rp) \enspace,
\end{equation}
where the normalizing constant
\[
Z_{J^*,h} := \sum_{x \in \{-1,1\}^n} \exp\lp(\frac{1}{2} x^\top J^* x + h^\top x \rp)
\]
is called the \emph{partition function} of the model. We will focus on the case $h=0$, so we drop the dependence on $h$ whenever that happens. The matrix $J^*$ can also be thought of as the weighted adjacency matrix of a graph with $n$ nodes. This gives rise to the interpretation of an Ising model as a \emph{Markov Random Field (MRF)} \cite{koller2009probabilistic,wainwright2008graphical}, where conditional independence relations between variables are encoded as connectivity properties between nodes in the graph defined by $J^*$. 
Since its introduction in statistical physics~\cite{lenz1920beitrvsge}, the Ising model has been studied intensely in probability theory, and has found profound applications in a variety of fields, including computer vision, economics, computational biology, and the social sciences; see e.g.~\cite{geman1986markov,ellison1993learning,felsenstein2004,daskalakis2006optimal,montanari2010spread}. Motivated in part by these applications, a common challenge is estimating an Ising model given a number of observations that are assumed to be distributed according to some Ising model, e.g.~capturing opinions of individuals in a social network, or traits of species in some phylogenetic tree.

The problem of learning the structure of the underlying graph, i.e.~the non-zero entries of $J^*$, given access to $l$ independent samples from an Ising model has received significant attention, due to its importance in capturing conditional independence properties. Under the assumption that the degree of the graph is bounded by $d$ and the non-zero entries of $J^*$ are both upper and lower bounded in absolute value, a breakthrough result by Bresler~\cite{bresler2015efficiently} provided a polynomial-time algorithm for identifying the graph topology, albeit with doubly exponential dependence of the sample-size requirements in the degree $d$.
A flurry of subsequent works~ \cite{hamilton2017information,vuffray2016interaction,lokhov2018optimal,klivans2017learning,wu2019sparse,zhang2020privately} have culminated in polynomial-time algorithms that estimate every entry of $J^*$ up to error $\epsilon$, given $l = \Theta(e^{\Theta(d)}\log n /\epsilon^4)$ samples from the model, under the assumption that each row of $J^*$ has infinity norm that is upper-bounded. This matches the information-theoretic lower bound of \cite{santhanam2012information}. Thus, the problem of estimating the graph structure of an Ising model is by now well understood under relatively general assumptions. 

However, often the purpose of estimation is to make predictions of events for various downstream uses of the model~\cite{bresler2020learning,daskalakis2021sample,bhattacharyya2021near}. It is clear that the right metric to capture this property is not graph similarity, but the \emph{total variation (TV) distance}  (formal definition in Section~\ref{sec:results}) between the true and the estimated distribution. Information-theoretically,  Devroye et al.~\cite{devroye2020minimax} show that $\Tilde{O}(n^2/\epsilon^2)$ samples from some Ising model are both necessary and sufficient for estimating an Ising model that is $\epsilon$-close in TV. The algorithm proposed in \cite{devroye2020minimax} involves exhaustive search over all models and is thus computationally infeasible.
This motivates the following natural question.

\begin{quote}
    \emph{Is there a \emph{polynomial-time} algorithm that uses independent samples from an Ising model and outputs some Ising model that is close to the one providing samples in TV distance?}
\end{quote}

While for the problem of structure recovery, there has been significant progress towards computationally efficient and statistically optimal algorithms under fairly general settings, attaining similar guarantees for the TV estimation problem has been challenging. 
Remarkable progress has occurred across several different directions, such as when the graph is a tree~\cite{daskalakis2021sample,bhattacharyya2021near}, when the interaction matrix is sampled from a Gaussian ensemble \cite{gaitonde2024unified,chandrasekaran2024learning}, or when most of the spectrum of the matrix lies in a small interval \cite{anari2024universality,koehler2024efficiently}, but no unified framework exists for statistically and computationally efficient procedures for this task.
The main contribution of this work is to provide a refined understanding of a natural, polynomial-time algorithm for Ising model estimation, under broad conditions, and derive from our understanding near-optimal (and in some cases optimal) sample-complexity guarantees in a variety of important, specific settings. In particular, we focus on the so-called \emph{Maximum Pseudo-Likelihood Estimator (MPLE)} (formally defined in Section~\ref{sec:technical}), which was introduced in \cite{besag1974spatial} and is often employed for statistical estimation of autoregressive models. We study the performance of this polynomial-time computable estimator for two general classes of Ising models, which we review in the next couple of paragraphs.

{\em (i) Modified Log-Sobolev Condition:} One common way of obtaining information about the distribution of an Ising model is by studying the properties of an associated Markov Chain, called \emph{Glauber dynamics}, that has this model as its stationary distribution.
We say that an Ising model satisfies a \emph{Modified Log-Sobolev Inequality (MLSI)}, if the KL-divergence of any distribution from this model contracts at a constant rate after a single step of the Glauber dynamics \cite{bobkov2006modified} (for formal definitions, see Section~\ref{sec:preliminaries}).
This condition has been shown to hold for a variety of Ising models under different assumptions \cite{anari2021entropic,anari2024trickle,caputo2015approximate,chen2022localization,blanca2022mixing,chen2021optimal}. It is also known to imply many structural properties for the Ising model distribution, such as fast mixing \cite{caputo2023lecture} and concentration of measure \cite{marton2015logarithmic,sambale2019modified,gotze2021concentration}. 
We prove a general result about the performance of MPLE for estimating Ising Models that satisfy MLSI, which implies optimal or near-optimal sample complexity for learning a variety of Ising models, improving results from prior work.

{\em (ii) Bounded Width Condition:}  The second class of models we study are \emph{bounded-width} models, where $\|J^*\|_\infty = O(1)$. This has been the canonical class of models considered in most studies of the structure learning problem \cite{santhanam2012information,bresler2015efficiently,hamilton2017information,vuffray2016interaction,klivans2017learning,wu2019sparse}. For this class of models, \cite{klivans2017learning} give a polynomial-time algorithm that uses $O(n^8/\epsilon^4)$ samples and learns an $\epsilon$-multiplicative approximation of $\Pr_{J^*}$, which implies $\epsilon$-closeness in TV.
We provide a refined analysis of MPLE, which involves a convex objective that can be optimized efficiently. As a consequence, we show how one could obtain sample complexity guarantees within a single $O(n)$ factor from optimal, assuming the model satisfies a suitable regularity condition that we appropriately define.

Our improved general bounds can be applied in a variety of models, yielding comparable or superior sample complexity to that of prior work. For simplicity, we only discuss examples of models where we get improved complexity guarantees. Table~\ref{table:results} contains a detailed comparison.

At a technical level, our improvements are made possible by using a connection to the problem of \emph{single-sample} estimation of Ising models, which was formulated in \cite{dagan2021learning} and implicitly studied elsewhere before; for some references see~e.g.~\cite{chatterjee2007estimation,bhattacharya2018inference,daskalakis2019regression,dagan2021learning,mukherjee2021high,kandiros2021statistical}. This line of work assumes that we are only given a \emph{single} sample from an Ising model whose interaction matrix lies in some low-dimensional subspace, and our goal is to estimate this matrix. These works manage to extract a useful signal from a single sample from the model, even in the presence of strong dependencies within the sample. In contrast, algorithms that rely on multiple samples make strong use of the independence across different samples to employ generalization bounds from learning theory within the estimation procedure. In \cite{dagan2021learning}, the authors show how to leverage their single-sample estimation methods to obtain a polynomial-time algorithm for learning the structure of an Ising model in the multiple-sample regime. However, their sample complexity is far from optimal. In this paper, we improve upon this single-sample-based approach to obtain various state-of-the-art results in the multiple-sample regime. It is crucial for our development to establish tight upper and lower bounds for the first and second derivatives of the pseudolikelihood function, which we do by using a variety of tools from high-dimensional probability and statistics, such as tensorization inequalities, measure decompositions, and concentration of measure. Along the way, we provide refined guarantees for pseudolikelihood estimation, which could be of independent interest. We believe the connection between single-sample and multiple-sample learning could be more broadly applicable to a variety of estimation problems for Markov Random Fields. Thus, our work serves as a first step towards obtaining efficient and sample-optimal algorithms for TV learning of high-dimensional distributions with complex dependencies. 

\begin{table}[ht]
\centering
\renewcommand{\arraystretch}{1.3}
\begin{tabular}{
|>{\centering\arraybackslash}m{0.22\textwidth}|
 >{\centering\arraybackslash}m{0.22\textwidth}|
 >{\centering\arraybackslash}m{0.26\textwidth}|
 >{\centering\arraybackslash}m{0.22\textwidth}|}
\hline
\textbf{General Class of Models} & 
\textbf{Applications} & 
\textbf{Our Work} & 
\textbf{Prior Work} \\ \hline

\multirow{3}{*}{\shortstack{\textbf{MLSI +}\\\textbf{bounded $2$-norm}}}
 & Spectrally-bounded models 
 & \shortstack{$\Tilde{O}(n^2/\epsilon^2)$ \\(Corollary~\ref{thm_informal:bounded_spectral}, optimal)} 
 & \shortstack{$O(n^3/\epsilon^4)$} \\ \cline{2-4}

 & SK/diluted SK model ($\beta < 0.295...$) 
 & \shortstack{$\Tilde{O}(n^4/\epsilon^2)$ \\ (Corollaries~\ref{infcor:sk} and~\ref{infcor:dsk})} 
 & \shortstack{$\Tilde{O}(n^9/\epsilon^8)$} \\ \cline{2-4}

 & Antiferromagnetic expanders 
 & \shortstack{$\Tilde{O}(n^2/\epsilon^2)$\\(Corollary~\ref{cor_informal:antiffero}, optimal)} 
 & $\Tilde{O}(n^5/\epsilon^4)$\\ \hline

\textbf{Bounded-width} 
 & $(1/\sqrt{n},1)$-regular models 
 & \shortstack{$\Tilde{O}(n^3/\epsilon^2)$\\ (Corollary~\ref{cor:n3_samples_informal})} 
 & \shortstack{$\Tilde{O}(n^8/\epsilon^4)$} \\ \hline

\end{tabular}
\caption{This table contains the sample complexity bound implied by our work, as well as the best known bound from prior work, for the problem of estimating an Ising model to within $\epsilon > 0$ in TV. The best-known prior bound is discussed, where these results are stated.}
\label{table:results}
\end{table}

\section{Results}\label{sec:results}
Let $\mathcal{S}_0^n$ denote the set of all symmetric matrices $A \in \R^{n \times n}$ with zeroes on the diagonal. We will make use of the \emph{infinity norm} of a matrix $A \in \mathcal{S}_0^n$, which is defined as $\|A\|_\infty := \max_{i\in [n]} \sum_j |A_{ij}|$, as well as the \emph{operator norm} of $A$, defined as $\|A\|_2 := \sup_{x \neq 0}\|Ax\|_2/\|x\|_2$. We will occasionally denote this operator norm by $\|A\|_{op}$. Also, the \emph{Frobenius norm} is defined as $\|A\|_F := \sum_i \sum_j A_{ij}^2$. For $A \in \mathcal{S}_0^n$, we denote by $A_i$ the $i$-th row of $A$. For a vector $x \in \R^n$ and $i \in [n]$, we denote by $x_{-i}$ the sub-vector obtained by removing the value of coordinate $i$ from the vector. In general, we will use lowercase letters for deterministic quantities and uppercase letters for random quantities.
When we refer to an Ising model with interaction matrix $J$ and the external field $\pmb{h}$ is the zero vector, we will write $\Pr_{J}$ for convenience and simply omit the external field.  If we sample $X \sim \Pr_J$, we refer to $X_i \in \{-1,1\}$ as the \emph{spin} of node $i$. 
For two probability measures $\mathbb{P},\mathbb{Q}$ supported on $\{-1,1\}^n$, their \emph{Total Variation (TV) Distance} is defined to be
$
\dtv(\mathbb{P},\mathbb{Q}) := \sup_{A \subseteq \{-1,1\}^n} |\mathbb{P}(A) - \mathbb{Q}(A)|\enspace,
$
where $A$ ranges over all subsets of $\{-1,1\}^n$. 

In this work, we are given $l$ independent samples from an Ising model $\Pr_{J^*}$ as in \eqref{eq:ising_model}. Our goal will be to properly learn the model in distribution, i.e., to estimate some matrix $\hat{J} \in \mathcal{S}_0^n$ so that $\Pr_{J^*}$ and $\Pr_{\hat{J}}$ are close in TV. 
Additionally, we would like an algorithm that runs in polynomial time. We now state the results of this investigation for the two different classes of Ising models that we consider.

\subsection{Estimating Ising Models that satisfy MLSI}\label{sec:results_spectrally_bounded}

A considerable amount of work has focused on identifying classes of Ising models where sampling and inference from the model are computationally tractable.
In particular, it is now well understood that these tasks are easy when the Glauber dynamics associated with the model converge quickly to the stationary distribution.
Indeed, in that case, one could run the chain many times independently to produce ''approximate samples'' from the distribution and use these to evaluate relevant quantities such as marginal or conditional probabilities of the model.

Therefore, a large part of the literature is devoted to the study of the mixing time properties of the Glauber dynamics for various Markov Chains, resulting in a deep and beautiful theory. 
It has been shown that the Glauber dynamics converge exponentially fast to the stationary distribution in KL divergence if and only if the model satisfies a \emph{Modified Log-Sobolev Inequality (MLSI)}, a functional inequality that is weaker than the usual Log-Sobolev inequality in discrete spaces \cite{bobkov2006modified}. The MLSI has been established for a variety of Ising models under different constraints \cite{marton2015logarithmic,anari2021entropic,chen2022localization,blanca2022mixing,anari2024trickle}.

Our first main Theorem establishes estimation guarantees for Ising models that satisfy MLSI, when running the MPLE over some set of interaction matrices. Crucially, the only properties that this set needs to satisfy are that $J^*$ belongs to that set and that the matrices in the set have bounded operator norm, which is a mild requirement that holds for almost all classes of Ising models where MLSI has been shown. This flexibility with respect to the optimizing set enables us to obtain polynomial-time algorithms in various cases, particularly when the set is convex and admits efficient projections (see Section~\ref{sec:results_SK} for more discussion).
The estimation guarantees are phrased in terms of Frobenius norm closeness to the matrix $J^*$, but we will see in Sections~\ref{sec:results_spectrally_bounded} and~\ref{sec:results_SK} how these can be easily translated to bounds on the TV distance. Also, in the formal version of the theorem, we have included the case of a non-zero external field, which doesn't change the analysis in any significant way and is omitted here for simplicity. 

\begin{theorem}[informal, see Theorem \ref{thm:frobenius_estimation_high_temp}]\label{thm_informal:MLSI_general}
    Suppose we are given $l$ independent samples $X^{(1)},\ldots,X^{(l)} \sim \Pr_{J^*}$ and $\Pr_{J^*}$ satisfies MLSI. Let $\mathcal{R} \subseteq \mathcal{S}_0^n$
    be a subset of matrices such that $\sup_{A \in \mathcal{R}}\|A\|_2 = O(1)$ and $J^* \in \mathcal{R}$. Then, running MPLE with optimizing set $\mathcal{R}$ produces an estimate $\hat{J} \in \mathcal{R}$, such that with high probability over the choice of samples $\|J^* - \hat{J}\|_F \leq \epsilon$, as long as $l = \Tilde{\Omega}(n^2/\epsilon^2)$.
\end{theorem}

As far as we know, a result of this generality has not appeared in the literature. The most related prior result is Theorem 5.2 from \cite{anari2024universality}, where they obtain an elegant result for TV estimation using MPLE. However, their result assumes that the Ising model satisfies \emph{Approximate Tensorization of Entropy (ATE)}, which is a stronger functional inequality than MLSI. Moreover, they require that every matrix in the optimizing set of MPLE satisfies ATE, which places strong constraints on the choice of this set. Finally, they require the matrices to be of \emph{bounded-width}, which is a stronger assumption than bounded operator norm and results in the loss of additional polynomial factors. We provide more details about precise sample complexity guarantees in Section~\ref{sec:results_spectrally_bounded}.
We next present a variety of applications of the main result for TV learning in classes of Ising models that satisfy MLSI.

\subsubsection{Application: Estimating Spectrally-Bounded Ising Models in TV}

Perhaps the most widely-studied condition that enables computationally efficient sampling and inference in Ising models is \emph{Dobrushin's Uniqueness Condition} \cite{dobruschin1968description}, which asserts that $\|J\|_{\infty } < 1$, or equivalently $\sum_j |J_{ij}| < 1$ for all rows $i$.  This condition has been shown to imply a variety of structural properties for the Ising measure, such as fast mixing \cite{levin2017markov}, correlation decay \cite{kunsch1982decay}, and concentration inequalities \cite{gotze2021concentration,adamczak2019note,marton2015logarithmic}.

Unfortunately, this condition is sometimes too strict and does not capture the tractable regime of an Ising model. A notable example is the celebrated \emph{Sherrington-Kirkpatric (SK)} model, where $J^*$ is a random matrix with each $J^*_{ij}$ sampled independently from $\mathcal{N}(0,\beta^2/n)$, where $\beta>0$ is a parameter called the \emph{inverse temperature} of the model. It is straightforward to observe that the expected $l_1$-norm of each column is $\Theta(\beta \sqrt{n})$, thus Dobrushin's condition is only satisfied if $\beta = O(1/\sqrt{n})$. However, it is expected that the model exhibits weak interactions for all sufficiently small constant $\beta$. 

Motivated by this gap, \cite{eldan2022spectral} introduced an alternative condition for fast mixing. In particular, we say that an Ising model as in \eqref{eq:ising_model} is \emph{spectrally bounded} if $\lambda_{max}(J^*) - \lambda_{min}(J^*) < 1$(note that $J^*$ is symmetric, hence diagonalizable). In \cite{anari2021entropic}, they prove that if a model is spectrally bounded, then MLSI holds and the Glauber dynamics mix in polynomial time. 
Thus, we can apply Theorem~\ref{thm_informal:MLSI_general} for this class of Ising models, which results in information-theoretically optimal sample complexity $\tilde{O}(n^2/\epsilon^2)$ for estimating spectrally bounded Ising models in TV.

\begin{corollary}[informal, see Corollary~\ref{cor_informal:Jop<1}]\label{thm_informal:bounded_spectral}
    Suppose we are given $l$ independent samples $X^{(1)},\ldots,X^{(l)} \sim \Pr_{J^*}$, where $\lambda_{max}(J^*) - \lambda_{min}(J^*) < 1 -\alpha$ and $J^*$ has zero-diagonal. Then, there is a polynomial time algorithm (MPLE) that produces an estimate $\hat{J} \in \mathcal{S}_0^n$, such that with high probability over the choice of samples we have $\dtv(\Pr_{\hat{J}},\Pr_{J^*}) \leq \epsilon$, as long as $l = \tilde{\Omega}(n^2/\epsilon^2)$.
\end{corollary}

The implicit constant in the sample complexity contains a factor of the form $\exp(1/\alpha^2)$ and additional sub-polynomial factors of the form $e^{\sqrt{\log n}}$.
As far as we know, the most relevant prior work is \cite{anari2024universality}, where they prove that MPLE succeeds in finding a model that is $\epsilon$-close to the true Ising model $\Pr_{J^*}$ using $O(n^{3 + C}/\epsilon^4)$ samples for some $C<1$, by establishing ATE with an inverse polynomial constant.
Subsequent work \cite{lee2023parallelising} has shown that, in fact, ATE holds with a $\Theta(1)$ constant in this setting, which can be used to remove the $C$ from the exponent, yielding $O(n^3/\epsilon^4)$ sample complexity.
Our result thus improves over this bound, by showing that the MPLE actually achieves the information theoretically optimal sample complexity $O(n^2)$ for estimating Ising models in TV \cite{devroye2020minimax}.
We should remark, though, that the implicit constant in the sample complexity of Corollary~\ref{thm_informal:bounded_spectral} contains a factor that is exponential in $1/\alpha^2$, while the bound in \cite{anari2021entropic} is free of such dependence. 
As noted above, spectrally bounded models do not necessarily have bounded width (see Section~\ref{sec:results_bounded_width} for definition) e.g. for the SK model, $\|J\|_{\infty}$ could be $\Theta(\sqrt{n})$, so the prior work \cite{klivans2017learning} would give exponential sample complexity.

\subsubsection{Application: Estimating the SK-model in TV}\label{sec:results_SK}
As mentioned in Section~\ref{sec:results_spectrally_bounded}, the SK model is one of the canonical examples of a mean-field model in statistical physics, exhibiting fascinating phase transition phenomena that have been the subject of extensive study in probability theory \cite{panchenko2012sherrington,talagrand2010mean}. The relevant parameter is the inverse temperature $\beta > 0$. Standard results from random matrix theory imply that if $\beta < 1/4$, then with high probability the interaction matrix has spectrum inside an interval of size $<1$, which means the model is spectrally bounded. Thus, in this regime, Corollary~\ref{thm_informal:bounded_spectral} can be used to learn the model optimally with $\tilde{O}(n^2/\epsilon^2)$ samples. 

However, it turns out that efficiently learning the model in TV distance is possible for much larger values of $\beta$. In particular, in \cite{gaitonde2024unified}, the authors remarkably prove that a polynomial time algorithm introduced in \cite{wu2019sparse} actually estimates the SK model in TV as long as $\beta < 1$. In a subsequent work, \cite{chandrasekaran2024learning} shows that the same algorithm succeeds even when $\beta = O(\sqrt{n})$, which extends well into the low-temperature region of the model. 

While these works greatly push the frontiers of efficient learnability, the sample complexity arising from these results is of the order of $\tilde{O}(n^9/\epsilon^8)$. 
In the next Corollary, we use Theorem~\ref{thm_informal:MLSI_general} and recently established MLSI in \cite{anari2024trickle} to obtain $\tilde{O}(n^4/\epsilon^2)$ sample complexity for learning the SK-model up to $\beta \approx 0.295$, which is beyond the threshold of spectrally-bounded models. 

\begin{corollary}[informal, see Corollary \ref{cor_informal:sk}]\label{infcor:sk}
    Suppose we are given $l$ independent samples $X^{(1)},\ldots,X^{(l)} \sim \Pr_{J^*}$, where $J^*$ is sampled according to the SK-model with $\beta < C$, where $C \approx 0.295$. Then, there is a polynomial time algorithm (MPLE) that produces an estimate $\hat{J} \in \mathcal{S}_0^n$, such that with high probability over the choice of samples and the choice of matrix $J^*$ we have $\dtv(\Pr_{\hat{J}},\Pr_{J^*}) \leq \epsilon$, as long as $l = \tilde{\Omega}(n^4/\epsilon^2)$.
\end{corollary}

Closely related to the SK-model are \emph{diluted} versions, where the matrix is supported on a sparse graph. One such version, which we call for simplicity the \emph{diluted SK-model}, arises from sampling a random d-regular graph $G$, where the matrix $J^*$ will be supported. Every non-zero entry of $J^*$ is then sampled independently and uniformly from $\{-\beta/\sqrt{d},\beta/\sqrt{d}\}$. Standard results from random matrix theory again imply that if $\beta < 0.25$, then the model is spectrally bounded with high probability and $\Tilde{O}(n^2/\epsilon^2)$ samples suffice by Corollary~\ref{thm_informal:bounded_spectral}. \cite{chandrasekaran2024learning} show that TV learning in polynomial time is possible if $\beta = O(\sqrt{\log n})$, with $\Tilde{O}(n^8 d/\epsilon^8)$ samples. We use the recently established MLSI for diluted SK up to $\beta \approx 0.295$ \cite{anari2024trickle} to establish that $\tilde{O}(n^4/\epsilon^2)$ samples suffice in that regime if we run MPLE.

\begin{corollary}[informal, see Corollary \ref{cor_informal:diluted_sk}]\label{infcor:dsk}
    Suppose we are given $l$ independent samples $X^{(1)},\ldots,X^{(l)} \sim \Pr_{J^*}$, where $J^*$ is sampled according to the diluted SK-model with $\beta < C$, where $C \approx 0.295$. Then, there is a polynomial time algorithm (MPLE) that produces an estimate $\hat{J} \in \mathcal{S}_0^n$, such that with high probability over the choice of samples and the choice of matrix $J^*$ we have $\dtv(\Pr_{\hat{J}},\Pr_{J^*}) \leq \epsilon$, as long as $l = \tilde{\Omega}(n^4/\epsilon^2)$.
\end{corollary}

Note that in this setting $\|J^*\|_\infty = O(\sqrt{d})$, so one could use the result of \cite{dagan2021learning} about learning in Frobenius norm together with Lemma~\ref{lem:TV_from_frobenius} that connects TV and Frobenius norms to prove that $O(e^{\Theta(\sqrt{d})}n^4)$ samples suffice for TV learning using MPLE. However, notice that even if the degree grows mildly with the number of nodes, i.e., $d = \omega(\log n)$, the sample complexity suffers from additional polynomial factors (or worse). In contrast, the sample complexity in Corollary~\ref{cor_informal:diluted_sk} only contains a sub-polynomial $\exp(\sqrt{\log n})$ factor, regardless of the value of $d$.

\subsubsection{Application: Antiferromagnetic Expanders}

Another prominent class of models is the ones where there is a gap between the largest and second-largest eigenvalue of the adjacency matrix. When the model is antiferromagnetic, then the spectrum essentially consists of a very negative eigenvalue and a bulk that is concentrated on a small interval.
Prior work \cite{anari2024trickle,koehler2022sampling} has shown that one can ``ignore'' this very negative eigenvalue and establish MLSI in this case.
Thus, if we can efficiently project on this set of matrices, then MPLE runs in polynomial time and has the optimal sample complexity. We show that this is indeed possible, which gives rise to the following Corollary.

\begin{corollary}[informal, see Corollary~\ref{cor_informal:negative_spike}]\label{cor_informal:antiffero}
    Let $\alpha \in (0,1), c > 0$ be constants and $\mathbf{1}$ the all-ones vector. Define the set $\mathcal{R}\subseteq \mathcal{S}_0^n$ of matrices that have $\mathbf{1}$ as an eigenvector with eigenvalue $-c$ and the rest of the spectrum is on an interval of size $\alpha$ around $0$. Suppose $J^* \in \mathcal{R}$. Then, given $l$ independent samples from $\Pr_{J^*}$, the MPLE over $\mathcal{R}$ can be implemented in polynomial time and returns $\hat{J}$ such that $\dtv(\Pr_{\hat{J}},\Pr_{J^*}) \leq \epsilon$ with high probability, as long as $l = \tilde{\Omega}(n^2/\epsilon^2)$.
\end{corollary}

For a canonical example in this set, consider the adjacency matrix $A_G$ of a random d-regular graph $G$ and take $J^* = -\beta A_G$. Then, from \cite{friedman2003proof} if follows that $J^*$ belongs in the set $\mathcal{R}$ with $c = \beta d$ and $\alpha = 4 \beta \sqrt{d-1}$, as long as we take $\beta < 1/(4\sqrt{d-1})$. Thus, we can learn this model in TV distance optimally and efficiently.
The most relevant prior work in this case is \cite{koehler2024efficiently}, which covers this class of models since it allows some eigenvalues to be very negative. However, the sample complexity is $\tilde{O}(n^4 R^2/\epsilon^4)$, where $R$ is the width of the model, which could be $\Theta(\sqrt{n})$ in that case (see Remark \ref{rmk:counter_for_bdd_w} for an example). 
Since $R = \Theta(\sqrt{n})$ in the worst case, the bounded width result of \cite{dagan2021learning} does not apply.

\subsection{Estimating Bounded-Width Ising Models in TV}\label{sec:results_bounded_width}

We say that an Ising model has \emph{bounded width}, if the interaction matrix is assumed to have infinity norm bounded by some constant $M>0$, i.e. $\|A\|_{\infty } \leq M$. Note that $M$ could be an arbitrary constant, which means the model could exhibit long-range correlations, Glauber dynamics might mix exponentially slowly (see e.g. \cite{mossel2009hardness}), and concentration of measure in general fails to hold. 
Our first contribution involves an improved analysis of the MPLE estimator, which results in the following guarantee for estimating the model $\Pr_{J^*}$.

\begin{theorem}[informal, see Theorem~\ref{thm:mple_refined}]\label{thm:mple_improved_informal}
     Suppose we are given $l$ independent samples $X^{(1)},\ldots,X^{(l)} \sim \Pr_{J^*}$, where $\|J^*\|_{\infty } \leq M$. Then, if $\hat{J}$ is the MPLE estimator\eqref{eq:pseudolikelihood_estimator}, with high probability over the choice of samples we have 
     \begin{equation}
         \E_{J^*}[\|(\hat{J}-J^*)X\|_2^2] \leq \epsilon,
     \end{equation}
     as long as $l = \tilde{\Omega}(n^2/\epsilon)$.
\end{theorem}
The implicit constant in the bound above contains a factor $\exp(M)$. 
The guarantee provided by Theorem~\ref{thm:mple_improved_informal} might seem non-standard, but we will see that it is well-suited for estimation in TV distance in the following Section. 

\subsubsection{Applications}\label{sec:bounded width applications}
First, as a direct corollary of Theorem~\ref{thm:mple_improved_informal} (proved in Section \ref{sec: condition stricter}), we can obtain the following.

\begin{corollary}[informal]\label{cor:frobenius_informal}
 Suppose we are in the setting of Theorem~\ref{thm:mple_improved_informal}. Then, with high probability over the choice of samples, we have 
     \begin{equation}\label{eq:improved_J_estimate}
         \|\hat{J} - J^*\|_F^2 \leq \epsilon,
     \end{equation}
     as long as $l = \tilde{\Omega}(n^2/\epsilon)$.
\end{corollary}

Corollary~\ref{cor:frobenius_informal} also appears as Corollary 6 in \cite{dagan2021learning}, hence we recover the previously established guarantees for learning in the Frobenius norm. Note that in general, if $J^*$ is in low temperature, $\E_{J^*}[\|(\hat{J}-J^*)X\|_2^2]$ could be significantly larger than $\|\hat{J} - J^*\|_F^2$ (we also give such examples in Section \ref{sec: condition stricter}), so Theorem~\ref{thm:mple_improved_informal} is a strict improvement over the result of \cite{dagan2021learning}. 

Now we are ready to state the implications of our results for learning in TV. First, we note that without imposing any additional assumptions, we can obtain a sample complexity of $\Tilde{O}(n^4)$ from the Frobenius norm approximation. The reason is that one can show that an $O(\epsilon)$ approximation in Frobenius norm implies an $O(n\epsilon)$ approximation in TV, using similar arguments to Theorem 7.3 in \cite{klivans2017learning}. For completeness, we give a self-contained proof of this fact in Section~\ref{sec:tvfrob}.
Thus, the following result follows from this connection together with Corollary~\ref{cor:frobenius_informal}.

\begin{corollary}[informal]\label{cor:n4_sample_informal}
Suppose we are in the setting of Theorem~\ref{thm:mple_improved_informal}. Then, if $l = \Omega(n^4/\epsilon^2)$, with high probability over the choice of sample $\dtv(\Pr_{\hat{J}}, \Pr_{J^*}) \leq \epsilon$. 
\end{corollary}

We now show how we can improve on the $O(n^4)$ sample complexity of Corollary~\ref{cor:n4_sample_informal} using the refined analysis of Theorem~\ref{thm:mple_improved_informal}.
To do that, we will assume that the second moments of the true model are ``robust'' to small perturbations of the matrix. Intuitively, we expect this to happen whenever $J^*$ is away from the critical temperature where a phase transition occurs. 
Formally, let us introduce the following regularity assumption. 

\begin{definition}\label{def:regularity}
    We say an Ising Model $\Pr_{J^*}$ satisfies $(\gamma,C)$-regularity for some $\gamma,C > 0$ if the following holds: for any $J \in \mathcal{S}_0^n$ such that $\E_{J^*}[\|(J-J^*)X\|_2^2] \leq \gamma$, we have $\E_{J}[\|(J-J^*)X\|_2^2] \leq C \cdot\E_{J^*}[\|(J-J^*)X\|_2^2]$.
\end{definition}

Of course, the crucial part of this definition is the scaling relation between $\gamma$ and $C>0$. We show as a Corollary of Theorem~\ref{thm:mple_improved_informal} that if a model is $(1/n,1)$-regular, then $O(n^3)$ samples suffice for TV learning.

\begin{corollary}[informal, see Corollary~\ref{cor:n3_samples}]\label{cor:n3_samples_informal}
    Suppose we are in the setting of Theorem~\ref{thm:mple_improved_informal}. Additionally, assume there exist constants $\gamma, C > 0$ such that $\Pr_{J^*}$ satisfies $(\gamma/n,C)$-regularity. Then, for any $\epsilon > 0$, if $l = \Omega(n^3/\epsilon^2)$, with high probability over the choice of samples $\dtv(\Pr_{\hat{J}}, \Pr_{J^*}) \leq \epsilon$.
\end{corollary}
Note that the sample complexity of Corollary~\ref{cor:n3_samples_informal} is only a factor $O(n)$ away from the optimal sample complexity of \cite{devroye2020minimax}.
This regularity condition covers a wide range of models that do not necessarily need to be in high temperature. In particular, in Section~\ref{sec: regularity verification}, we prove that the Curie-Weiss model in low temperature satisfies the condition. Of course, we can also show that models satisfying more familiar conditions such as Dobrushin's condition and spectrally-bounded models also satisfy this regularity condition (see Section \ref{sec: regularity verification}).

\section{Related Work}\label{sec:related_work}

\noindent \textbf{ Learning MRFs from multiple samples.} The problem of estimating a Markov Random Field (MRF) from multiple independent samples from the model has a rich history, starting from the seminal work \cite{chow1968approximating} from the 1960s, showing that for the Ising model, if the graph structure of the model is a tree, then Maximum Likelihood Estimation (MLE) can be solved in polynomial time.
Information theoretically, \cite{devroye2020minimax} establishes the minimax rate for estimating Ising models in TV as $\Theta(|E|/\epsilon^2)$, where $|E|$ is the number of non-zero entries of the interaction matrix (see also \cite{brustle2020multi} for an alternative argument using linear programming).
For the task of estimating the structure of Ising models with arbitrary graph topology and bounded degree $d$, the breakthrough work of \cite{bresler2015efficiently} provided the first polynomial time algorithm, where the sample complexity is doubly exponential in $d$. Subsequent works \cite{hamilton2017information,vuffray2016interaction,klivans2017learning,wu2019sparse} proposed new algorithms with improved guarantees. In particular, \cite{klivans2017learning} obtains the first polynomial time algorithm for learning the structure of bounded-width models, while only requiring $O(e^d \log n)$ independent samples using $l_2$-regularized per-node logistic regression. This matched the information-theoretic lower bound from \cite{santhanam2012information}. 
In the case of latent variables, \cite{bresler2019learning} gives a polynomial-time algorithm for learning ferromagnetic Restricted Boltzmann Machines. 
Beyond bounded width, a recent line of work studies Ising models under spectral constraints on the interaction matrix. \cite{anari2024universality} shows that for spectally bounded models, MPLE succeeds in TV learning with $O(n^{3 + C})$ samples for some constant $C$. For the SK model, this implies efficient learning for all $\beta < 1/4$. \cite{gaitonde2024unified} is the first work to obtain a polynomial time algorithm for learning the SK-model all the way up to $\beta < 1$, and \cite{chandrasekaran2024learning} extends the range of efficient learning for all $\beta = o(\sqrt{\log n})$. 

A related line of work studies the problem of learning the structure of MRFs using samples from the trajectory of Glauber dynamics \cite{bresler2017learning,gaitonde2024unified,gaitonde2025better}. The recent work of \cite{gaitonde2024bypassing} provides a near-linear time algorithm that learns the structure of a $t$-th order MRF using $O(n\log n)$ updates of Glauber dynamics, bypassing fundamental barriers for efficient higher order MRF estimation from independent samples. The work \cite{jayakumar2024discrete}, which studies learning from independent samples of a metastable state of the Glauber dynamics, is also close in spirit. 

The particular case of MRFs with tree structure has also received attention, starting with \cite{chow1968approximating}. \cite{bhattacharyya2021near,daskalakis2021sample} establish that the Chow-Liu algorithm is information-theoretically optimal for finite samples. For the related problem of estimating the low-dimensional marginals of the model, \cite{bresler2020learning,boix2022chow} give polynomial-time and sample-optimal algorithms, and \cite{nikolakakis2021predictive} studies the setting where noisy labels of the nodes are observed. Finally, \cite{kandiros2023learning} provides guarantees for polynomial-time estimation of latent tree Ising models in TV. 

Finally, a recent line of work aims at estimating Ising models using score matching \cite{koehler2022statistical,koehler2023sampling,koehler2024efficiently}. In particular, as noted in \cite{koehler2024efficiently}, they focus on a class of low-complexity Ising models. Remarkably, they obtain sample complexity bounds that scale polynomially with the width of the model, in contrast to the exponential dependence in most prior work on MRF estimation.\\

\noindent \textbf{Learning MRFs from a single sample.} In this line of work, it is usually assumed that the true model belongs in a class of ``low-dimensional'' models and the task is to estimate it given a \emph{single} $n$-dimensional observation from the model. In the case of an Ising model whose interaction matrix is known up to a scalar parameter $\beta$, \cite{chatterjee2007estimation} initially showed that MPLE is $\sqrt{n}$-consistent for $\beta$ using the technique of exchangeable pairs. \cite{bhattacharya2018inference} extended these results under general conditions on the log-partition function. \cite{ghosal2020joint} studied the problem when, in addition to $\beta$, there is also an unknown scalar parameter in the external field, and \cite{daskalakis2019regression} generalized the result for logistic regression with dependencies.
\cite{dagan2021learning} provided estimation guarantees when the interaction matrix lies in a low-dimensional subspace. Several variations of these settings have been studied, including optimal joint estimation of parameters for logistic regression with dependencies \cite{kandiros2021statistical,mukherjee2021high}, estimation of tensor Ising models \cite{daskalakis2020logistic,mukherjee2022estimation}, estimation of hard-constrained models \cite{bhattacharya2021parameter,galanis2024learning}, inference on dense graphs \cite{xu2023inference}.\\

\noindent \textbf{Sampling Ising models.} There is a vast literature in probability theory that focuses on proving fast sampling of Ising models under different constraints. Here, we focus on reviewing the results that are most relevant to the classes of Ising models that we study. Modified Log-Sobolev inequalities were introduced in \cite{bobkov2006modified} to prove fast mixing of Markov Chains in discrete spaces. The classical Dobrushin's condition has been known to imply MLSI \cite{levin2017markov} and is tight in the case of the Curie-Weiss model.
The class of spectrally bounded models, where $\lambda_{\max}(J) - \lambda_{\min}(J) < 1$, was introduced in \cite{eldan2022spectral} to capture relevant models from statistical physics, such as the SK-model. They established the Poincaré inequality when $\beta < 1/4$, complementing the result of \cite{bauerschmidt2019very} that proved a version of the log-Sobolev inequality in the same regime. Subsequently, \cite{anari2021entropic} established the MLSI for spectrally bounded models, which implies optimal mixing of the Glauber dynamics. A different proof using localization schemes was given in \cite{chen2022localization}. In the work \cite{lee2023parallelising}, they establish the stronger ATE property for these models. The condition of spectrally bounded models was shown to be tight for polynomial time sampling in \cite{kunisky2024optimality,galanis2024sampling}. A polynomial time algorithm for sampling from the SK model in Wasserstein distance using algorithmic stochastic localization was given for $\beta < 1/2$ in \cite{el2022sampling} and extended to all $\beta < 1$ in \cite{celentano2024sudakov}. The recent work of \cite{anari2024trickle} established an MLSI for the SK-model for $\beta < 0.295$. 
Finally, the class of bounded-width Ising models includes many examples where sampling from the model is NP-hard. Indeed, for $d$-regular graphs with $\beta > \beta_c$, where $\beta_c$ is the Kesten-Stigum threshold, \cite{sly2010computational,sly2012computational,galanis2015inapproximability} show that approximate sampling from the distribution is NP-hard.

\section{Technical Contributions}\label{sec:technical}
We first describe the algorithm that is employed for all results.
The most common approach for obtaining an estimate $\hat{J}$ is to compute the \emph{Maximum Likelihood Estimator (MLE)} given the samples. In the case of Ising models, to compute the MLE, one has to calculate the probability of the observed samples under different models. However, this involves computing the partition function of the model $Z_J$, which is NP-hard even to approximate in many interesting regimes \cite{sly2010computational,sly2012computational,galanis2015inapproximability,galanis2024sampling}.

An attractive alternative, first proposed in \cite{besag1974spatial}, is the so-called \emph{Maximum Pseudo-Likelihood Estimator (MPLE)}. Computing the Pseudo-Likelihood of a given model involves computing the conditional probability of the spin of a node $i$ conditioned on the spins of all the other nodes. Formally, suppose we get independent samples $X^{(1)},\ldots,X^{(l)} \sim \Pr_{J^*}$. Then, the MPLE over a set of matrices $\mathcal{R} \subseteq \mathcal{S}_0^n$ is defined as
\begin{equation}\label{eq:pseudolikelihood_estimator}
    \hat{J} := \arg\max_{J \in \mathcal{R}} PL(J;X^{(1)},\ldots,X^{(l)}) := \arg\max_{J \in \mathcal{R}} \prod_{k=1}^l \prod_{i=1}^n \Pr_J[X^{(k)}_i | X^{(k)}_{-i}]
\end{equation}
The set $\mathcal{R}$ will be chosen depending on the particular class of Ising models we are trying to estimate.
One useful property of this objective function is that it is a concave function of $J$, which makes it easy to optimize using first-order methods whenever $\mathcal{R}$ is a convex set that admits efficient projections. 
To make calculations more convenient, we will consider instead minimizing the negative log pseudolikelihood, which in the case of Ising models can be written as
\begin{equation}\label{eq:pseudolikelihood_function}
\phi(J) := - \log PL(J;X^{(1)},\ldots,X^{(l)}) = \sum_{k=1}^l \sum_{i=1}^n \lp(\log \cosh (J_i X^{(k)}) - X_i^{(k)} J_i X^{(k)} + \log 2\rp)
\end{equation}
In the above, we have omitted the dependence of $\phi$ on the samples $X^{(1)},\ldots,X^{(l)}$ for convenience.
Thus, the objective in MPLE has a simple closed form that does not involve the partition function, which is why it is preferred over MLE.
Since the optimization takes place in a high-dimensional space, we will be interested in computing the first and second derivatives of $\phi$ at a point $J \in \mathcal{S}_0^n$ and at direction $A \in \mathcal{S}_0^n$. These are given by the formulas (see also Section\eqref{sec:preliminaries})

\begin{equation}\label{eq:first_second_der}
\frac{\partial \phi(J^*)}{\phi A}=\sum_{k=1}^l\sum_{i=1}^n(A_iX^{(k)})(\tanh(J_iX^{(k)})-X^{(k)}_i) \quad, \quad \frac{\partial^2 \phi(J)}{\partial A^2} = \sum_{k=1}^l \sum_{i=1}^n (A_i X^{(k)})^2 \sech(J_iX^{(k)})^2
\end{equation}

The standard way of analyzing the MPLE in the single sample literature \cite{chatterjee2007estimation,bhattacharya2018inference,daskalakis2019regression,dagan2021learning} is to upper bound the first derivative at $J^*$ for all directions $A$ and to lower bound the second derivative for all $J$ and for all directions $A$. This, combined with a union bound argument, suffices for obtaining estimation guarantees for $\hat{J}$. Since these derivatives are random quantities that depend on the samples, in order to bound them, we introduce a number of technical novelties, as well as combine a variety of tools from high-dimensional probability.
We now highlight these contributions for the two settings that we study.

\subsection{Models Satisfying MLSI}\label{sec:technical_spectrally}
The work that is closest to the level of generality we are aiming for is \cite{anari2024universality}, where they show that the Approximate Tensorization of Entropy (ATE) inequality implies closeness in KL (and thus in TV by Pinsker's inequality) for the MPLE estimator.
ATE is a stronger functional inequality than MLSI. Furthermore, the argument results in a dependence of the sample complexity on the width of the model, which could incur additional polynomial factors in the regime where only the operator norm is bounded.
For these reasons, the sample complexity obtained in \cite{anari2024universality} is suboptimal, even in the special case of spectrally-bounded models, where improved ATE was later established by \cite{lee2023parallelising}.

We bypass these obstacles by instead showing that a Modified Log-Sobolev Inequality (MLSI) (formally defined in Section~\ref{sec:preliminaries}) and boundedness of the operator norm are enough to analyze the pseudolikelihood function $\phi$. 
The work that is closest in spirit to our goal in that regard is \cite{dagan2021learning}, where they analyze the MPLE when the infinity norm of the matrix is $O(1)$(bounded-width).
As we shall see, establishing tight upper and lower bounds for the first and second derivatives of the pseudolikelihood function $\phi$ when the infinity norm is unbounded poses additional challenges, which we now describe.

The first step towards that goal is to establish upper bounds for $\partial \phi(J^*)/ \partial A$ for all directions $A \in \mathcal{R}$, where $\mathcal{R}$ is the optimizing set. 
Notice by \eqref{eq:first_second_der} that this is not a polynomial function of the samples $X$. Despite that, in \cite{dagan2021learning}, the authors use second-level concentration inequalities from \cite{adamczak2019note} together with a linearization argument of the $\tanh$ to get the correct concentration radius of $O(l\cdot \|A\|_F^2)$ for the first derivative. Unfortunately, the bounds in \cite{adamczak2019note} rely on ATE. 
Instead, we show that MLSI suffices by using a general result about two-level concentration from \cite{sambale2019modified}. The concentration radius in \cite{sambale2019modified} is given with respect to the expected $l_2$-norm of a version of the discrete derivative of the function (for more details see Section~\ref{sec:preliminaries}).
In the proof, we show that even if we only know that the operator norm of the matrix $A$ is bounded, this concentration radius can be bounded by $O(l\cdot\|A\|_F^2)$. 
Since the expectation of the first derivative is $0$ when evaluated by $J^*$, this implies that with high probability $O|\partial \phi(J^*)/ \partial A|= O(l \cdot \|A\|_F^2)$. 

The next step is to lower bound the second derivative $\partial^2 \phi(J)/\partial A^2$ for all $J$ and all directions $A$.
Ignoring the term involving $\sech$ momentarily, which comes with additional challenges, the second derivative in \eqref{eq:first_second_der} is a second-degree polynomial of the Ising model. 
Using again the machinery from \cite{sambale2019modified}, we can show that this polynomial concentrates at an $O(l\cdot\|A\|_F^2)$ radius with high probability. 
To conclude that the second derivative is large, we need to establish that, on expectation, this polynomial is $\Omega(l \cdot\|A\|_F^2)$.
This lower bound would certainly hold if we had a product distribution instead of an Ising model. Motivated by that, we use the well-known Hubbard-Stratonovich transform to decompose the model into a mixture of product distributions with external fields. The distribution of the external fields in this mixture is well known, and we use it to lower bound the expectation of this polynomial for the majority of these product measures, which suffices to obtain the desired lower bound on the expectation.

Finally, if we wish to lower bound the second derivative, we have to lower bound the term involving $\sech$ in \eqref{eq:first_second_der}. 
If the infinity norm of the matrices was bounded, as in \cite{dagan2021learning}, then this step is trivial, as $|J_i X^{(k)}| = O(1)$ always. In contrast, with an operator norm bound, this term could be as large as $\Theta(\sqrt{n})$, which would result in an exponentially small lower bound for the second derivative in the worst case. One idea would be to use the concentration of linear functions of models satisfying MLSI, in order to bound $J_i X^{(k)}$ by $o(\log n)$, which holds with probability $1-1/\text{poly}(n)$. Unfortunately, since we have to establish this for all matrices $J$, the probability of error does not decay fast enough to take a union bound over this high-dimensional set of matrices. Our solution is to instead write
$J_i X^{(k)} = J^*_iX^{(k)} + (J_i - J^*_i) X^{(k)}$. Since $J^*$ is now fixed, we can use the preceding argument to bound $J^*_i X^{(k)}$ with high probability.
For $(J_i - J^*_i) X^{(k)}$, we can also use the preceding concentration bounds to bound it by $\|J-J^*\|_F$ with high probability. However, since we are considering an arbitrary matrix $J$, we cannot know \emph{a priori} that $\|J-J^*\|_F$ will be small. Intuitively, $J$ is any matrix we wish to show satisfies $\phi(J) > \phi(J^*)$.
To address this issue, we instead focus the analysis of the second derivative only on a \emph{ring} of matrices (Figure~\ref{fig:shell})
$$
 \mathcal{R}_\epsilon := \{J \in\mathcal{R}: \epsilon \leq \|J - J^*\|_F \leq 2 \epsilon\}
$$
\begin{figure}[h!]
  \centering
  \includegraphics[width=0.6\textwidth]{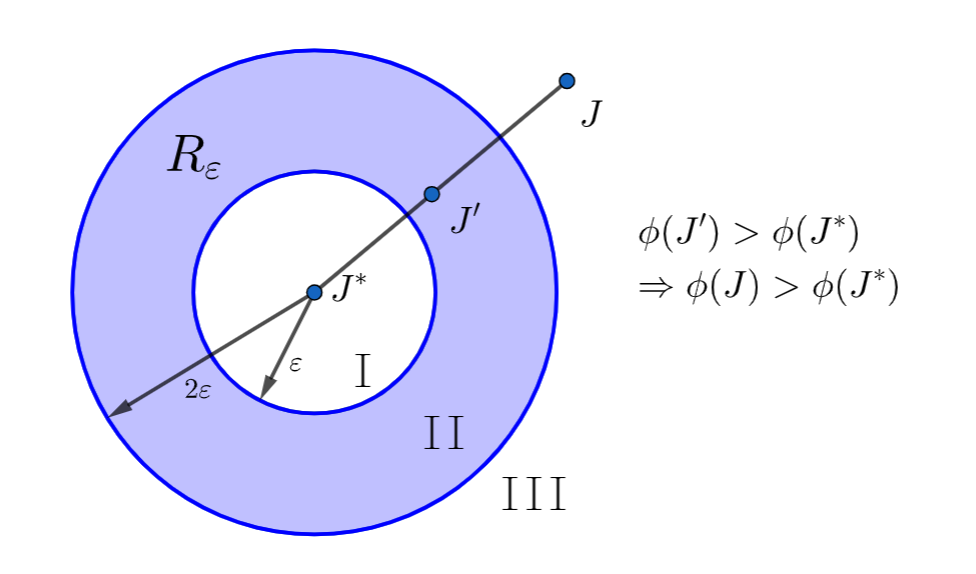}
  \caption{The set $\mathcal{R}_\epsilon$ naturally partitions the whole space into three regions, I, II, and III. We know that all $J$ in region II satisfy $\phi(J) > \phi(J^*)$. If we consider an arbitrary $J$ in region III, then by appropriate scaling, we can find $J'$ in region II that lies on the line connecting $J$ and $J^*$. The assumption about $J^*$ combined with the convexity of the pseudolikelihood $\phi$ allows us to conclude that $\phi(J) > \phi(J^*)$. Thus, $\hat{J}$ lies in region I.}
  \label{fig:shell}
\end{figure}

The reason we take a ring is that on the one hand we would like $\|J - J^*\|_F$ to be upper bounded in order to lower bound the $\sech$ term, but on the other hand we would like $\|J - J^*\|_F$ to be lower bounded, so that we are guaranteed that with high probability the second derivative dominates the first.
Indeed, using all the previous observations and some careful union bound arguments over a suitable discrete net of points, we can establish that with high probability $\forall J \in \mathcal{R}_\epsilon$ we have $\phi(J) > \phi(J^*)$. It turns out that this also implies that $\phi(J) > \phi(J^*)$ for all $J$ with $\|J - J^*\|_F > 2 \epsilon$, by using a convexity argument, as shown in Figure~\ref{fig:shell}. Thus, we can conclude that $\|\hat{J} - J^*\|_F$ will be small. 

\subsection{Bounded Width Models}\label{sec:technical_width}
In \cite{klivans2017learning}, they obtain TV learning guarantees by using the guarantees for structure learning, which ensures that $\max_{ij}|J_{ij} - J^*_{ij}| \leq \epsilon$ using $O(\log n/\epsilon^4)$ samples. They subsequently set the accuracy to $\epsilon/n^2$, so that it guarantees closeness in TV. However, requiring accuracy for each entry of $J^*$ is a very strict condition, which results in sample complexity $O(n^8)$. In this work, we instead obtain guarantees of closeness to $J^*$ through different metrics, which, while counterintuitive at first, are more suitable for implying closeness in TV.

We first discuss how to get the improved guarantee of Theorem~\ref{thm:mple_improved_informal}. Since bounded-width models could be in low temperature, we cannot, in general, rely on concentration bounds for the original distribution $\Pr_{J^*}$. We follow the strategy presented in \cite{dagan2021learning}, which involves finding a small ($O(\log n)$) collection of subsets $I_j \subseteq [n]$, such that conditioned on each subset the model satisfies Dobrushin's condition, and every node $i\in [n]$ is contained in a constant fraction of the subsets. 
For the second derivative, we notice that it can be lower bounded by a larger quantity than that in \cite{dagan2021learning}, without breaking it into parts, by simply conditioning on one of the subsets $I_j$. Indeed, using \eqref{eq:first_second_der} and the bounded-width property, we can lower bound the conditional expectation of the second derivative by the variance of a linear function of the model as follows. 

\begin{align*}
\E\lp[\frac{\partial^2 \phi(J)}{\partial A^2}\middle|X_{-I_j}\rp] &\geq \sech(M)^2 \sum_{k=1}^l \sum_{i=1}^n \E\lp[\lp(A_iX^{(k)}\rp)^2\middle|X^{(k)}_{-I_j}\rp]\\
&= \sech(M)^2 \sum_{k=1}^l \sum_{i=1}^n \lp(\Var\lp(A_iX^{(k)}\middle|X_{-I_j}^{(k)}\rp)+ \E\lp[A_iX^{(k)}\middle|X_{-I_j}^{(k)}\rp]^2\rp)
\end{align*}
In the above, we have used the definition of the conditional variance.
We can show that the conditional variance summed over all nodes is of the order $\Theta(\|A_{I_j}\|_F^2)$, since the conditional model satisfies Dobrushin's condition. For the variance of the conditional expectation, there is no general formula, so we would like to approximate it by its expectation $\E[\|E[AX|X_{-I_j}]\|_2^2]$. Simply using the Chernoff bound for the independent samples does not suffice, because we would like this lower bound to hold uniformly for all directions $A$, and the union bound will incur additional polynomial factors. We avoid this union bounding argument by a careful application of matrix Bernstein's inequality. This enables us to approximate the second term 
in the sum by $\E[\|E[AX|X_{-I_j}]\|_2^2]$ uniformly over all matrices. 
These insights result in a lower bound of $\Omega(l \cdot \E_{J^*}[\|AX\|_2^2])$ for the second derivative.

For the first derivative, we can upper bound it by $O(l \cdot \E_{J^*}[\|AX\|_2^2])$ uniformly for all matrices $A$ by using the technique of splitting it into parts and a similar application of matrix Bernstein's inequality. Since both the first and second derivatives are upper and lower bounded by the same quantity, we can proceed as previously using a Taylor expansion for $\phi$ and establish that $E_{J^*}[\|(\hat{J} - J^*\|_2^2 X] \leq \epsilon$ with high probability.

To obtain the improved sample complexity of $O(n^3/ \epsilon^2)$, we use the improved estimation guarantees provided by Theorem~\ref{thm:mple_improved_informal}. Again using Pinsker's inequality, it is enough to upper bound $\dkl(\Pr_{\hat{J}},\Pr_{J^*})$, which using standard manipulations for exponential families and Taylor's theorem can be expressed as $\Var_{J_\xi}(X^\top (\hat{J} - J^*)X)$, where $J_\xi$ is a matrix lying in the line connecting $\hat{J}$ and $J^*$. Using the Cauchy-Schwarz inequality, we can upper bound it as $\Var_{J_\xi}(X^\top (\hat{J} - J^*)X) \leq n \cdot \E_{J_\xi}[\|(\hat{J} - J^*)X\|_2^2]$, which connects KL with the guarantee obtained by Theorem~\ref{thm:mple_improved_informal}. However, the expectation is now with respect to $J_\xi$ instead of $J^*$. This is where we make use of the regularity assumption~\ref{def:regularity}, which ensures that this expectation is robust to perturbing the measure from $J^*$ to $J_\xi$. Choosing accuracy $\epsilon/n$ for Theorem~\ref{thm:mple_improved_informal} then concludes the claim.

\section{Preliminaries}\label{sec:preliminaries}
In this Section, we collect useful definitions, notations, and facts.
Throughout the proofs, we will use generic constant notations such as $C, C', K, K'$. These refer to absolute constants that could be different from place to place, unless it is specified that some constant depends on some other quantity. 
For some matrix $A\in \mathcal{S}_0^n$, we denote by $A^{(l)} \in \R^{ln \times ln}$ the corresponding block diagonal matrix with each block equal to $A$ and a total of $l$ blocks. Similarly, for some set of matrices $\mathcal{R} \subseteq \R^{n \times n}$, we denote $\mathcal{R}^{(l)} := \{J^{(l)}: J \in \mathcal{R}\}$. 
For a symmetric matrix $J \in \mathcal{S}_0^n$ and a subset $I \subseteq [n]$, we denote by $J_I \in \R^{|I|\times n}$ the matrix consisting only of the rows of $J$ that are indexed by elements in $I$. We also denote $J_{II} \in \R^{|I|\times |I|}$ the submatrix with rows and columns indexed by $I$. 

We start by reviewing properties of the pseudolikelihood function $\phi$, as defined in \eqref{eq:pseudolikelihood_function}. 
We define the first and second derivatives of $\phi$ at a matrix $J \in \mathcal{S}_0^n$ in the direction of a matrix $A \in \mathcal{S}_0^n$ as in \cite{dagan2021learning}
\begin{align}
    \left.\frac{\mathrm{d}\phi(J+tA)}{\mathrm{d}t}\right|_{t=0}=\frac{1}{2}\sum_{i=1}^n(A_ix)(\tanh(J_ix+h_i)-x_i) \label{eq:first_derivative}\\
    \left.\frac{\mathrm{d}^2\phi(J+tA)}{\mathrm{d}^2t}\right|_{t=0}=\frac{1}{2}\sum_{i=1}^n(A_iX+h_i)^2\sech(J_iX+h_i)^2 \label{eq:second_derivative}
\end{align}

Since the second derivative is always non-negative, we conclude that $\phi$ is a convex function. We have included a known external field $h\in \R^n$ in the expression, since the general result does not require $0$ external field.

We now introduce some important functional inequalities that will be used in this work. The definitions below are standard and can be found, e.g., in \cite{van2014probability}. For a measure $\mu$ in the hypercube $\{-1,1\}^n$, we write $\E_\mu, \Var_\mu$ for the expectation and variance with respect to this measure $\mu$. 
The \emph{Dirichlet form} associated with the Glauber dynamics for a measure $\mu$ can be defined as the following operator on two arbitrary functions $f,g:\{-1,1\}^n \mapsto \R$
\begin{equation}\label{eq:dirichlet}
    \mathcal{E}(f,g) := \frac{1}{n} \E_\mu\lp[\sum_{i=1}^n \lp(\E_\mu[f(X)|X_{-i}] - f(X)\rp) \cdot \lp(\E_\mu[g(X)|X_{-i}] - g(X)\rp)\rp]
\end{equation}

We define the \emph{entropy} of a positive function $f:\{-1,1\}^n \mapsto \R^+$ as
\begin{equation}
    \Ent{f} := \E_\mu[f(X) \log f(X)] - \E_\mu[f(X)] \cdot \log \E_\mu[f(X)]
\end{equation}

We will need two notions of a modified log-Sobolev inequality. To define the first one, let $\mathcal{G}$ be the set of all functions $\{-1,1\}^n \mapsto \R$ and $\mathcal{G}'$ the subset of all positive functions. Let $\Gamma:\mathcal{G}\mapsto \mathcal{G}'$ be a so-called \emph{difference operator}. We say a measure $\mu$ supported on the hypercube satisfies a $\Gamma$-MLSI$(\rho)$, for some $\rho > 0$, if and only if for all $f \in \mathcal{G}$
\[
\Ent{e^f} \leq \frac{\rho}{2} \E_\mu\lp[\Gamma(f)^2 e^f\rp] \enspace,
\]
This notation is a classical notion of the Modified Log-Sobolev inequality and has been used in prior work for proving concentration inequalities \cite{bobkov1999exponential,sambale2019modified}.

There is a related notion of MLSI based on the Glauber dynamics. 
We say $\mu$ satisfies the \emph{modified log-Sobolev inequality for the Glauber dynamics} with constant $C>0$ and write Glauber-MLSI($C$), if for every function $f:\{-1,1\}^n \mapsto \R$
\[
\Ent{e^f} \leq C \cdot n \cdot \mathcal{E}(e^f,f)
\]

This version of the Modified Log-Sobolev inequality arises naturally when considering the rate of contraction of the KL-divergence between the distribution of a Markov Chain at a given time and the stationary distribution of the chain \cite{bobkov2006modified}.

Also, we say a measure $\mu$ satisfies a \emph{Poincare inequality} with constant $\rho>0$ and write Po($\rho$), if for every function $f:\{-1,1\}^n \mapsto \R$
\begin{equation}\label{eq:poincare}
    \Var_\mu(f) \leq \rho\cdot n \cdot \mathcal{E}(f,f)
\end{equation}

Both the Poincaré inequality and Glauber-MLSI have been established in prior work for spectrally bounded Ising models.
\begin{lemma}[\cite{eldan2022spectral}]\label{lem:Poincare}
    If $\lambda_{\max}(J^*) - \lambda_{\min}(J^*) < 1 - \alpha$, then $\Pr_{J^*}$ satisfies Po($1/\alpha$). 
\end{lemma}

\begin{lemma}[\cite{anari2021entropic,chen2022localization}]\label{lem:mlsi}
     If $\lambda_{\max}(J^*) - \lambda_{\min}(J^*) < 1 - \alpha$, then $\Pr_{J^*}$ satisfies Glauber-MLSI($1/\alpha$).
\end{lemma}

Furthermore, it is a well-known identity (for example, \cite{bobkov2006modified}) that, in general, the Glauber MLSI implies the Poincaré inequality with a slightly worse constant. For completeness, we give a proof below.

\begin{lemma}\label{lem:mlsi_poincare}
    If for some $\mu$, Glauber-MLSI$(C)$ satisfies, then it also satisfies Po$(2C)$.
\end{lemma}
\begin{proof}
    Consider $F(t)=Cn\mathcal E(e^{tf},tf)-\Ent{e^{tf}}$. We know that $F(0)=0$. Since we have $F\ge 0$, if we have $F'(0)= 0$, we must have $F''(0) \geq 0$. Taking the derivative, we have
    \begin{align*}
        \frac{\mathrm d^2}{\mathrm dt^2}n\mathcal E(e^{tf},tf)&=\frac{\mathrm d^2}{\mathrm d t^2}\left(C \E_\mu\lp[\sum_{i=1}^n \lp(\E_\mu[e^{tf}|X_{-i}] - e^{tf}\rp)\lp(\E_\mu[{tf}|X_{-i}] - {tf}\rp)\rp]\right)\\
        &=C\frac{\mathrm d}{\mathrm d t}\left( \E_\mu\lp[\sum_{i=1}^n \lp(\E_\mu[fe^{tf}|X_{-i}] - fe^{tf}\rp)\lp(\E_\mu[{tf}|X_{-i}] - {tf}\rp)\right.\rp.\\
        &~~~~~~~~~~~~~~~~+\lp.\left.\lp(\E_\mu[e^{tf}|X_{-i}] - e^{tf}\rp)\lp(\E_\mu[{f}|X_{-i}] - {f}\rp)\rp]\right)\\
        &=C\left( \E_\mu\lp[\sum_{i=1}^n \lp(\E_\mu[f^2e^{tf}|X_{-i}] - f^2e^{tf}\rp)\lp(\E_\mu[{tf}|X_{-i}] - {tf}\rp)\right.\rp.\\
        &~~~~~~~~~~~~~~~~+2\lp.\left.\lp(\E_\mu[fe^{tf}|X_{-i}] - fe^{tf}\rp)\lp(\E_\mu[{f}|X_{-i}] - {f}\rp)\rp]\right)
    \end{align*}

    When we set $t=0$, we have
    \begin{align*}
        \left.\frac{\mathrm d^2}{\mathrm dt^2}\mathcal E(e^{tf},tf)\right|_{t=0}&=2\E_\mu\lp(\left[\lp(\E_\mu[f|X_{-i}] - f\rp)\lp(\E_\mu[{f}|X_{-i}] - {f}\rp)\rp]\right)=2n\mathcal E(f,f)
    \end{align*}

    Also, we have
    \begin{align*}
        \frac{\mathrm d^2}{\mathrm dt^2} \Ent{e^{tf}}&=\frac{\mathrm d^2}{\mathrm dt^2}\lp(\E_\mu[e^{tf}\cdot  t f]-\E_\mu[e^{tf}]\log\E_\mu[e^{tf}]\rp)\\
        &=\frac{\mathrm d}{\mathrm dt}\lp(\E_\mu[e^{tf}\cdot f+e^{tf}\cdot tf^2]-\E_\mu[fe^{tf}]\log\E_\mu[e^{tf}]-\E_\mu[fe^{tf}]\rp)\\
        &=\lp(\E_\mu[2e^{tf}\cdot f^2+e^{tf}\cdot tf^3]-\E_\mu[f^2e^{tf}]\log\E_\mu[e^{tf}]-\frac{\E_\mu[fe^{tf}]^2}{\E_\mu[e^{tf}]^2}-\E_\mu[f^2e^{tf}]\rp)
    \end{align*}

    Therefore, we have
    \begin{align*}
        \left.\frac{\mathrm d^2}{\mathrm dt^2} \Ent{e^{tf}}\right|_{t=0}&=\lp(2\E_\mu[ f^2]-\E_\mu[ f^2]^2-\E_\mu[f]^2\rp)=\Var_{\mu}(f)
    \end{align*}

    Therefore, we have $F''(0)=2Cn\mathcal E(f,f)-\Var_\mu(f)\ge 0$, which means that Po$(2C)$ holds.
\end{proof}

The reason $\Gamma$-MLSI is useful is that concentration results for second-order polynomials are proven in \cite{sambale2019modified} when the distribution $\mu$ satisfies this property.
On the other hand, Glauber-MLSI is usually the property that follows from results about fast mixing \cite{eldan2022spectral,chen2022localization,anari2021entropic}.
It would thus be desirable to connect the two notions of MLSI in order to establish concentration results for the Ising model. To facilitate that connection, following \cite{sambale2019modified} we define the operator $\mathfrak{d}^+$ as follows
\begin{equation}\label{eq:operator}
    \mathfrak{d}^+f (X) = \sqrt{\sum_{i=1}^n \E\lp[\lp((f(X) - f(X_i',X_{-i})\rp)_+^2\middle | X_{-i}\rp]}\enspace.
\end{equation}
In the above, we use $(X'_i, X_{-i})$ be a shortened random vector $(X_1', X_2',\dots, X_n')$ that has all other coordinate $j$, $X_j'=X_j$, and $X_i'$ is sampled independently of everything else according to the distribution conditioning on $X_{-i}'=X_{-i}$.
Also, we have used the shorthand notation $x_+ := \max(0,x)$. 
The quantity $\mathfrak{d}^+f (x)$ can be thought of as the $l_2$-norm of a discrete derivative of $f$ at $x$.
It is therefore capturing the Lipschitzness of $f$ in some appropriate sense and will thus be important for proving that $f(X)$ concentrates. The connection between the two definitions is that Glauber-MLSI$(\rho)$ implies $\mathfrak{d}^+$-MLSI$(2\rho)$, since by direct calculation we can get that $\E_{\mu}[\mathfrak d^{+}f\cdot e^f]\ge n\mathcal E(f,e^f)$ \cite{sambale2019modified}. Indeed, in the following lemma, we use this fact and derive a generic concentration result involving $\mathfrak{d}^+$.

\begin{lemma}\label{lemma:generic_concentration}
    Suppose that $\mu$ satisfy Glauber-MLSI$(\rho)$. Then, for any functions $f,g:\{-1,1\}^n \mapsto \R$ and constant $b > 0$ such that $\mathfrak{d}^+f \leq g$ and $\mathfrak{d}^+g \leq b$, we have for any $t > 0$
    \[
    \Pr\lp[|f(X) - \E[f(X)]| > t\rp] \leq \frac{8}{3}\exp\lp(-\frac{1}{16\rho} \min\lp(\frac{t^2}{\E[g]^2},\frac{t}{b}\rp)\rp).
    \]
\end{lemma}
\begin{proof}
    By Proposition 2.18 in \cite{sambale2019modified}, if $\mu$ satisfies Glauber-MLSI$(1/\alpha)$, it follows that the same Ising model also satisfies $\mathfrak{d}^+-$MLSI$(2/ \alpha)$. Thus, we can apply Corollary 1.2 from \cite{sambale2019modified} and the result immediately follows. 
\end{proof}

We now state a useful lemma for upper-bounding second moments of the model.

\begin{lemma}\label{lem: poincare}
    Let $A$ be any matrix, and $x\sim \mu$ be an Ising model satisfying Po$(\rho)$. We have the following:
    \[\E\lr{\lrnorm{Ax}^2}\le \|\E[Ax]\|^2+\frac{1}{\alpha}\lrnorm{A}_{F}^2\]
\end{lemma}
\begin{proof}
    By the Poincaré Inequality, we have that 
    \[\E[(A_i x)^2]-(\E[A_ix])^2=\Var(A_ix)\le \rho\E_{\mu}\sum_{j=1}^n(\E[A_ix|x_{-j}]-A_ix)^2 \le \sum_{j=1}^nA_{ij}^2.\]

    Therefore, we have in total
    $$\E(\lrnorm{Ax}^2)\le\sum_{i=1}^n(\E[A_ix])^2+\rho\sum_{i=1}^n\sum_{j=1}^n A_{ij}^2=\|\E[Ax]\|^2+ \rho\lrnorm{A}_F^2.
    $$
\end{proof}

Another useful consequence of MLSI is the simpler concentration of Lipschitz functions, which is well known. 
Below is one version of this implication from \cite{cryan2019modified}. 

\begin{lemma}[Lemma 15 in \cite{cryan2019modified}]\label{lem: concentration of linear}
     Let $P$ be the transition matrix of a reversible Markov Chain with stationary distribution $\pi$
 on a finite set $\Omega$, and $f : \Omega\to\R$ be some observable function. Then,
 \[\Pr_{x\sim \pi}(f(x) - \E_{\pi}f \ge a)\le \exp(-\frac{a^2}{2n\cdot\rho\cdot v(f)})\]
 where $a\ge 0$ and
 \[v(f) := \max_{x\in\Omega}\left\{
 \sum_{y\in\Omega}P(x,y)(f(x) -f(y))^2
 \right\}\enspace,\]
 if $\pi$ satisfies Glauber-MLSI($\rho(P)$).
\end{lemma}

Finally, we also require the \emph{Hubbard-Stratonovich} transform, which is a way to decompose any Ising model distribution into a mixture of product distributions.
Formally, suppose $X \sim \Pr_{J^*}$ and $G \sim \calN(0,I_n)$. Let us consider the random variable
\begin{equation}\label{eq:ydist}
Y=X+J^{-1/2}G
\end{equation}
An easy calculation now shows that the distribution of $X$ conditioned on $Y$ is an Ising model with \emph{zero} interaction matrix and external field $J^*Y$ (for details see e.g. Theorem 3.12 in \cite{liu2024fast}). 
Thus, if $\pi_y$ is the distribution of $Y$, we can decompose any Ising measure as
\begin{equation}\label{eq:decomposition}
    \Pr_{J^*}[x] = \int_{y \in \R^n} \pi(y)\Pr_{0,J^*y}[x] dy
\end{equation}

\section{Learning Ising Models that satisfy MLSI}

In this Section, we study Ising models of the form \eqref{eq:ising_model} with $h =0$ that satisfy the modified log-Sobolev Inequality. In particular, throughout the section, we will make the following assumption. 
\begin{assumption}\label{ass:spectral_bounded}
    The set of all candidate matrices, $\mathcal R \subseteq \mathcal{S}_0^n$, contains $J^*$. The Ising model, as in \eqref{eq:ising_model} with $h$ a known external field, satisfies Glauber-MLSI$(\rho)$. 
    Also, $\max_{1\le i\le n}|h_i|\le h_{\max}$. Also, $J^*$ satisfies $\|J^*\|_{op}\le\lambda$ for some constant $\lambda>0$.
\end{assumption}

Without loss of generality, we assume $\rho,\lambda\ge 1$.

\subsection{Concentration of the first derivative}

The first derivative of the pseudolikelihood is given by the formula in \eqref{eq:first_derivative}.
Even though this function is not a polynomial, intuitively it behaves similarly to a second-degree polynomial if we linearize the $\tanh$ function. 
Indeed, our goal in this section will be to prove that it concentrates similarly to a second-degree polynomial. This was shown to hold in \cite{dagan2021learning} in the case where $\|J^*\|_\infty < 1$, using the Approximate Tensorization of Entropy (ATE) and the bound on the infinity norm.
Here, we will show that it still holds under MLSI and bounded operator norm, using Lemma~\ref{lemma:generic_concentration}.

\begin{theorem}\label{thm:concentration_first_generic}
Suppose $X \in \{-1,1\}^n$ is sampled from an Ising model with interaction matrix $J^*$ satisfying Assumption~\ref{ass:spectral_bounded}.
For a fixed vector $b \in \R^n$ and fixed symmetric matrix $A \in \R^{n \times n}$ with zero diagonal, let us define the function
\[
f(X) = \sum_{i=1}^n(A_iX + b_i)(\tanh(J_i^*X + h_i)-X_i) \enspace.
\]
Then, we can take $C=\frac{1}{2^{21}\lambda^4\rho^2}>0$, such that for all $t > 0$
    \[
    \Pr\lp[\lp| f(X)\rp| > t\rp] \leq \frac{8}{3} \exp\lp(- C \min\lp(\frac{t^2}{\|A\|_F^2 + \lp\|\E[AX + b]\rp\|_2^2}, \frac{t}{\|A\|_{op}}\rp)\rp)\enspace.
    \]

\end{theorem}
\begin{proof}
    Let us define for a vector $X \in \{-1,1\}^n$ the vector $\flip{X}{i} \in \{-1,1\}^n$, where the $i$-th coordinate is flipped. Also, define vectors $X^{(k+)},X^{(k-)}$ which $X^{(k+)}_j=X^{(k-)}_j=X_j$ holds for all coordinates $j=1,\dots,n,j\ne k$, while $X^{(k+)}_k=1,X^{(k-)}_k=-1$.
    Moreover, let us define the matrix $W = W(X) \in \R^{n \times n}$, where the element in the $j$th row, $k$th column is equal to 
    \[
    W_{jk} = \tanh(J_j^* X^{(k+)} + h) - \tanh(J_j^*X^{(k-)} + h) - 2J_{jk}^*,
    \] 
    In the sequel, we might omit the dependence of $W$ on $X$ for brevity. Note that $W$ is not necessarily symmetric. We denote by $W_k$ the $k$th column and $W^k$ the $k$-th row of $W$.
    First, Let us bound $\|W\|_{op}$, where $\|W\|_{op}=\sqrt{\|WW^{\top}\|_{op}}$
    Note the well-knownn fact that for all $a,b \in \R$ 
    \begin{equation}\label{eq:tanh_second_order}
    |\tanh(a) - \tanh(b) - (a-b)| \leq \frac{1}{2} (a-b)^2
    \end{equation}
    Then using \eqref{eq:tanh_second_order},  we can bound each entry $W_{jk}$ of the matrix as
    \[
    |W_{jk}| \leq \frac{1}{2}(2J^*_{jk})^2 = 2(J^*_{jk})^2 \enspace,
    \]
    Thus, the $l_1$ norm of every column $j$ of $W(X)$ can be bounded as follows
    \[
    \|W_j\|_1 = \sum_{k=1}^n W_{jk} \leq 2 \sum_{k=1}^n (J^*_{jk})^2 = 2 \|J^*_j\|_2^2 \leq 2\lambda^2\enspace.
    \]
    
    The last inequality follows from the fact that the $l_2$ norm of every row of $J^*$ is bounded by $\|J^*\|_{op}$. 
    Thus $\|W\|_\infty \leq 2\lambda^2$. Similarly, we can get that $\|W\|_1 \leq 2\lambda^2$, which implies that 
    $$
    \|W\|_{op} =\sqrt{\|WW^{\top}\|_{op}}\le\sqrt{\|WW^{\top}\|_1}\leq \sqrt{\|W\|_\infty \cdot \|W\|_1} \leq 2\lambda^2\enspace.
    $$
    
    We start by bounding $\mathfrak{d}^+ f$.
    We first notice that for any given $X$ (again, $\E$ in the following expression is same as Equation (\ref{eq:operator}): $X'_i$ is sampled according to the conditional distribution with fixed $X_{-i}$),
    \begin{align*}
        (\mathfrak{d}^+f(X))^2  &= \sum_{i=1}^n \E\lp[\lp((f(X) - f(X_i',X_{-i})\rp)_+^2\middle | X_{-i}\rp]\\
        &= \sum_{i=1}^n \Pr[X_i' = 1 - X_i|X_{-i}] (f(X) - f(X_{-i},1 - X_i))_+^2 \\
        &\leq \sum_{i=1}^n \lp(f(X) - f(\flip{X}{i})\rp)^2
    \end{align*}
    We also define the vector functions
    \[
    g(x) := Ax + b \quad,\quad h(x) := \tanh(J^*x+b) - x\enspace,
    \]
    which allows us to write $f(x) = g(x)^\top h(x)$. Here, we write $h(x)$ as a vector, with $h(x)_i=\tanh(J^*_ix+b)-x_i$. Therefore, we can calculate that
    \[h(x)_j-h(\flip{x}{i})_j= \tanh(J_j^*x+b)-\tanh(J_j^*\flip{x}{i}+b)-2x_i\mathbb 1(i=j).\]

    So, we have
    \[h(x)-h(\flip{x}{i})=\tanh(J^*x+b)-\tanh(J^*\flip{x}{i}+b)-2x_ie_i=x_iW_i+2x_iJ^*_i-2x_ie_i=x_i(W_i+2J^*_i-2e_i),\]
    
    where $e_i$ is the standard basis vector with all coordinates $0$ except for the $i$-th that is $1$. With that, we can further bound the discrete derivative as follows. For all $x\in\{-1,1\}^n$, we have
    \begin{align*}
        &(\mathfrak{d}^+ f (x))^2\\
        \leq &\sum_{i=1}^n \lp(g(x)^\top h(x) - g(\overline{x^{(i)}})^\top h(\flip{x}{i})\rp)^2\\
        = &\sum_{i=1}^n \lp(\lp(g(x) - g(\flip{x}{i})\rp)^\top h(x) + \lp(g(\flip{x}{i}) - g(x)\rp)^\top \lp(h(x) - h(\flip{x}{i})\rp) + g(x)^\top \lp(h(x) - h(\flip{x}{i})\rp)\rp)^2\\
        = &\sum_{i=1}^n \lp(2x_i A_i^\top h(x) + 2x_i A_i^\top \cdot x_i\lp( 2J^*_i + W_i-2e_i\rp) + (Ax + b)^\top \cdot x_i\lp( 2J^*_i + W_i-2e_i\rp)\rp)^2.
    \end{align*}

        By Cauchy-Schwartz inequality, $(x+y+z)^2\le 3(x^2+y^2+z^2)$, and we know that $x_i=\pm 1$, we have the above expression is no larger than
        \begin{align*}
        &{ 3 \sum_{i=1}^n \lp(4\lp(A_i^\top h(x)\rp)^2 + 4 \lp(A_i^\top (2J^*_i +W_i - 2e_i)\rp)^2 + \lp((Ax+b)^\top \lp(2J^*_i +W_i - 2e_i\rp)\rp)^2\rp)
        } 
        \end{align*}
        
        First, we deal with the middle term $\lp(A_i^\top (2J^*_i +W_i - 2e_i)\rp)^2$. Using Cauchy-Schwartz inequality for vectors ($x^\top y\le \|x\|\cdot \|y\|$), for the middle term, we have
        \begin{align*}
            &\lp(A_i^\top (2J^*_i +W_i - 2e_i)\rp)^2
            \le \|A_i\|_2^2\cdot \lp\|2J^*_i +W_i - 2e_i\rp\|_2^2\\
            \le & \|A_i\|_2^2\cdot \lp(2\|J^*_i\|_2 + \|W_i\|_2+2\rp)^2\le (2+2\lambda+2\lambda^2)^2\|A_i\|^2.
        \end{align*}

        Here, the last inequality is because $\|J_i\|_2\le \lambda$ and $\|W_i\|_2\le 2\lambda^2$. Therefore, summing over $i$, we have
        \[3\sum_{i=1}^n 4\lp(A_i^\top (2J^*_i +W_i - 2e_i)\rp)^2\le 3\sum_{i=1}^n 4\cdot (2+2\lambda+2\lambda^2)^2\|A_i\|^2=48(1+\lambda+\lambda^2)^2\|A\|_F^2.\]
        
        Now, we deal with the last term. $\lp((Ax+b)^\top \lp(2J^*_i +W_i - 2e_i\rp)\rp)^2$ By Cauchy-Schwartz inequality again, for the last term, we have
        \begin{align*}
            &\lp((Ax+b)^\top \lp(2J^*_i +W_i - 2e_i\rp)\rp)^2\\
            \le & 12 \lp(A_i^\top x + b_i\rp)^2 + 12\lp((Ax+b)^\top  J^*_i\rp)^2  + 3 \lp((Ax+b)^\top W_i\rp)^2
        \end{align*}

        Again, after summing over $i$, then using the fact that $\|J^*\|_{op}\le \lambda$ and $\|W\|_{op}\le 2\lambda^2$, we can deduce that
         \begin{align*}
            &\sum_{i=1}^n 12 \lp(A_i^\top x + b_i\rp)^2 + 12\lp((Ax+b)^\top  J^*_i\rp)^2  + 3 \lp((Ax+b)^\top W_i\rp)^2\\
            =&12\|(Ax+b)^\top J_i^*\|^2+3\|(Ax+b)^\top W\|^2+12\|Ax+b\|_2^2\\
            \le&12\lambda^2\|Ax+b\|_2^2+12\lambda^4\|Ax+b\|_2^2+12\|Ax+b\|_2^2\\
            =&12(1+\lambda^2+\lambda^4)\|Ax+b\|^2=12(1+\lambda^2+\lambda^4)\sum_{i=1}^n (A_i^\top x+b_i)^2
        \end{align*}
        
        Combining them together, we have
        \begin{align*}
            (\mathfrak{d}^+ f (x))^2\le&3 \sum_{i=1}^n \lp(4\lp(A_i^\top h(x)\rp)^2 + 4 \lp(A_i^\top (2J^*_i +W_i - 2e_i)\rp)^2 + \lp((Ax+b)^\top \lp(2J^*_i +W_i - 2e_i\rp)\rp)^2\rp)\\
            \le & 48(1+\lambda+\lambda^2)^2\|A\|_F^2+(3\cdot 12)(1+\lambda^2+\lambda^4)\sum_{i=1}^n (A_i^\top x+b_i)^2+3 \sum_{i=1}^n 4\lp(A_i^\top h(x)\rp)^2\\
            \le &48(1+\lambda+\lambda^2)^2\|A\|_F^2+36(1+\lambda^2+\lambda^4)\sum_{i=1}^n (A_i^\top x+b_i)^2+\lp(A_i^\top h(x)\rp)^2
        \end{align*}
    
    Let 
    \begin{equation}\label{eq:qbound}
         q(x) = \sqrt{48(1+\lambda+\lambda^2)^2\|A\|_F^2+36(1+\lambda^2+\lambda^4)\sum_{i=1}^n (A_i^\top x+b_i)^2+\lp(A_i^\top h(x)\rp)^2}.
     \end{equation} 
    We thus have $\mathfrak{d}^+ f(x) \leq q(x)$ for all $x \in \{-1,1\}^n$.
    Let us first bound $\E\lp[q(X)\rp]^2\le \E[q(x)^2]$. 
    First, notice that
    \[
    \E[h(X)] = \E[X - \tanh(J^*X+h)] = 0 \enspace.
    \]
    This enables us to write.
    \[
    \E[(A_i^\top h(X))^2] =  \Var[A_i^\top h(X)]
    \]
    We can now use the Poincaré inequality to bound the above variance. Note that by Lemma~\ref{lem:mlsi_poincare}, if $\mu$ satisfies Glauber MLSI$(\rho)$, then it also satisfies Po$(2\rho)$. So we get
    \begin{align*}
        \Var[A_i^\top h(X)] &\leq 2\rho \sum_{k=1}^n \E\lp[\Var\lp[A_i^\top h(X)|X_{-k}\rp]\rp]\\
        &\leq 2\rho \sum_{k=1}^n \E\lp[\lp(A_i^\top (h(X_{k+}) - h(X_{k-}))\rp)^2\rp]
    \end{align*}
    We have that
    \begin{align*}
    A_i^\top (h(X_{k+}) - h(X_{k-})) &= A_i^\top (\tanh(J^* X_{k+} + h) - \tanh(J^*X_{k-}+h)-2e_k)\\
         &= -2A_{ik} + A_i^\top (2J^*e_k) + A_i^\top W_k\\
        &= -2A_{ik} +2A_i^\top J^*_k +  A_i^\top W_k
    \end{align*}
    Applying the Cauchy-Schwarz inequality now yields
    \begin{align*}
        \lp(A_i^\top (h(X_{k+}) - h(X_{k-})) \rp)^2 &\leq 3\lp(4A_{ik}^2 + 4(A_i^\top J^*_k)^2 + (A_i^\top W^k)^2\rp)\\ 
    \end{align*}
    Summing over all $k$ now gives
    \begin{align*}
        \sum_{k=1}^n \lp(A_i^\top (h(X_{k+}) - h(X_{k-})) \rp)^2 &\leq 3 \sum_{k=1}^n \lp(4A_{ik}^2 + 4(A_i^\top J^*_k)^2 + (A_i^\top W^k)^2\rp)\\
        &= 12 \|A_i\|_2^2 + 12 \|J^* A_i\|_2^2 + 3\|W^\top A_i\|_2^2\\
        &\leq 12 \|A_i\|_2^2 + 12 \|{J^*}^\top\|_{op}^2 \|A_i\|_2^2 + 3\|W^\top\|_{op}^2 \|A_i\|_2^2.
    \end{align*}
    
    Using the fact that $\|J^*\|_{op}< \lambda$ and $\|W\|_{op} \leq 2\lambda^2$, we have
    \[\sum_{k=1}^n \lp(A_i^\top (h(X_{k+}) - h(X_{k-})) \rp)^2 \leq 12(1+\lambda^2+\lambda^4) \|A_i\|_2^2.\]

    Therefore, we have
    \[
    \E[(A_i^\top h(X))^2] \leq 24\rho(1+\lambda^2+\lambda^4) \|A_i\|^2\enspace.
    \]

    For the second term, the analysis follows along similar lines if we apply a suitable centering to make the variance appear. In particular, we can write
    \begin{align*}
        \E[(A_i^\top X + b_i)^2] - \E\lp[A_i^\top X + b_i\rp]^2&= \Var[A_i^\top X] \\
        &\leq 2\rho \sum_{k=1}^n \E\lp[\Var[A_i^\top X|X_{-i}]\rp]\\
        &\leq 2\rho \sum_{k=1}^n A_{ik}^2\\
        &= 2\rho \|A_i\|_2^2
    \end{align*}
    Summing over all $i$ and using \eqref{eq:qbound} gives
    \begin{align}\label{eq:radiusbound}
        \E[q(X)]^2 &\leq 48(1+\lambda+\lambda^2)^2\|A\|_F^2+36(1+\lambda^2+\lambda^4)\sum_{i=1}^n (A_i^\top x+b_i)^2+\lp(A_i^\top h(x)\rp)^2\notag \\
        &\leq  48(1+\lambda+\lambda^2)^2\|A\|_F^2\notag\\
        &+36(1+\lambda^2+\lambda^4)\sum_{i=1}^n 24\rho(1+\lambda^2+\lambda^4) \|A_i\|^2+2\rho (\|A_i\|^2+\E[A_i^\top X+b]^2)\notag\\
        &\leq 2^{11}(1+\lambda^4)(1+\rho) \lp(\|A\|_F^2 + \lp\|\E[Ax + b]\rp\|_2^2\rp).
    \end{align}
    
    We will now focus on bounding $\mathfrak{d}^+q(x)$. We have
    \begin{align*}
         q(x) = &\sqrt{48(1+\lambda+\lambda^2)^2\|A\|_F^2+36(1+\lambda^2+\lambda^4)\sum_{i=1}^n (A_i^\top x+b_i)^2+\lp(A_i^\top h(x)\rp)^2}\\
         \le &\sqrt{48}(1+\lambda+\lambda^2)\|A\|_F+6\sqrt{1+\lambda^2+\lambda^4}\sqrt{\sum_{i=1}^n (A_i^\top x+b_i)^2+\lp(A_i^\top h(x)\rp)^2}.
     \end{align*} 
    Let $r(x)=\sqrt{\sum_{i=1}^n (A_i^\top x+b_i)^2+\lp(A_i^\top h(x)\rp)^2}$. Then we define
    \[K^{(i)}(x)=A_i^{\top}h(x);~~~~L^{(i)}(x)=A_i^{\top}x+b.\]
    
    We first notice that $\mathfrak{d}^+ q(x) \le 6\sqrt{1+\lambda^2+\lambda^4}\cdot \mathfrak{d}^+ r(x)$. This is because for all positive $a,b,c$, we have $|\sqrt{a+b}-\sqrt{a+c}|\le|\sqrt{b}-\sqrt{c}|$. 
    We define the vector function $s=s(x) \in \R^{2n}$, with $s(x)_i = K^{(i)}(x)$ if $i \leq n$ and $s(x)_i = L^{(i-n)}(x)$ if $i > n$. 
    We then have that
    \[
    r(x) = \|s(x)\|_2 = \sup_{v \in \mathcal{S}^{2n-1}} \langle s(x), v \rangle \enspace.
    \]
    Let us denote $\Tilde{v} = (\Tilde{v}_1, \Tilde{v}_2) := \argmax_{v \in \mathcal{S}^{2n-1}} \langle s(x), v \rangle $ the unit vector (in $\R^{2n}$, where $\tilde v_1,\tilde v_2\in \R^n$) that is in the direction of $s(x)$. 
    Then, we can write 
    \begin{align*}
        (\mathfrak{d}^+\|r(x)\|)^2 &\leq  \sum_{k=1}^n \lp(\|s(x)\| - \|s(\overline{x^{(k)}})\|\rp)_+^2\\
        &=  \sum_{k=1}^n \lp(\langle s(x), \Tilde{v} \rangle - \sup_{v \in \mathcal{S}^{2n-1}} \langle s(\overline{x^{(k)}}) , v \rangle\rp)_+^2\\
        &\leq  \sum_{k=1}^n \lp(\langle s(x), \Tilde{v} \rangle - \langle s(\overline{x^{(k)}}) , \Tilde{v} \rangle\rp)_+^2\\
        &\leq  \|\mat{S}\cdot\Tilde{v}\|^2\enspace,
    \end{align*}
    where $\mat{S} \in \R^{n \times 2n}$ is the matrix with $k$-th row equal to $s(x) - s(\overline{x^{(k)}})$.
    We can write in block form $\mat{S} = (\mat{S}_1|\mat{S}_2)$, where $\mat{S}_1, \mat{S}_2 \in \R^{n \times n}$. Let us bound $\|\mat{S}_1\|_{op},\|\mat{S}_2\|_{op} $. 
    We have that
    \begin{align*}
        (\mat{S}_1)_{ki} &= s_1^{(i)}(x) - s_1^{(i)}(\flip{x}{k})\\
        &= A_i^\top (h(x) - h(\flip{x}{k}))\\
        &=x_k A_i^\top (2 J^*_k +  W_k-2 e_k )\\
        &= x_k( 2(J^*A)_{ki} + (WA)_{ki}- 2A_{ki})
    \end{align*}
    For a vector $x \in \R^n$, we denote $\text{diag}(x)$ the diagonal matrix with diagonal entries equal to $x$. Then, we can write in matrix form
    \begin{align*}
    \|\mat{S}_1\|_{op} &=\lp\|\text{diag}(x) (2J^* A + WA-2 A )\rp\|_{op}\\
    &\leq \|\text{diag}(x)\|_{op} \cdot \|2J^* A + WA-2 A \|_{op}\\
    &\leq 2 \|A\|_{op} + \|J^*\|_{op}\cdot \|A\|_{op} + \|W\|_{op}\cdot \|A\|_{op}\\
    &\leq (2+\lambda+2\lambda^2) \|A\|_{op}\enspace.
    \end{align*}
    For $\mat{S}_2$ the situation is similar and we can write
    \begin{align*}
        (\mat{S}_2)_{ki} &= A_i^\top (x - \flip{x}{k}) \\
        &= 2x_k A_{ki}\enspace. 
    \end{align*}
    Thus
    \[
    \|\mat{S}_2\|_{op} =2\lp\|\text{diag}(x) A\rp\|_{op} \leq 2 \|A\|_{op}
    \]
    We can now conclude
    \begin{align*}
        \|\mat{S}\cdot\Tilde{v}\|^2 &= \|\mat{S}_1 \Tilde{v}_1 + \mat{S}_2 \Tilde{v}_2\|^2\\
        &\leq 2 \|\mat{S}_1 \Tilde{v}_1\|^2 + 2 \|\mat{S}_2 \Tilde{v}_2\|^2\\
        &\leq 2 \|\mat{S}_1\|_{op}^2\|\|\Tilde{v}_1\|^2 + 2 \|\mat{S}_2\|_{op}^2\|\Tilde{v}_2\|^2 \\
        &\leq 2\max( \|\mat{S}_1\|_{op}^2,\|\mat{S}_1\|_{op}^2)\le (2+\lambda+2\lambda^2)^2\|A\|_{op}^2\enspace,
    \end{align*}
    where we used the fact that $\|\Tilde{v}_1\|^2+\|\Tilde{v}_2\|^2 \leq 1$. 
    
    To summarize, we have shown that $\mathfrak{d}^+ f(x) \leq q(x)$ for all $x$. Furthermore, by \eqref{eq:radiusbound} we proved that 
    $$
    \E[q]^2 \leq 2^{11}(1+\lambda^4)(1+\rho) \lp(\|A\|_F^2 + \lp\|\E[Ax + b]\rp\|_2^2\rp) \|A\|_F^2 \quad, \quad
    \mathfrak{d}^+ q(x) \leq (2+\lambda+2\lambda^2)^2 \|A\|_{op}
    $$
    We also know that
    \[
    \E\lp[\frac{\partial \phi(J^*)}{\partial A}\rp] = 0
    \]
    Thus, applying Lemma~\ref{lemma:generic_concentration}, take $C=\frac{1}{16\rho}\frac{1}{2^{15}(1+\lambda^4)(1+\rho)}\ge\frac{1}{2^{21}\lambda^4\rho^2}$ we obtain that for every $t > 0$
    \[
    \Pr\lp[\lp| \frac{\partial \phi(J^*)}{\partial A}\rp| > t\rp] \leq \frac{8}{3} \exp\lp(- C \min\lp(\frac{t^2}{\|A\|_F^2}, \frac{t}{\|A\|_{op}}\rp)\rp)\enspace.
    \]
\end{proof}

We are now ready to establish the concentration of the first derivative for Ising models satisfying MLSI. This basically follows as a direct Corollary of Theorem~\ref{thm:concentration_first_generic}

\begin{lemma}\label{lem:concentration_first_high_temp}
Suppose $X \in \{-1,1\}^n$ is sampled from an Ising model satisfying Assumption~\ref{ass:spectral_bounded}.
Then, if $\phi$ is the pseudolikelihood function evaluated at $X$ and $A \in \mathcal{S}_0^n$, then there is a constant $C'=\frac{1}{2^{21}\lambda^4\rho^2}$ such that for any $t > 0$:
\[
\lp| \frac{\partial \phi(J^*)}{\partial A}\rp| \leq t\|A\|_F
\]
with probability at least
    \[
    1- \frac{8}{3} \exp\lp(- C'\min\lp(t^2, \frac{t \|A\|_F}{\|A\|_{op}}\rp)\rp)
    \]
\end{lemma}
\begin{proof}
    We apply Theorem~\ref{thm:concentration_first_generic} and substitute $t\|A\|_F$ for $t$, completing the proof.
\end{proof}

\subsection{Anti-Concentration of the Second Derivative}

We want to show that the second derivative is lower bounded with high probability. First, we present the lemma, as  the second form of Corollary 1.2 in \cite{sambale2019modified} (the inequality right after Corollary 1.2 in \cite{sambale2019modified}, which does not have a specific number)
\begin{lemma}[Corollary 1.2 in \cite{sambale2019modified}]
        Assume that $\mu$ satisfies a $\Gamma$-MLSI$(\rho)$ for some difference operator $\Gamma$ and $\rho>0$. Let $f,g$ be two measurable functions such that $\Gamma(f)\le g$ and $\Gamma(g)\le b$. Then there is a universal constant $c$ such that for all $t\ge 0$ we have
        \begin{equation}\label{eq:intermediate}
            \Pr[f-\E_{\mu}[f]\ge t]\le\exp\lp(-c\min\lp(\frac{t^2}{\rho(\E_\mu g)^2+2b^2\rho^2},\frac{t}{\sqrt{2}\rho b}\rp)\rp).
        \end{equation}
    \end{lemma}

From \eqref{eq:second_derivative}, it is clear that the second derivative is almost equal to a second-degree polynomial, up to a factor that involves the $\sech$ function. 
We thus start by stating a concentration bound for second-degree polynomials that essentially follows from \cite{sambale2019modified}. 

\begin{lemma}\label{lem:degree2concentration}
Suppose $X \in \{-1,1\}^n$ is sampled from an Ising model with interaction matrix $J^*$ satisfying Glauber-MLSI$(\rho)$ and external field $h \in \R^n$. Also, let $S \in \mathcal{S}_0^n$. Then, there exists an absolute constant $c>0$, such that for any $t > 0$,
\[
\Pr[x^\top S x - \E[x^\top S x] \ge t] \leq \exp\lp(-\frac{c}{\rho^2}\min\lp(\frac{t^2}{\|S\|_F^2 + \|\E[Sx]\|^2}, \frac{t}{\|S\|_{op}}\rp)\rp).
\]

For the other side, it is analogous.
\end{lemma}
\begin{proof}
    As remarked in the proof of Lemma~\ref{lemma:generic_concentration}, combining Theorem 12 from \cite{anari2021entropic} and Proposition 2.18 from \cite{sambale2019modified} yields that our model satisfies Glauber-MLSI$(\rho)$, $\mathfrak{d}^+$-MLSI$(2\rho)$ and Po$(2\rho)$. Thus, we can use the second form of Corollary 1.2 in \cite{sambale2019modified}:

    Let $f(x)=x^\top Sx$. By the proof of Lemma 2.17 in \cite{sambale2019modified}, we know that we can take $\Gamma=\mathfrak{d}^{+}$, $g(x)=4\|Sx\|$, and $b=8\|S\|_{op}$. It now remains to upper bound $\E[\|SX\|])^2$. To do that, we can again use the Poincaré inequality (Lemma~\ref{lem:Poincare}).
    \begin{align*}
        \E[\|SX\|]^2 &\leq \E\lp[\|SX\|^2\rp])\\
        &= \sum_{i=1}^n \E[(S_i^\top X)^2]\\
        &= \sum_{i=1}^n \lp(\E[S_i^\top X]^2 + \Var[S_i^\top X]\rp)\\
        &= \|\E[SX]\|_2^2 + \sum_{i=1}^n \Var[S_i^\top X]\\
        &\leq \|\E[SX]\|_2^2 + 2\rho\sum_{i=1}^n \|S_i\|_2^2 \\
        &= \|\E[SX]\|_2^2 + 2\rho\|S\|_F^2
    \end{align*}
    By substituting this upper bound into \eqref{eq:intermediate}, and notice that $\|S\|_{op}\le\|S\|_{F}$ and $\rho\ge 1$ by assumption, we obtain the desired inequality.
\end{proof}

Having obtained this concentration result, showing that the second derivative is large with high probability boils down to the following two tasks.
\begin{enumerate}[label=(\alph*)]
    \item\label{it:first} First, we need to establish that the second derivative is lower bounded by a degree 2 polynomial. To do this, we need to show that the terms $\sech(J_iX)^2$ are lower bounded by a constant with high probability. 
    \item\label{it:second} Second, we need to show that the degree 2 polynomial that lower bounds the second derivative is large enough on expectation. 
\end{enumerate}
We next show how to establish each of these two properties.
We start by addressing~\ref{it:second}, namely establishing a lower bound for the expectation of the second moment part of the second derivative, which will prove useful later. The proof involves using the Hubbard-Stratonovich transform to decompose the Ising model into a mixture of product measures. For each product measure, lower-bounding the second moment is a much simpler problem. However, some product measures in the decomposition will have large external fields, giving weak lower bounds. We use properties of this decomposition to establish that with at least a constant probability, the external field in the decomposition will be bounded. 

    \begin{lemma}\label{lem: lower bound Jop<1}
    Let $x\sim\mu$ be an Ising model satisfying Assumption~\ref{ass:spectral_bounded}. Let $A$ be any matrix. Then, we have the following anti-concentration bound: 
    \[\E(x^\top A^{\top}Ax)-\|\E[Ax]\|^2\ge \frac{1}{2}\lr{1-\tanh^2\lr{4\lambda\sqrt{\rho}}}\lrnorm{A}_F^2.\]
\end{lemma}
\begin{proof}
We use the decomposition described in ~\eqref{eq:decomposition} and \eqref{eq:ydist}.
In particular, let $\pi$ be the distribution of vector $y$ and $\Pr_{0,J^*x}(x)\sim \exp(\langle J^*y,x\rangle)$ the corresponding product measure. We can have that

\[\E(x^\top A^{\top}Ax)-\|\E[Ax]\|^2=\sum_{i=1}^n\Var({A_i}x)\ge\sum_{i=1}^n\E_{y\sim\pi}\Var_{x\sim\nu(y)}({A_i}x)=\sum_{i=1}^n\sum_{j=1}^n\E_{y\sim\pi}A_{ij}^2(\sech^2((J^*y)_j))\]

We know that $(J^*y)=(J^*x)+((J^*)^{1/2}g)$ is a  multivariate Gaussian distribution with means $J^*x$ and covariance matrix $J^*$. So for each $(J^*y)$, we have mean $J^*_ix$ and variance is $J^*_{ii}$. Therefore, our plan is to prove that $J^*_ix$ and $\calN(0,J^*_{ii})$ are small with high probability.

Consider $P$ to be the Glauber Dynamics of the Ising model. Consider $f(x)=\sum_{j=1}^n J^*_{ij}x_j$. So, because of the distribution, we know that $J_i^*x$ has zero mean. Also, we can calculate that
\[v(f)\le \frac{1}{n}\sum_{i=1}^n 4(J^*_{ij})^2\le \frac{4\lambda^2}{n}.\]

We know that $\mu$ satisfies Glauber-MLSI$(\rho)$. Therefore, by Lemma \ref{lem: concentration of linear}, we can bound the probability that ${y_j}$ is large:
\begin{align*}
    &\mathbb P((J^*y)_j>a)\le\mathbb P(f(x)>a/2)+\mathbb P(\mathcal{N}(0,J_{ii})>a/2)\\
    \le &\exp(-\frac{ a^2}{8\rho\lambda^2})+\exp(-\frac{a}{8})\le 2\exp(-\frac{a^2}{8\rho\lambda^2}).
\end{align*}

Therefore, we take $a= 4\sqrt{\rho}\lambda$, we can get that the probability that $(J^*y)_j\le 4\sqrt{\rho}\lambda$ is at least $1-2/e^2>1/2$. 
Therefore, we have
\[\sum_{i=1}^n\Var({A_i}x)=\sum_{i=1}^n\sum_{j=1}^n\E_{y\sim\pi}a_{ij}^2(1-\tanh^2((J^*y)_j))\ge \frac{1}{2}\lr{1-\tanh^2\lr{4\lambda\sqrt{\rho}}}\lrnorm{A}_F^2.\]
\end{proof}

We are now ready to prove our main Lemma about lower-bounding the second derivative of the pseudolikelihood. The proof essentially involves addressing~\ref{it:first}, i.e., showing that $\sech(J_iX)^2$ is lower bounded by a constant with high probability.
Since we only know a bound on $\|J\|_{op}$, in general, this quantity could be very small. We use concentration results for second-degree polynomials to relate $\sech(J_iX)$ with $\sech(J_i^*X)$, which does not depend on the matrix direction $J$.
The result is stated for one sample for simplicity, but we'll see how to apply it for multiple samples in the sequel.

\begin{lemma}\label{lem:anticoncentration_one_sample}
    Suppose $X \in \{-1,1\}^n$ is a sample drawn from an Ising model satisfying Assumption~\ref{ass:spectral_bounded}. Let $A \in \mathcal{S}_0^n$ be a symmetric matrix with $\sqrt{\|A\|_F^2+\|\E[Ax]\|^2} \leq M$ for some constant $M > 0$. Let us also denote $J' = J^* + A$. Then, for any $J$ that lies in the line segment connecting $J^*$ and $J'$, we have 
    \[
    \frac{\partial^2 \phi(J)}{\partial A^2} \geq K_2 \cdot ({\|A\|_F^2+\|\E[Ax]\|^2}) \cdot \min_{i \in [n]} \sech^2(|J_i^*X^{(k)}| + h_i + K_1M )
    \]
    with probability at least \[1 - \exp(-\frac{c}{\rho^2}\cdot \min(t_1^2,t_1)) - \exp(-\frac{c}{\rho^2}\cdot \min(t_2^2,t_2)),\]
    where $c$ is an absolute constant as in Lemma \ref{lem:degree2concentration}, and the expression $K_1,K_2$ are as follows:
    \[K_1=\sqrt{t_1+2\rho}\]
    \[K_2=\frac{1}{2}\lr{1-\tanh^2\lr{4\lambda\sqrt{\rho}}}-t_2\]
\end{lemma}
\begin{proof}
    We have that
    \begin{align*}
        \frac{\partial^2 \phi(J)}{\partial A^2} &= \sum_{i=1}^n (A_iX)^2 \sech^2(J_iX) \\
        &\geq \sum_{i=1}^n (A_iX)^2 \sech^2(|J_i^*X| + |(J_i - J_i^*)X|)\\
        &= \sum_{i=1}^n (A_iX)^2 \sech^2(|J_i^*X| + |A_iX|)
    \end{align*}
    The last inequality follows since $|(J_i - J_i^*)X| \leq |((J')_i - J_i)X| = |A_iX|$, since $J$ lies in the segment connecting $J',J^*$. 
    Now, let us use Lemma~\ref{lem:degree2concentration} for the quadratic form $x^\top A^\top Ax$. When substituting $t \gets t (\| A\|_F^2+\|\E[AX]\|^2)$, this gives
    \begin{equation}\label{eq:concentration2degree}
    X^\top A^\top A X - \E[X^\top A^\top A X] = \sum_{i=1}^n (A_iX)^2 - \E[X^\top A^\top A X] \leq  t(\|A\|_F^2+\|\E[AX]\|^2).
    \end{equation}
    
    This bound of $X$ means in particular that
    \[
    \sum_{i=1}^n (A_iX)^2 \leq \E[X^\top A^\top A X] + t(\|A\|_F^2+\|\E[AX]\|^2). 
    \]
    
    By Lemma \ref{lem:degree2concentration}, this happens with probability at least 
    \[1 - \exp(-c/\rho^2\min(\frac{t^2(\|A\|_F^2+\|\E[AX]\|^2)^2}{\|A^\top A\|_{F}^2+\|\E[A^\top AX]\|^2},\frac{t(\|A\|_F^2+\|\E[AX]\|^2)}{\|A^\top A\|_{op}})).\]
    
    Using the well known properties $\|A^\top A\|_{op}=\|A\|_{op}^2$, $\|\E[A^\top AX]\|^2=\|A^{\top}\E[ AX]\|^2\le \|A\|_{op}\|\E[ AX]\|^2$ and $\|A^{\top }A\|_F\le\|A\|_F\|A\|_{op}$, we can lower bound the probability by
    \begin{align*}
        &1 - \exp(-c/\rho^2\min(\frac{t^2(\|A\|_F^2+\|\E[AX]\|^2)^2}{\|A^\top A\|_{F}^2+\|\E[A^\top AX]\|^2},\frac{t(\|A\|_F^2+\|\E[AX]\|^2)}{\|A^\top A\|_{op}}))\\
        \ge&1 - \exp(-c/\rho^2\min(\frac{t^2(\|A\|_F^2+\|\E[AX]\|^2)^2}{\|A \|_{op}^2(\|A\|_{F}^2+\|\E[ AX]\|^2)},\frac{t(\|A\|_F^2+\|\E[AX]\|^2)}{\|A\|_{op}^2}))\\
        =&1 - \exp(-c/\rho^2\frac{\|A\|_F^2+\|\E[AX]\|^2}{\|A\|_{op}^2}\min(t,t^2)).
    \end{align*}
    
    To deal with upper bound of $\E[X^\top A^\top AX]$, we use lemma~\ref{lem: poincare} (since $\mu$ satisfy Glauber-MLSI$(\rho)$, hence also Po$(2\rho)$).
    \begin{align*}
        &\sum_{i=1}^n (A_iX)^2 \leq \E[X^\top A^\top A X] + t(\|A\|_F^2+\|\E[AX]\|^2)\\
        \leq &(1+t)(\|\E[AX]\|^2)+\lp(t +2\rho\rp) \|A\|_F^2\le(t+2\rho)(\|\E[AX]\|^2+\|A\|_F^2)
    \end{align*}

    with probability at least $1 - \exp(-\frac{c(\|A\|_F^2+\|\E[AX]\|^2)}{\rho^2\|A\|_{op}^2}\cdot \min(t^2,t))$. Denote $K_1=\sqrt{t+2\rho}$.
    It follows that with the same probability we have $|A_iX| \leq K_1 M $ for all $i$.
    Thus, we have that
    \begin{align*}
        \frac{\partial^2 \phi(J)}{\partial A^2} &\geq \sum_{i=1}^n (A_iX)^2 \sech^2(|J_i^*X| + |A_iX|)\\
        &\geq \sum_{i=1}^n (A_iX)^2 \sech^2(|J_i^*X| + K_1M)\\
        &\geq \min_{i \in [n]} \sech^2(|J_i^*X| + K_1 M) \sum_{i=1}^n (A_iX)^2
    \end{align*}
    
    We now need to lower bound $\E[X^\top A^\top AX]$ in order to get a high probability lower bound for $\sum_{i=1}^n (A_i X)^2$. Using Lemma \ref{lem:degree2concentration} we have 
    \begin{align*}
    X^\top A^\top A X &\ge \E[X^\top A^\top A X] - t(\|A\|_F^2+\|\E[AX]\|^2) \\
    &\ge (\frac{1}{2}(1-\tanh^2(4\lambda^2\sqrt\rho))-t)({\|A\|_F^2+\|\E[AX]\|^2})  
    \end{align*}
    
    with probability at least $1 - \exp(-\frac{c(\|A\|_F^2+\|\E[AX]\|^2)}{\rho^2\|A\|_{op}^2}\cdot \min(t^2,t))$. 
    Denote $K_2=\frac{1}{2}\lr{1-\tanh^2\lr{4\lambda\sqrt{\rho}}}-t_2$. Using the union bounds, and noticing that $\|A\|_{F}>\|A\|_{op}$, we can finish the proof.
\end{proof}

We now use Lemma~\ref{lem:anticoncentration_one_sample} to lower-bound the second derivative when we have multiple independent samples. The main issue is that we need to choose the parameters $t_1,t_2$ to ensure $K_2$ is bounded away from $0$, while at the same time obtaining a probability of failure that decays exponentially with the number of samples $l$. As we see, the particular block structure of the matrix in that case is crucial for the argument to go through. 

\begin{theorem}\label{thm:anticoncentration_multiple_samples}
    Suppose $X^{(1)}, \ldots, X^{(l)} \in \{-1,1\}^n$ are independent samples from an Ising model satisfying Assumption~\ref{ass:spectral_bounded}. Let $A \in \R^{n \times n}$ be a symmetric matrix with $\sqrt{{\|A\|_F^2+\|\E[Ax]\|^2}} \leq M\le 1$ for some constant $0<M<1$. Let us also denote $J_1 = J^* + A$. Then, for any $J$ that lies in the line segment connecting $J^*$ and $J_1$, we have 
    \[
    \frac{\partial^2 \phi(J^{(l)})}{\partial (A^{(l)})^2} \geq r(\lambda,\rho) \cdot l \cdot ({\|A\|_F^2+\|\E[Ax]\|^2}) \cdot \min_{k \in [l], i \in [n]} \sech^2(|J_i^*X^{(k)}_i| + K(\lambda,\rho)M+h_i)
    \]
    with probability at least $1 - \exp(- C(\lambda,\rho) \cdot l )$, where $r(\lambda,\rho), C(\lambda,\rho)$ are constants that depend only on $\lambda,\rho$. More specifically, 

    \[r(\lambda,\rho)=C(\lambda,\rho)^2/4;~~~ C(\lambda,\rho)=\frac{c}{48\rho^2}\cdot\sech(4\lambda\sqrt{\rho})^4;~~~K(\lambda,\rho)=\frac{10\rho^{5/4}\sqrt{\lambda}}{{c}}\]

    where $c$ is the constant defined in Lemma \ref{lem:degree2concentration}.
\end{theorem}
\begin{proof}
    Using the block structure of the matrix, we can write
    \begin{align*}
        \frac{\partial^2 \phi(J^{(l)})}{\partial (A^{(l)})^2} = \sum_{k=1}^l \underbrace{\sum_{i=1}^n (A_iX^{(k)})^2 \sech^2(J_iX^{(k)}+h_i)}_{G_k}
    \end{align*}
    The random variables $G_1,\ldots,G_k$ are independent and identically distributed. We use Lemma~\ref{lem:anticoncentration_one_sample} with the substitution
    \[t_2=\frac{\sech(4\lambda\sqrt{\rho})^2}{4},~~~t_1=\frac{\log(\frac{2\rho^2}{c t_2^2})\rho^2}{c}.\] 
    We can clearly see that $t_2<1/4$ and $t_1>1$ (we, without loss of generality, take the absolute constant $c$ in Lemma \ref{lem:degree2concentration} smaller than $1$). Therefore, we have 
    
    \begin{align*}
        &1 - \exp(-c/\rho^2\cdot \min(t_1^2,t_1))-\exp(-c/\rho^2\min(t_2^2,t_2))\\
        = &1 - \exp(-c/\rho^2\cdot t_1)-\exp(-c/\rho^2\cdot t_2^2)\\
        = &1 - \exp(-c/\rho^2\cdot \frac{\log(\frac{2\rho^2}{c t_2^2})\rho^2}{c})-\exp(-c/\rho^2\cdot t_2^2)\\
       = &1 - c/\rho^2\cdot t_2^2/2-\exp(-c/\rho^2\cdot t_2^2)
    \end{align*}
    
    We know that for small $x<1/4$, $1-e^{-x}-x/2\ge x/3$. So, for all $k$, with probability at least
    
    \[p := c/\rho^2 \cdot t^2_2/3=c/\rho^2 \cdot\frac{\sech(4\lambda\sqrt{\rho})^4}{16}\cdot\frac{1}{3}\ge c/\rho^2\cdot\sech(4\lambda\sqrt{\rho})^4/48,\]
    the following holds
    \[
    G_k \geq K_2 \cdot ({{\|A\|_F^2+\|\E[Ax]\|^2}}) \cdot \min_{i \in [n]} \sech^2\lp(|J_i^*X^{(k)}| + K_1 M+h_i \rp),
    \]
    
    where $K_1,K_2$ are defined in Lemma \ref{lem:anticoncentration_one_sample}. We have \[K_2=\frac{1}{2}\sech^2(4\lambda\sqrt{\rho})-t_2\ge\sech^2(4\lambda\sqrt{\rho})/4,\]
    
    and
    \[K_1=\sqrt{t_1+2\rho}\le 2\sqrt{t_1}=\frac{2\rho}{\sqrt{c}}\cdot\sqrt{\log(\frac{2\rho^2}{ct_2^2})}\le\frac{10\rho^{5/4}\sqrt{\lambda}}{c}.\]

    We will now use the independence of the different samples.
    Let's call the above event $E_k$, then $\indicator{E_k}$ is a Bernoulli random variable with probability at least $p$. 
    By the standard Chernoff bound for Bernoulli random variables, we have that for all $\delta \in (0,1)$
    \[
    \Pr\lp[\sum_{k=1}^l \indicator{E_k} \leq (1-\delta)lp\rp] \leq \exp(-\delta^2 lp/2),\]
    Choosing $\delta = 1/2$ yields that with probability at least $1 - \exp(-lp/8)$ at least an $lp/4$ fraction of the events $E_k$ will be satisfied. We conclude that with the same probability
    \begin{align*}
        &\frac{\partial^2 \phi(J^{(l)})}{\partial (A^{(l)})^2} \geq \sum_{k: \indicator{E_k}= 1} G_k \\
        \geq &p^2/4 \cdot l \cdot ({\|A\|_F^2+\|\E[Ax]\|^2}) \cdot \min_{k \in [l], i \in [n]} \sech^2\lp(|J_i^*X^{(k)}_i| + K_1 M+h_i\rp)
    \end{align*}

    Plugging in $p,K_1,M$ yields the result.
\end{proof}

\subsection{Learning in Frobenius norm}
After establishing the concentration of the first derivative and anti-concentration of the second derivative, we now combine them to show that we can learn in the Frobenius norm. 

The precise statement is given in the following lemma. Its purpose is to show that if for some matrix $J_1$ the norm $\|J^* - J_1\|_F$ is large, then with high probability $J_1$ will have a lower pseudo-likelihood value than $J^*$, which means that it will not be selected as the maximizer of the pseudolikelihood.
The proof essentially combines the concentration and anticoncentration properties of the two derivatives for a single matrix $J_1$. 

\begin{lemma}\label{lem:1dim_closeness}
Let $X$ be a sample drawn from an Ising model with interaction matrix $J^*$ satisfying Assumption~\ref{ass:spectral_bounded} and let $J_1 \in\mathcal{R}$ be a different matrix, with $A = J_1 - J^*$. Assume $\sqrt{{\|A\|_F^2+\|\E[Ax]\|^2}} \leq M<1$, for some $M > 0$. 
Then, there exist constants $C, C', r, K$ that depend only on $\rho,\lambda$, such that for any $t > 0$, with probability at least 
\[1 - \exp(- C \cdot l ) - 8/3 \exp\lp(- {C'} \cdot l \cdot \min(t^2, {t})\cdot M^2 \rp), \]
holds
\begin{align*}
    \phi(J_1) \geq \phi(J^*) +& r \cdot l \cdot  M^2  \cdot \min_{k \in [l], i \in [n]} \sech^2(|J_i^*X^{(k)}_i| 
    + KM + h_i) - t \cdot l \cdot M^2 . 
\end{align*}

Here, $C, r, K$ are the constants $C(\rho,\lambda),r(\rho,\lambda),K(\rho,\lambda)$ defined in Theorem \ref{thm:anticoncentration_multiple_samples}, and $C'=\frac{1}{2^{21}\rho^2\lambda^4}$ is the constant defined in Theorem \ref{thm:concentration_first_generic}.
\end{lemma}

\begin{proof}[Proof of Lemma~\ref{lem:1dim_closeness}]
    We begin as in Lemma 1 in \cite{dagan2021learning} by defining the function $J:[0,1]\mapsto \R^{n \times n}$ as $J(t) = (1-t)J^* + t J_1$. Let $A = J_1 - J^*$. By definition, then we have
    \[
    \frac{d \phi(J(t))}{dt} = \left.\frac{\partial \phi(J)}{\partial A}\right|_{J = J(t)} \quad, \quad \frac{d^2 \phi(J(t))}{dt^2} = \left.\frac{\partial^2 \phi(J)}{\partial A^2}\right|_{J = J(t)}\enspace.
    \]
    Now define
    $$
    \alpha = \min_{t \in [0,1]}\frac{d^2 \phi(J(t))}{dt^2} \quad, \quad\gamma = \left.\frac{d \phi(J(t))}{dt}\right|_{t = 0} 
    $$
    By Taylor's theorem, we have that
    \[
    \phi(J(1)) \geq \phi(J(0)) + \gamma + \frac{\alpha}{2}
    \]
    By the definition of $J(t)$, this is equivalent to
    \[
    \phi(J_1) \geq \phi(J^*) + \gamma + \frac{\alpha}{2}\enspace.
    \]
    Now, using Theorem~\ref{thm:anticoncentration_multiple_samples} immediately gives
    \[
    \alpha \geq r \cdot l \cdot \|A\|_F^2 \cdot \min_{k \in [l], i \in [n]} \sech^2(|J_i^*X^{(k)}_i| + KM)
    \]
    with probability at least $1 - \exp(- C \cdot l)$, where $r, C,K$ are constants that depend only on $\rho,\lambda$, as in Theorem \ref{thm:anticoncentration_multiple_samples}. 
    At the same time, we can apply Lemma~\ref{thm:concentration_first_generic}, where $A$ is now $A^{(l)}$ and so instead of $t\gets t({\|A\|_F^2+\|\E[Ax]\|^2})$ we have $t\gets tl M^2 $. After replacing  $t$ with $t \|A\|_F$ and notice that $\|A\|_{op}\le \|A\|_{F}\le 1$ we obtain
    \[
    \Pr[|\gamma| > t l  M^2 ] \leq \frac{8}{3} \exp\lp(- {C'} \cdot l \cdot \min(t,t^2) \cdot  M^2 \rp)
    \]
    
    Where $C'$ is the constant in Lemma \ref{lem:concentration_first_high_temp}. Thus, with probability at least $1 - \exp(- C \cdot l) - 8/3 \exp\lp(- C' \cdot l \cdot \min(t,t^2) \cdot  M^2 \rp) $ we have that
\begin{align*}
    \phi(J_1) \geq \phi(J^*) +& r \cdot l \cdot  M^2  \cdot \min_{k \in [l], i \in [n]} \sech^2(|J_i^*X^{(k)}_i| \\
    + &KM) - t \cdot l \cdot M^2 . 
\end{align*}
    And the proof is complete.
\end{proof}

We now have to establish a similar property as in Lemma~\ref{lem:1dim_closeness} but across all directions $J_1$ with high probability.
If we had this property, we could conclude that whatever matrix is returned from maximizing the pseudolikelihood function, it needs to be close to the true matrix $J^*$ in Frobenius norm. 
To do this, we utilize the fact that $\phi$ is Lipschitz with respect to the spectral norm.

\begin{lemma}[Lemma 10, \cite{dagan2021learning}]\label{lem:lipschitz}
    For any symmetric matrices $J_1,J_2 \in \R^{n \times n}$, and for $l$ samples,
    \[
    |\phi(J_1) - \phi(J_2)| \leq n\cdot l \cdot\|J_1 - J_2\|_{op}
    \]
\end{lemma}
We next define the important notion of an $\varepsilon$-net.

\begin{definition}
Given a metric space $(\mathcal{U},d)$ and $\varepsilon > 0$, we say that a subset $\mathcal{N} \subseteq \mathcal{U}$ is an \emph{$\varepsilon$-net} for $\mathcal{U}$ if for every $u \in \mathcal{U}$ there exists a $v \in \mathcal{N}$ such that $d(u,v) \leq \varepsilon$. The cardinality of the smallest possible $\varepsilon$-net is denoted by $\mathcal{N}(\mathcal{U},d,\varepsilon)$. We also refer to $\mathcal{N}(\mathcal{U},d,\varepsilon)$ as the $\varepsilon$-covering number of the set $\mathcal{J}$.  
\end{definition}

The strategy now will be to show that all matrices that are far from $J^*$ in Frobenius norm will have a higher pseudolikelihood value with high probability. Arguing simultaneously over all such matrices is a daunting task, since this is an infinite set. The strategy that was employed in \cite{dagan2021learning} was the following: first, we construct an $\varepsilon$-net to cover the entire space of matrices. By choosing $\varepsilon$ sufficiently small and using the Lipschitzness property, we can show that any point in the space has a pseudolikelihood value close to some point in the net. Consequently, if we can guarantee that with high probability all points in the net have pseudolikelihood value smaller than that of $J^*$, then the same should be true for all points in the set as well.

However, this approach cannot work in our case, because by Lemma~\ref{lem:1dim_closeness} we can only argue about the value of $\phi$ for points that are close to $J^*$ in Frobenius norm. Thus, our goal will be to show that there is a \emph{shell} of matrices of the form $\{J: \varepsilon \leq \|J^* - J\|_F \leq M\}$, such that all points in this shell have a higher pseudolikelihood value than $J^*$. It will then follow from the convexity of $\phi$ that the same is true for points outside of this shell as well. It would then follow that the true minimizer $\hat{J}$ should satisfy $\|\hat{J} - J^*\|_F \leq \varepsilon$, which is our final estimation bound. The challenge is to show that for $\varepsilon$ taking a relatively small value, this property will hold, as this will result in a small estimation error. To argue about that, we are aided by the fact that we can choose $l$ large enough to make $\varepsilon$ smaller than $M$, so this shell is not empty. 
The details are given below. 

\begin{theorem}\label{thm:frobenius_estimation_high_temp}
    Let $X_1,\ldots,X_l$ be independent samples drawn from an Ising model with interaction matrix $J^*$ satisfying Assumption~\ref{ass:spectral_bounded}. Let $\hat{J} \in \R^{n \times n}$ be the estimate of $J^*$ that is obtained by maximizing the pseudo-likelihood function \eqref{eq:pseudolikelihood_function}, i.e.
    \[
    \hat{J}^{(l)} := \argmax_{J^{(l)} \in \mathcal{R}^{(l)}} \phi(J;X)\enspace.
    \]
    Then, for any $\varepsilon \in (0,1)$, if 
    \[
    l = \tilde{O}\lp(\frac{n^2 \log (1  / (\delta\varepsilon))}{\varepsilon^2}\rp)
    \]
    then with probability at least $1-\delta$ 
    \[
    \sqrt{\|J - J^*\|_F^2+\|\E[(J-J^*)X]\|^2}\leq \varepsilon\enspace,
    \]
    where $\tilde O$ hides sub-polynomial factors of $n$, and other terms about $\lambda, \rho$ and $h_{\max}$. 
\end{theorem}

\begin{proof}
    Let $\varepsilon \in (0,1/2)$ and set $M = 2\varepsilon$. Let us define the shell
    \[
   \mathcal{R}_\varepsilon := \{J \in\mathcal{R}: \varepsilon \leq \sqrt{\|J - J^*\|_F^2+\|\E[(J-J^*)X]\|^2} \leq 2 \varepsilon\}
    \]
    Our goal will be to choose $l$ such that with high probability $\hat{J} \notin \mathcal{R}_\varepsilon$. 
    First of all, since $\mathcal{R}_\varepsilon \subseteq \mathcal{R}$, we have that for any $\theta > 0$ 
    \[
    \mathcal{N}(\mathcal{R}_\varepsilon ,\|\cdot\|_{op}, \theta) \leq \mathcal{N}(\mathcal R'_\varepsilon ,\|\cdot\|_{op}, \theta)
    \]
    Here, we have defined the set $\mathcal{R}'_\varepsilon := \{J \in\mathcal{R}:\|J - J^*\|_F \leq 2 \varepsilon\}$.

    Thus, it suffices to bound $\mathcal{N}(\mathcal{R},\|\cdot\|_{op}, \theta)$. To that end, we can view $\mathcal{R}$ as a subset of $n^2$-dimensional Euclidean space, where the basis vectors are matrices $\{E_{ij}\}_{i,j=1}^n$, where $E_{ij}$ has the $i,j$ entry $1$ and the rest $0$. Thus, we seek to cover the ball of matrices with spectral radius at most $1$ with balls of radius $\theta$. Notice that the ball with spectral radius $\theta$ is contained inside the ball with Frobenius radius $\theta$. Since the ball of spectral radius $1$ is a centrally symmetric convex body, we can apply Corollary 4.1.15 from \cite{artstein2021asymptotic}, which is based on a standard volume argument, to obtain 
    \begin{equation}\label{eq:cover_bound}
    \mathcal{N}(\mathcal{R}_\varepsilon,\|\cdot\|_{op}, \theta) \le\mathcal{N}(\mathcal{R}'_\varepsilon,\|\cdot\|_{op}, \theta) \le\mathcal{N}(\mathcal{R}'_\varepsilon,\|\cdot\|_{F}, \theta) \leq \lp(1 + \frac{2\varepsilon}{\theta}\rp)^{n^2}
    \end{equation}
    Let $\mathcal{U} \subseteq \mathcal{R}$ be an $\theta/n$-net of $\mathcal{R}_\varepsilon$ of cardinality $ \mathcal{N}(\mathcal{R}_\varepsilon ,\|\cdot\|_{op}, \theta/n) $, where $\theta$ will be chosen in the sequel.
    Then, applying Lemma~\ref{lem:1dim_closeness} and a union bound gives that for all $t > 0$, with probability at least 
     $$
     1 - \mathcal{N}(\mathcal{R}'_\varepsilon,\|\cdot\|_{op}, \theta/n)\lp(\exp(- C \cdot l) + 8/3 \exp\lp(- {C'} \cdot l \cdot \min(t^2, {t})\cdot\varepsilon^2)\rp)\rp).
     $$
     We have that the following event occurs.
    \begin{align*}\label{eq:net}
    \mathcal{E} := \{\phi(J) \geq \phi(J^*)+ l&\cdot(r\cdot \min_{k \in [l], i \in [n]} \sech^2(|J_i^*X^{(k)}_i| + 2K\varepsilon+h_i) - 4t) \cdot M^2 \quad, \forall J \in \mathcal{U} \}\enspace.
    \end{align*}
    Here, $C,C',r,K$ are the constants in Lemma \ref{lem:1dim_closeness}. Now let us assume $\mathcal{E}$ happens. We now upper bound $|J_i^*X^{(k)}|$. Indeed, using Lemma~\ref{lem: concentration of linear} we have that for each $k$
    \[
    \Pr[|J_i^*X^{(k)}| \geq a] \leq \exp\lp(-\frac{1}{8\rho\lambda^2} a^2\rp)
    \]
    By union bound over all $i \in [n], k \in [l]$, we have
    \[
    \Pr[\max_{i \in [n], k \in [l]} |J_i^*X^{(k)}| \geq a ] \leq l\cdot n \cdot\exp\lp(-\frac{1}{8\rho\lambda^2} a^2\rp)
    \]
    Thus, by choosing $a = C''\lambda\sqrt{\rho\log (n/\delta)}$ for some suitable constant $C''$, we get that with probability at least $1 - \delta/2$ (we remind the readers that $h_{\max}$ is the maximum value of $|h_i|$)
    \begin{align*}
        &\min_{k \in [l], i \in [n]} \sech^2(|J_i^*X^{(k)}_i| + 2K\varepsilon)
        \geq\sech^2\lp(C''\lambda\sqrt{\rho\log n} + 2K\varepsilon+h_{\max}\rp) := \xi(n)^{-1} \enspace.
    \end{align*}

    We notice that $\xi(n)$ grows sub-polynomially, i.e. $\xi(n) = o(n^r)$ for any $r > 0$.
    Thus, we conclude that with probability at least 
    \[
    1 - \mathcal{N}(\mathcal{R}'_\varepsilon,\|\cdot\|_{op}, \theta/n)\lp(\exp(- C \cdot l) + 8/3 \exp\lp(- {C'} \cdot l \cdot \min(t^2, {t})\cdot\varepsilon^2)\rp)\rp) - \frac{\delta}{2},
    \]
    we have that
    \[
    \phi(J) \geq \phi(J^*) + r \cdot l \cdot \varepsilon^2 \cdot \xi(n)^{-1} - 4t \cdot l \cdot\varepsilon^2\quad, \forall J \in \mathcal{U}\enspace.
    \]
    Now, choosing $t\le r\xi(n)^{-1}/8$. 
    Notice that $r$ is really small compare to $C$ and $C'$, we have $\exp(-C\cdot l)\le \exp\lp(- {C'} \cdot l \cdot \min(t^2, {t})\cdot\varepsilon^2)\rp)=\exp\lp(- {C'} \cdot l \cdot t^2\cdot\varepsilon^2)\rp)$ and  we have that with probability at least
    \[
    1 - \frac{11}{3}\mathcal{N}(\mathcal{R}'_\varepsilon,\|\cdot\|_{op}, \theta/n) \exp\lp(-C'\cdot r^2\cdot\xi(n)^{-2} \cdot l \cdot \varepsilon^2/64\rp) - \frac{\delta}{2},
    \]
    the following event holds
    \[
    \mathcal{E}' := \lp\{\phi(J) \geq \phi(J^*) + \frac{1}{2}r \cdot l \cdot \varepsilon^2 \cdot \xi(n)^{-1}  \quad, \forall J \in \mathcal{U} \rp\}\enspace.
    \]
    Let us now see how we should choose $l$ so that 
    \[
    \mathcal{N}(\mathcal{R}_{\varepsilon}',\|\cdot\|_{op}, \theta/n) \exp\lp(-C'\cdot r^2\cdot\xi(n)^{-2} \cdot l \cdot \varepsilon^2/64\rp) < \frac{\delta}{2} \enspace.
    \]
    Using the covering number bound \eqref{eq:cover_bound}, it suffices to pick
    \begin{equation}\label{eq:l_bound}
    l \geq\frac{64\xi(n)\lp(n^2 \log (1+ 2\varepsilon n/\theta) + \log (10/\delta)\rp)}{C'r^2\varepsilon^2}
    \end{equation}
    Finally, let us see how to choose $\theta$. We would like to show that, if event $\mathcal{E}'$ holds, then for an arbitrary element 
     $J \in \mathcal{R}_\varepsilon$ we have $\phi(J) > \phi(J^*)$. By definition, there exists $\overline{J} \in \mathcal{U}$ with $\|J - \overline{J}\|_{op} \leq \theta/(2n)$. This, combinined with Lemma \ref{lem:lipschitz}, implies that
    \[
    \phi(J) \geq \phi(\overline{J}) - \frac{\theta}{2} \geq \phi(J^*) + \frac{1}{2}r \cdot l \cdot \varepsilon^2 \cdot \xi(n)^{-1}- \frac{\theta}{2}\cdot l 
    \]
    The last quantity is $> \phi(J^*)$ if we pick $\theta = \frac{1}{2}r \cdot\varepsilon^2 \cdot\xi(n)^{-1}$.
    Thus, \eqref{eq:l_bound} becomes
    \[
    l \geq  \frac{\xi(n)\lp(n^2 \log (1+2n\cdot\xi(n) /(r\varepsilon)) + \log (10/\delta)\rp)}{C'r^2\varepsilon^2}
    \]
    Thus, with this choice of $l$, we know that with probability at least $1- \delta/2$ 
    \[
    \phi(J) > \phi(J^*) \quad, \forall J \in \mathcal{R}_\varepsilon\enspace.
    \]
    Call the above event $\mathcal{E}''$. 
    We argue that if $\mathcal{E}''$ holds, then for all $J$ with $\|J-J^*\|_F > 2\varepsilon$ we have $\phi(J) > \phi(J^*)$. Indeed, for any such $J$, the line segment connecting $J$ to $J^*$ intersects $\mathcal{R}_\varepsilon$ in at least one point, call it $J'$.
    This is because the inequalities defining $\mathcal{R}_\varepsilon$ scale by a constant as we move from $J$ to $J^*$ (see also Figure~\ref{fig:shell}).
    Thus, for some $t \in (0,1)$ we can write
    $J'= (1-t) J^* + t J$. Now, since $\phi$ is a convex function, it holds
    \[
    \phi(J') \leq (1-t)\phi(J^*) + t \phi(J) < (1-t)\phi(J') + t \phi(J)
    \]
    The last inequality holds by definition of event $\mathcal{E}''$. By rearranging we get $\phi(J) > \phi(J^*)$. Thus, $\hat{J}$ can only lie inside the Frobenius norm sphere of radius $\varepsilon$ around $J^*$, which concludes the proof.
\end{proof}

\begin{remark}
    If the reader is curious about the precise form of $l$, we can express $l$ as ($C_0$ is a universal constant) 
    \[l\ge \exp\lp(C_0\lp(\lambda\sqrt{\log(n/\delta)\cdot\rho}+{\rho^{5/4}{\lambda}}\varepsilon\rp)\rp)\cdot e^{2h_{\max}}\cdot h_{\max}\cdot\frac{n^2\log(1/(\delta\varepsilon)}{\varepsilon^2}.\]
\end{remark}

\section{Learning Ising Models with Bounded Width}\label{sec: bounded width}

In this section, our goal will be to establish sample complexity guarantees for learning Ising Models of bounded width in TV distance. We say an Ising Model has width bounded by $M>0$, if and only if $\|J^*\|_\infty \leq M$. This enables us to handle the second derivative more easily since the $\sech$ terms are always lower bounded by a constant that depends on $M$.

\begin{assumption}\label{ass:bounded_width}
    In this section, we assume $h=0$ and $\mathcal{R}=\{J^* \in \mathcal{S}_0^n:\|J^*\|_\infty \leq M\}$ for some $M> 0$. 
\end{assumption}

For this Section, $\mathcal{R} \subseteq \mathcal{S}_0^n$ will denote the set of matrices with $\|A\|_\infty \leq M$.
In \cite{dagan2021learning}, it was established that we can learn the interaction matrix in the Frobenius norm. Here, we will need a more refined analysis, which will result in stronger guarantees that will enable us to bound the total variation distance between the estimated and the true model.

For the reader's convenience, we remind some important notation that will be used in this section. For a symmetric matrix $J \in \R^{n \times n}$ and a subset $I \subseteq [n]$, we denote by $J_I \in \R^{|I|\times n}$ the matrix consisting only of the rows of $J$ that are indexed by elements in $I$. We also denote $J_{II} \in \R^{|I|\times |I|}$ the submatrix with rows and columns indexed by $I$. For non-square matrices, the Frobenius norm extends in the usual fashion
\[
\|J_I\|_F^2 = \sum_{i \in I}\sum_{j=1}^n J_{ij}^2.
\]

We start by briefly highlighting some of the technical tools used in \cite{dagan2021learning}, as they will prove useful in our case as well.
An important observation is that we can select $O(\log n)$ subsets of nodes, such that each submatrix of $J^*$ satisfies Dobrushin's condition. Furthermore, we require that each node belongs to a constant fraction of these subsets. Formally, the following result was proven in \cite{dagan2021learning}. 

\begin{lemma}[Lemma 2 from \cite{dagan2021learning}]\label{lem:conditioning}
    Let $J^* \in \R^{n \times n}$ be a symmetric matrix with $\|J^*\|_\infty \leq M$ and let $\eta \in (0,M)$. Then, there exist subsets $I_1, \ldots, I_r$ with $r \leq C M^2 \log n / \eta^2$, such that the following properties hold.
    \begin{enumerate}
        \item For all $i \in [n]$
        \[
        \lp|j \in [r]: i \in I_j\rp| = \lp\lceil \frac{\eta r}{8 M} \rp\rceil\enspace.
        \]\label{prop:equal_splitting}
        \item For all $j \in [r]$, $\|J^*_{I_jI_j} \|_\infty \leq \eta$. \label{prop:high_temp}
    \end{enumerate}
\end{lemma}

This Lemma will allow us to split the first derivative sum into terms, where each term is a sum over the nodes of each subset. Then, for each subset $I_j$, by property ~\ref{prop:high_temp}, conditioned on the values $X_{-I_j}$, the model is in high temperature, so we can apply concentration bounds that are valid in that case. 
For simplicity, for $j \in [r]$ we use the notations
\begin{align*}
\phi_j (J) = \sum_{k=1}^l\sum_{i \in I_j} \lp(\log \cosh(J_iX^{(k)}) X_i^{(k)} J_i X^{(k)} + \log 2\rp)\\
\frac{\partial \phi_j(J^*)}{\partial A} := \sum_{k=1}^l\sum_{i \in I_j} (A_iX^{(k)})(\tanh(J_i^*X^{(k)})-X^{(k)}_i)\\
\frac{\partial^2 \phi_j(J)}{\partial A^2} := \sum_{k=1}^l \sum_{i \in I_j} \sech^2(J_iX^{(k)}) (A_iX^{(k)})^2
\end{align*}
The above are random variables, but we omit the dependence on $X$ for simplicity. 
We will argue about each component $\phi_j$ separately, conditioned on the variables $X_{-I_j}$. 
By property~\ref{prop:equal_splitting}, we have that 
\begin{equation}\label{eq:splitting}
    \phi(J) = \frac{8M}{\eta r} \sum_{j=1}^r \phi_j(J)
\end{equation}

Our goal will be to show that this bound of the first derivative is of the same order as $l \cdot \E[\|AX\|^2]$ with high probability, which is a quantity independent of the conditioning.
To do that, we will bound the deviation between the empirical and true mean of the variable $\|\E[AX^{(k)}|X^{(k)}_{-I}]\|^2$, uniformly over all matrices $A$ with small Frobenius norm. 
The challenge here is that $X_{-I_j}$ comes from a model at low temperature, so we do not have information about its concentration properties. If we naively use the Chernoff bound, then the large magnitude of $\|\E[AX^{(k)}|X^{(k)}_{-I}]\|^2$ will incur a large concentration radius, which, combined with a union bound over a high-dimensional subset of matrices, will result in high sample complexity. Instead, we notice that because of the special structure of the random variable, it is enough to upper bound the deviation of the random matrix $\E[X|X^{(k)}_{-I_j}]\E[X|X^{(k)}_{-I_j}]^\top$ from its mean, which can be done using matrix concentration results. This avoids the costly union bound over the set of matrices and results in a polynomial reduction in the number of samples required.
Let us introduce the set of matrices of small Frobenius norm
\[
\mathcal{A}_\epsilon := \{J \in \R^{n \times n} : \|A\|_F \leq \epsilon\},
\]
The details are given in the following Lemma.

\begin{lemma}\label{lem:matrix_conc}
    Suppose $X^{(1)}, \ldots, X^{(l)} \in \{-1,1\}^n$ are independent samples from an Ising model with interaction matrix $J^*$ satisfying $\|J^*\|_{\infty} \leq M$ and zero external field.
    Then, for any $t > 0$ we have that with probability at least 
    \[
    1 - 2n \exp\lp(-\frac{lt^2/2}{4n^2 + 2nt/3}\rp),
    \]
    the following holds
    \[
    \lp|\frac{1}{l}\sum_{k=1}^l \|\E[ AX^{(k)}|X^{(k)}_{-I}]\|^2 - \E\lp[\|\E[AX|X_{-I}]\|^2\rp]\rp| \leq t \|A\|_F^2 \enspace, \forall A 
    \]
\end{lemma}
\begin{proof}
    We first notice that we can write
    \begin{align*}
    \|\E[AX|X_{-I}]\|^2 &=
    \E[X|X_{-I}]^\top A^\top A \E[JX|X_{-I}] = \mathrm{Tr}\lp( \E[X|X_{-I}]^\top A^\top A \E[JX|X_{-I}]\rp)\\
    &= \mathrm{Tr}(A \E[X|X_{-
    I}]\E[X|X_{-
    I}]^\top A^\top)\enspace.
    \end{align*}
    Let $S=\E[X|X_{-
    I}]\E[X|X_{-
    I}]^\top$ be this random matrix and 
    denote by $S^{(k)} := \E[X|X^{(k)}_{-
    I}]\E[X|X^{(k)}_{-
    I}]^\top$ and $S = \E\lp[\E[X|X_{-
    I}]\E[X|X_{-
    I}]^\top\rp]$ the $l$ independent samples from the distribution of $S$.
    Then, we can write the difference between the empirical and the true mean as follows
    \begin{align*}
        \lp|\frac{1}{l}\sum_{k=1}^l \|\E[AX^{(k)}|X^{(k)}_{-I}]\|^2 - \E\lp[\|\E[AX|X_{-I}]\|^2\rp]\rp|
        =\lp|\mathrm{Tr}\lp(A\lp(\frac{1}{l}\sum_{k=1}^lS^{(k)}-\E[S]\rp)A^\top\rp)\rp|
    \end{align*}
    Assume momentarily that we somehow know that
    \begin{equation}\label{eq:concentration_property}
        \lp\|\frac{1}{l}\sum_{k=1}^lS^{(k)}-\E[S]\rp\|_2 \leq t
    \end{equation}
    Then, by the previous calculation and the definition of $A \in \mathcal{A}_\epsilon$, we would have
    \begin{align*}
    \lp|\mathrm{Tr}\lp(A\lp(\frac{1}{l}\sum_{k=1}^lS^{(k)}-\E[S]\rp)A^\top\rp)\rp| &\leq \sum_{i=1}^n\left|A_i^{\top}\lp(\frac{1}{l}\sum_{k=1}^lM^{(k)}-\E[M]\rp)A_i\right|\\
        &\le t \cdot \sum_{i=1}^n\left\|A_i\right\|_2^2= t \cdot \|\hat J-J^*\|_F^2 \enspace.
    \end{align*}
    In the above, we have used $A_i$ to denote the $i$-th row of matrix $A$, together with the fact that $A$ is symmetric. 
    Thus, to prove the claim, it suffices to establish \eqref{eq:concentration_property} with high probability.
    We turn our attention to that task now.

    First of all, we notice that $S^{(k)}$ are sampled independently from the same distribution of matrices. Since $X$ is a binary vector, each entry of the random vector $\E[AX|X_{-I_j}]$ lies within $[-1,1]$.
    Thus, so each $S^{(k)}$ is a symmetric rank-$1$ matrix with all its entries bounded by $1$ in absolute value.
    Therefore we know that $\|S^{(k)}-\E [S]\|_{2}\le 2n$. And finally, $\|\E[(S^{(k)}-\E [S])^2]\|_{2}\le 4n^2$. Therefore, if we use the matrix Bernstein concentration inequality (see Theorem 1.6.2 in \cite{tropp2015introduction}), we can have
    \[
    \Pr\left[\left\|\frac{1}{l}\sum_{k=1}^lS^{(k)}-\E[S]\right\|_{op}>t\right]=2n\exp\lp(\frac{-lt^2/2}{4n^2+2nt/3}\rp)\enspace.
    \]
    This concludes the proof.
\end{proof}

The next step will be to use this concentration property to obtain a uniform upper bound for the first derivative of the pseudo-likelihood. 
We start by proving such a bound for a single direction. 
Let us define the following event, which depends on the values of $X^{(k)}_{-I_j}$ for $k=1,\ldots,l$. 
\[
E_{j,u} := \lp\{    \lp|\frac{1}{l}\sum_{k=1}^l \|\E[(\hat J-J^*)X^{(k)}|X^{(k)}_{-I_j}]\|^2 - \E\lp[\|\E[(\hat J-J^*)X|X_{-I_j}]\|^2\rp]\rp| \leq u \cdot \|A\|_F^2 \enspace, \forall A \rp\}
\]

\begin{lemma}\label{lem:first_der_bound}
    Suppose $X^{(1)}, \ldots, X^{(l)} \in \{-1,1\}^n$ are independent samples from an Ising model with interaction matrix $J^*$ satisfying Assumption~\ref{ass:bounded_width}. Let $A \in \R^{n \times n}$ be a symmetric matrix with $\|A\|_\infty \leq M$. Suppose we condition on the values of $X_{-I_j}^{(k)}$ for all $k$ and that these values are such that $E_{j,u}$ holds. Then we have that with probability at least 
\[
    1 - \exp\lp(- c \min(t,t^2) \lp(l \cdot\lp(\E[\|AX\|_2^2] \rp) \rp) \rp),
    \]
we have that
\[
\lp|\frac{\partial \phi_j(J^*)}{\partial A}\rp| \leq C t\cdot l \cdot\lp(\E[\|AX\|_2^2]\rp) \enspace.
\]
\end{lemma}
\begin{proof}
    We know that if we choose $\eta=1/3$ the distribution $X_{I_j}|X_{-I_j}$ satisfy $\|A_{I_j}\|_{op}\le 1$, satisfy Glauber-MLSI($6$) and Po($6$). First, we use Theorem \ref{thm:concentration_first_generic}. Consider a large Ising model, with diagonal matrix blocks $A_{I_j}$, and also external fields $A_{-I_{j},-I_j}X_{-I_j}^{(k)}$. This large Ising model is a tensor product of i.i.d. Ising models of $X_{I_j}|X_{-I_j}^{(k)}$. Therefore, we have, with at least
    $$1-\frac{8}{3}\exp(-C\min(t,t^2)(l\cdot \|A_{I_j}\|^2_F+\sum_{i=1}^l\|\E[A_{I_j}X^{(k)}|X_{i_j}^{(k)}]\|^2))$$
    probability, we have
\[
\lp|\frac{\partial \phi_j(J^*)}{\partial A}\rp| \leq t \cdot\lp(l\cdot \|A_{I_j}\|_F^2 + \sum_{i=1}^l\|\E[A_{I_j}X^{(k)}|X_{i_j}^{(k)}]\|^2\rp).
\]
    
    If $E_{j,u}$ holds, we have, with at least probability
\[
    1 - \exp\lp(- c \min(t,t^2) \lp(l \cdot\lp(\|A_{I_j}\|_F^2 + \E[\|\E[AX|X_{-I_j}]\|]^2 - u\cdot \|A\|_F^2\rp) \rp) \rp),
    \]
    
we have that
\[
\lp|\frac{\partial \phi_j(J^*)}{\partial A}\rp| \leq t\cdot l \cdot\lp(\|A_{I_j}\|_F^2 + \E[\|\E[AX|X_{-I_j}]\|]^2 + u \cdot \|A\|_F^2\rp) 
\]

We now use the following Lemma, which connects the variance of linear functions of Ising models in high temperature with the Frobenius norm and has been repeatedly used in our analysis so far.

Thus, we can write
\begin{align*}
    \E[\|AX\|_2^2] &= \sum_{i=1}^n \E[(A_i^\top X)^2] \\
    &= \sum_{i=1}^n \E\lp[\E[(A_i^\top X)^2|X_{-I_j}]\rp]\\
    &= \sum_{i=1}^n \lp(\E\lp[\Var[(A_i^\top X)^2|X_{-I_j}]\rp]+ \E\lp[\lp(\E[A_i^\top X|X_{-I_j}]\rp)^2\rp]\rp)\\
    &= \sum_{i=1}^n \lp(\E\lp[\Var[(A_i^\top X)^2|X_{-I_j}]\rp] \rp)+ \E\lp[\lp\|\E[A X|X_{-I_j}]\rp\|_2^2\rp]
\end{align*}

Now, we can apply Lemma~\ref{lem: lower bound Jop<1} and Poincaré Inequality to the conditional Ising model conditioned on the values of $X_{-I_j}$, there are absolute constants $c_M, C_M$ such that 
\begin{align*}
c_M \|A_{I_j}\|^2_F \leq  \sum_{i=1}^n \lp(\E\lp[\Var[(A_i^\top X)^2|X_{-I_j}]\rp] \rp) \leq C_M \|A_{I_j}\|^2_F.
\end{align*}

Thus, the preceding bound implies

\[
\E[\|AX\|_2^2] = \Theta\lp(\|A_{I_j}\|_F^2 + \|\E[AX|X_{-I_j}]\|^2\rp)
\]

Thus, by adjusting the constants, we know that for any $A \in \R^{n \times n}$, with probability at least 
\[
    1 - \exp\lp(- c \min(t,t^2) \lp(l \cdot\lp(\E[\|AX\|_2^2] - u\cdot \|A\|_F^2\rp) \rp) \rp)
    \]
we have that
\[
\lp|\frac{\partial \phi_j(J^*)}{\partial A}\rp| \leq t\cdot l \cdot\lp(\E[\|AX\|_2^2]+ u \cdot \|A\|_F^2\rp) \enspace.
\]

Notice also that by Lemma~\ref{lem: lower bound Jop<1}, we can absorb the term $\|A\|_F^2$ inside $\E[\|AX\|^2]$ in the upper bound, with the possibility of incurring an extra constant factor. 
Also, by choosing $u$ to be a small enough constant, again by Lemma~\ref{lem: lower bound Jop<1} we can write 
\[
\E[\|AX\|_2^2] - u \|A\|_F^2 \geq \frac{1}{2} \E[\|AX\|_2^2] 
\]
Thus, under even $E_{j,u}$ for this choice of constant $u$, we have that with probability at least 
\[
    1 -\frac{8}{3} \exp\lp(- c \min(t,t^2) \lp(l \cdot\lp(\E[\|AX\|_2^2] \rp) \rp) \rp)
    \]
we have that
\[
\lp|\frac{\partial \phi_j(J^*)}{\partial A}\rp| \leq C t\cdot l \cdot\lp(\E[\|AX\|_2^2]\rp) \enspace.
\]
\end{proof}

We would like to prove that the bound of Lemma~\ref{lem:first_der_bound}  holds uniformly for all matrices $A$ in a given set. 
To do that, we first establish a Lipschitzness property of the first derivative of the pseudo-likelihood, similar to the one that was established for the pseudolikelihood itself in \cite{dagan2021learning}. 

\begin{lemma}\label{lem:lipschitzness_first_der}
    Let $A,B$ be two symmetric matrices with $\|A\|_\infty,\|B\|_\infty \leq M$. 
    Then
    \[
    \lp|\frac{\partial \phi_j(J^*)}{\partial A} - \frac{\partial \phi_j(J^*)}{\partial B}\rp| \leq 2 \cdot l \cdot n \cdot \|A-B\|_2
    \]
\end{lemma}
\begin{proof}
    We have that 
    \begin{align*}
        \lp|\frac{\partial \phi_j(J^*)}{\partial A} - \frac{\partial \phi_j(J^*)}{\partial B}\rp| &=
        \lp|\sum_{k=1}^l\sum_{i \in I_j} \lp(\lp(A_i-B_i\rp)X^{(k)}\rp)(\tanh(J_i^*X^{(k)})-X^{(k)}_i)\rp|\\
        &\leq 
        2 \sum_{k=1}^l \sum_{i=1}^n \lp|(A_i - B_i)X^{(k)}\rp|\\
        &\leq 2 \sqrt{n}\sum_{k=1}^l \sqrt{ \sum_{i=1}^n \lp|(A_i - B_i)X^{(k)}\rp|^2}\\
        &\leq 2 \cdot l \cdot n \cdot \|A-B\|_2 \enspace.
    \end{align*}
    In the last step, we used the fact that $\|X^{(k)}\|_2 \leq \sqrt{n}$ for all $k$ and the definition of the operator norm. 
\end{proof}

Our method of bounding the first derivative uniformly is similar to the one employed in \cite{dagan2021learning}. In particular, we construct a net over the set of matrices and then take a union bound over all the elements of the set to bound the first derivative for all these points. If the radius is chosen small enough, then the upper bound of the first derivative extends to all elements in our set.

Since the probability of failure is governed by $\E[\|AX\|^2]$, we need to choose a set of matrices for which this quantity is large, if we wish to prove high probability bounds. 
Thus, we define the following set of matrices. 
\[
\mathcal{R}_s := \{A: \E[\|AX\|^2] \geq s\}
\]
\begin{lemma}\label{eq:first_der_uniform}
    Suppose $X^{(1)}, \ldots, X^{(l)} \in \{-1,1\}^n$ are independent samples from an Ising model with interaction matrix $J^*$ satisfying Assumption~\ref{ass:bounded_width}. Let $A \in \R^{n \times n}$ be a symmetric matrix with $\|A\|_\infty \leq M$. 
Consider the net $\mathcal{U}_{s} := \mathcal{N}(\mathcal{R}_s, \|\cdot\|_2, \theta)$.  Then, 
    \[
\lp|\frac{\partial \phi(J^*)}{\partial A}\rp| \leq \frac{8CM}{\eta} \cdot t\cdot l \cdot\E[\|AX\|_2^2] +2\cdot l \cdot n \cdot \theta \enspace, \forall A \in \mathcal{R}_s
\]
with probability at least 
\[
1 - \frac{8}{3} (\log n ) \cdot |\mathcal{U}_s|\cdot\exp\lp(- c \min(t,t^2) \cdot l \cdot s  \rp) - 2n (\log n) \cdot  \exp\lp(-\frac{l\cdot u^2/2}{4n^2 + 2u\cdot n/3}\rp)
\]
\end{lemma}

\begin{proof}

By taking a union bound over the elements of $\mathcal{U}_s$, by definition of the set $\mathcal{R}_s$ we know that 
\[
\lp|\frac{\partial \phi_j(J^*)}{\partial A}\rp| \leq C \cdot t\cdot l \cdot\lp(\E[\|AX\|_2^2]\rp) \enspace, \forall A \in \mathcal{U}_s
\]
with probability at least
\[
    1 - |\mathcal{U}_s|\cdot\exp\lp(- c \min(t,t^2) \cdot l \cdot s  \rp)\enspace.
\]

Using the Lipschitzness of the first derivative, this implies that with the same probability, we have
\[
\lp|\frac{\partial \phi_j(J^*)}{\partial A}\rp| \leq C \cdot t\cdot l \cdot\E[\|AX\|_2^2] +2\cdot l \cdot n \cdot \theta \enspace, \forall A \in \mathcal{R}_s
\]

This bound holds conditional on $X^{(1)},\ldots,X^{(l)}$, assuming they have values that satisfy the event $E_{j, u}$.
But we have already bounded the probability that this event occurs in Lemma~\ref{lem:matrix_conc}. 
Thus, for the choice of constant $u$ that we have made, we have established that with probability at least 
\[
1 - |\mathcal{U}_s|\cdot\exp\lp(- c \min(t,t^2) \cdot l \cdot s  \rp) - 2n \exp\lp(-\frac{l\cdot u^2/2}{4n^2 + 2u\cdot n/3}\rp),
\]
it holds
\[
\lp|\frac{\partial \phi_j(J^*)}{\partial A}\rp| \leq C \cdot t\cdot l \cdot\E[\|AX\|_2^2] +2\cdot l \cdot n \cdot \theta \enspace, \forall A \in \mathcal{R}_s.
\]

We will see in the sequel what the optimal way is to adjust these parameters. 
Finally, by taking another union bound with respect to all different subsets $I_j$ and using \eqref{eq:splitting}, we have that

\[
\lp|\frac{\partial \phi(J^*)}{\partial A}\rp| \leq \frac{8CM}{\eta} \cdot t\cdot l \cdot\E[\|AX\|_2^2] +2\cdot l \cdot n \cdot \theta \enspace, \forall A \in \mathcal{R}_s,
\]
with probability at least 
\[
1 - (\log n ) \cdot |\mathcal{U}_s|\cdot\exp\lp(- c \min(t,t^2) \cdot l \cdot s  \rp) - 2n (\log n) \cdot  \exp\lp(-\frac{l\cdot u^2/2}{4n^2 + 2u\cdot n/3}\rp).
\]
\end{proof}

This is the uniform bound on the first derivative that we were aiming for. 
We will see how to pick the value of $s$ later.

We now focus on the second derivative. 
We start with a Lipschitzness property for the second moment in the second derivative. 

\begin{lemma}
    Let $A,B$ be two symmetric matrices with $\|A\|_\infty,\|B\|_\infty \leq M$. 
    Then
    \[
    \lp|\sum_{k=1}^l \lp(\|AX^{(k)}\|_2^2 - \|BX^{(k)}\|_2^2\rp)\rp| \leq  M \cdot l \cdot n \cdot \|A - B\|_2 
    \]
\end{lemma}
\begin{proof}
    We have that
    \begin{align*}
        \lp|\sum_{k=1}^l \lp(\|AX^{(k)}\|_2^2 - \|BX^{(k)}\|_2^2\rp)\rp| &= \lp|\sum_{k=1}^l \sum_{i=1}^n \lp((A_iX^{(k)})^2 - (B_iX^{(k)})^2\rp)\rp|\\
        &\leq \sum_{k=1}^l \sum_{i=1}^n |(A_i - B_i)X^{(k)}| \cdot |(A_i + B_i)X^{(k)}| \\
        &\leq \sum_{k=1}^l \sqrt{\sum_{i=1}^n |(A_i - B_i)X^{(k)}|^2} \cdot \sqrt{\sum_{i=1}^n |(A_i + B_i)X^{(k)}|^2}\\
        &\leq M \cdot l \cdot n \cdot \|A - B\|_2 
    \end{align*}    
\end{proof}
We are now ready to state the uniform guarantee for the second derivative. The proof will again be based on Lemma~\ref{lem:matrix_conc} to avoid unnecessary union bounds.

\begin{lemma}\label{lem:second_der_uniform}
    Suppose $X^{(1)}, \ldots, X^{(l)} \in \{-1,1\}^n$ are independent samples from an Ising model with interaction matrix $J^*$ satisfying Assumption~\ref{ass:bounded_width}. Let $A \in \R^{n \times n}$ be a symmetric matrix with $\|A\|_\infty \leq M$. Consider the net $\mathcal{U}_{s} := \mathcal{N}(\mathcal{R}_s, \|\cdot\|_2, \theta)$. Then, with probability at least
    \[
    1 - |\mathcal{U}_s|\exp\lp(- c \cdot l \cdot s\rp) - 2n \cdot  \exp\lp(-\frac{l\cdot u^2/2}{4n^2 + 2u\cdot n/3}\rp)
    \]
    we have
    \[
    \frac{\partial^2 \phi (J)}{\partial A^2}
    \geq C \cdot l \cdot \E[\|AX\|^2] - M \cdot l \cdot n \cdot \theta  \enspace, \forall A \in \mathcal{R}_s
    \]
\end{lemma}
\begin{proof}

We use the same argument as the first derivative. If we take $\eta=1/3$, take any $j$, the distribution of each $X_{I_j}|X_{-I_j}$ is Glauber MLSI($6$) and Po($6$). When we do the tensor product of all the samples, the same Glauber MLSI and Poincaré inequality holds. By Lemma \ref{lem:anticoncentration_one_sample}, we have that there exists absolute constants $c,C$ such that
    \[
    \frac{\partial^2 \phi (J)}{\partial A^2}
    \geq C \lp(l \|A_{I_j}\|_F^2 + \sum_{k=1}^l \|\E[AX^{(k)}|X^{(k)}_{-I_j}]\|^2\rp)
    \]
    with probability at least
    \[
    1 - \exp\lp(- c \cdot\lp(l \cdot\|A_{I_j}\|_F^2 + \sum_{k=1}^l \|\E[AX^{(k)}|X^{(k)}_{-I_j}]\|^2\rp)\rp)
    \]

The previous arguments have already established that 
\[
\|A_{I_j}\|_F^2 + \|\E[AX|X_{-I_j}]\|_2^2 = \Theta(\E[\|AX\|^2]).
\]

Using the concentration of Lemma~\ref{lem:matrix_conc} as before, we can establish that if event $E_{j,u}$ holds for a small enough constant $u$, then conditional on $X^{(k)}_{-I_j}$ for $k = 1,\ldots,l$ we have, there exists absolute constants $c,C$ such that 
\[
    \frac{\partial^2 \phi (J)}{\partial A^2}
    \geq C \cdot l \cdot \E[\|AX\|^2]
    \]
    with probability at least
    \[
    1 - \exp\lp(- c \cdot l \cdot \E[\|AX\|^2]\rp)
    \]
We can now implement the exact same union bound argument that we had for the first derivative. The only thing we need to check is the Lipschitzness of the second derivative, which will determine the size of our net.

Now, again by considering the net $\mathcal{U}_s := \mathcal{N}(\mathcal{R}_s, \|\cdot\|_2, \theta)$, taking a union bound over its elements and using the lipschitzness of the second derivative, in exactly the same fashion as with the first derivative, we get that conditioned on $E_{j,u}$
\begin{equation}\label{eq:lower_bound}
    \frac{\partial^2 \phi (J)}{\partial A^2}
    \geq C \cdot l \cdot \E[\|AX\|^2] - M \cdot l \cdot n \cdot \theta  \enspace, \forall A \in \mathcal{R}_s
\end{equation}
    with probability at least
    \[
    1 - |\mathcal{U}_s|\exp\lp(- c \cdot l \cdot s\rp)
    \]
    Removing the conditioning and using Lemma~\ref{lem:matrix_conc}, we have that \eqref{eq:lower_bound} holds with probability at least
     \[
    1 - |\mathcal{U}_s|\exp\lp(- c \cdot l \cdot s\rp) - 2n \cdot  \exp\lp(-\frac{l\cdot u^2/2}{4n^2 + 2u\cdot n/3}\rp)
    \]
\end{proof}
We are now ready to use the above lemmas to argue about the value of the pseudolikelihood for matrices that are ``far'' from $J^*$. 

\begin{theorem}\label{thm:mple_refined}
       Let $X_1,\ldots,X_l$ be independent samples drawn from an Ising model with interaction matrix $J^*$ satisfying Assumption~\ref{ass:bounded_width}. Let $\hat{J} \in \R^{n \times n}$ be the estimate of $J^*$ that is obtained by maximizing the pseudo-likelihood function \eqref{eq:pseudolikelihood_function}, i.e.
    \[
    \hat{J}^{(l)} := \argmax_{J^{(l)} \in \mathcal{R}^{(l)}} \phi(J;X)\enspace,
    \]
    where $\mathcal{R} = \{J\in \mathcal{S}_0^n: \|J\|_\infty \leq M$. 
    Then, for any $\epsilon \in (0,1)$ and $\delta > 0$, if 
    \[
    l = \tilde{\Omega}\lp(\frac{n^2 \lp( \log (1  / (\delta\varepsilon)) + \log n\rp)}{\epsilon}\rp)
    \]
    then with probability at least $1-\delta$ 
    \[
    \E_{J^*}[\|(\hat{J}-J^*)X\|_2^2] \leq \epsilon\enspace,
    \]
    where $\tilde{O}$ hides $\exp(M)$ factors. 
\end{theorem}
\begin{proof}
    For any matrix $J \in \mathcal{J}_M$, an application of Taylor's Theorem yields
    \[
    \phi(J) = \phi(J^*) + \frac{\partial \phi (J^*)}{\partial A}|_{A = J - J^*} + \frac{1}{2} \frac{\partial^2 \phi(J_\xi)}{\partial A^2}|_{A = J - J^*}
    \]
    where $J_\xi$ belongs in the line segment connecting $J,J^*$.
    Thus, using the preceding arguments, we know that with probability at least
    \[
    1 - (\log n ) \cdot |\mathcal{U}_s|\cdot\exp\lp(- c \min(t,t^2) \cdot l \cdot s  \rp) - 2n (\log n) \cdot  \exp\lp(-\frac{l\cdot u^2/2}{4n^2 + 2u\cdot n/3}\rp),
    \]
    we have
    \begin{equation}\label{eq:phiJ_prelim}
    \phi(J) \geq \phi(J^*) + C \cdot l \cdot \E[\|(J-J^*)X\|_2^2] - C' \cdot t \cdot l \cdot \E[\|(J-J^*)X\|_2^2] - C''\cdot l \cdot n \cdot \theta \enspace, \forall J: J-J^* \in \mathcal{R}_s.
    \end{equation}
    We now show how to choose the various parameters. First, we choose $t$ to be a small enough constant so that $C't < C/4$. We also choose $\theta = O(s/n)$ so that $C''\cdot l \cdot n \cdot \theta \leq C s / 4$. These choices mean that \eqref{eq:phiJ_prelim} can be written as
    \begin{equation}\label{eq:phiJ}
        \phi(J) \geq \phi(J^*) + \frac{C}{4} \cdot l \cdot s  \enspace, \forall J: J-J^* \in \mathcal{R}_s
    \end{equation}
    This means that all matrices $J$ such that $J - J^* \in \mathcal{U}_s$ have a higher negative pseudolikelihood value than $J^*$, which means that they will not be selected by the optimization procedure. Thus, this allows us to conclude that $\hat{J} - J^* \notin \mathcal{R}_s$, which implies that 
    $\E[\|(\hat{J}-J^*)X\|_2^2] \leq s$. 
    We now turn to analyze the probability that this event occurs. 
    Let us set $s = \epsilon$. 
    We would like to have
    \begin{equation}\label{eq:first_desirable}
    (\log n ) \cdot |\mathcal{U}_s|\cdot\exp\lp(- c \min(t,t^2) \cdot l \cdot \epsilon  \rp) \leq \frac{\delta}{2}
    \end{equation}
    for the specified error probability $\delta > 0$. As we argued in the previous section, using Corollary 4.1.15 from \cite{artstein2021asymptotic} gives
    \[
    |\mathcal{U}_s| \leq \lp(M + \frac{2}{\theta}\rp)^{n^2} = O\lp(\lp(M + \frac{2n}{\epsilon}\rp)^{n^2}\rp)
    \]
    Thus, if we choose 
   \begin{equation}\label{eq:first_cond}
       l \geq C \frac{n^2 \log (n/\varepsilon) + \log \log n + \log(1/\delta)}{\epsilon}
   \end{equation}
   for some constant $C>0$ that depends on $M$, 
   we can satisfy \eqref{eq:first_desirable}. 
   We would also like to have
   \begin{equation}\label{eq:second_desirable}
       2n (\log n) \cdot  \exp\lp(-\frac{l\cdot u^2/2}{4n^2 + 2u\cdot n/3}\rp) \leq \frac{\delta}{2}
   \end{equation}
    In the above, $u$ is a sufficiently small constant. To satisfy \eqref{eq:second_desirable}, it suffices to choose
    \begin{equation}\label{eq:second_cond}
        l \geq C \lp(n^2 \log n + n^2 \log (1/\delta) \rp)
    \end{equation}
    samples. The conditions \eqref{eq:first_cond} and \eqref{eq:second_cond} give us the final sample complexity. 

    \end{proof}
\subsection{Proof of Proof of Corollary~\ref{cor:n3_samples_informal}}\label{sec:proofn3}
Finally, we give a formal version of Corrolary~\ref{cor:n3_samples_informal} and then its proof.

\begin{corollary}\label{cor:n3_samples}
    Suppose we are in the setting of Theorem~\ref{thm:mple_refined}. Additionally, assume there exist constants $\gamma ,C > 0$, such that $\Pr_{J^*}$ satisfies $(\gamma/n,C)$-regularity regularity. Then, for any $\epsilon > 0$, if $l = \tilde{\Omega}(n^3/\epsilon^2)$, with probability $1-o(1)$ over the choice of samples $\dtv(\Pr_{\hat{J}}, \Pr_{J^*}) \leq \epsilon$.
\end{corollary}

\begin{proof}
    Using the same representation for KL between Ising models as Lemma 20 of \cite{dagan2021learning}, we get that for some matrix $J_\xi = \xi J^* + (1-\xi)\hat{J}$ 
    \[
    KL\lp(P_{\hat{J}}\middle\| P_{J^*}\rp) = \frac{1}{2}\Var_{X \sim J_\xi}\lp[X^\top \lp(J^*-\hat{J}\rp)X\rp]\enspace,
    \]
    where in the above, the notation $X$ is sampled from an Ising model with interaction matrix $J_\xi$. The Cauchy-Schwarz inequality now implies that
    \begin{align*}
    \Var_{X \sim J_\xi}\lp[X^\top \lp(J^*-\hat{J}\rp)X\rp] &\leq \E_{J_\xi}[(X^\top \lp(J^*-\hat{J}\rp)X)^2] \\
    &\leq n \cdot \E_{J_\xi}[\|(\hat{J}-J^*)X\|_2^2]\\
    &= n \cdot \frac{ \E_{J_\xi}[\|(J_\xi-J^*)X\|_2^2] }{(1 - \xi)^2}
    \end{align*}
    Without loss of generality, assume that $\epsilon^2 < \gamma$. 
    By Theorem~\ref{thm:mple_refined}, we know that if $l = \tilde{\Omega}(n^3/\epsilon^2)$, we have with probability $1-o(1)$
    \[
    \E_{J^*}[\|(\hat{J}-J^*)X\|_2^2] \leq \frac{\epsilon^2}{n}
    \]
    Then, by scaling, this implies
    \[
    \E_{J^*}[\|(J_\xi-J^*)X\|_2^2] \leq \frac{\epsilon^2}{n} \cdot (1-\xi)^2 \leq \frac{\epsilon^2}{n}
    \]
    Thus, since we know $J^*$ satisfies $(\gamma/\sqrt{n},C)$-regularity, we have that
    \[
    \E_{J_\xi}[\|(J_\xi-J^*)X\|_2^2] \leq C \frac{\epsilon^2 \cdot (1-\xi)^2}{n}
    \]
    which implies
    \[
    \dkl\lp(P_{\hat{J}}\middle\| P_{J^*}\rp) \leq C \cdot \epsilon^2
    \]
    Using Pinsker's inequality as in Lemma~\ref{lem:tv_to_frob_mlsi} concludes the proof.
\end{proof}

\section{Applications of Learning with MLSI}
In this Section, we present the proofs for the applications of Theorem~\ref{thm:frobenius_estimation_high_temp}. The step that remains is to bound the total variation distance between two Ising models by the Frobenius norm of the difference of their interaction matrices.
We give two such bounds. The first is cruder and makes no additional assumptions about the matrices.

\begin{restatable}{lemma}{tvtofrob}\label{lem:TV_from_frobenius}
    Suppose $\mu,\mu^*$ are the distributions corresponding to two Ising models with interaction matrices $J,J^* \in \mathcal{S}_0^n$ and zero external fields. Then, we have the following property:
    \[\|\mu-\mu^*\|_{TV}\le n\|J-J^*\|_F.\]
\end{restatable}

We give our version of the proof in Section \ref{sec:tvfrob}. Note that an alternative proof can be given using the technique in \cite{klivans2017learning} Lemma 7.3. 

The second Lemma additionally assumes that both matrices, as long as any matrix in the line that contains them, satisfy MLSI. In that case, we can obtain much more precise guarantees without losing polynomial factors.

\begin{lemma}\label{lem:tv_to_frob_mlsi}
    Suppose $J_1,J_2 \in \mathcal{S}_0^n$ are such that for every $t \in [0,1]$, $\Pr_{J_t}$ satisfies MLSI($\rho$), where $J_t = tJ_1 + (1-t)J_2$. Then,
    \[
    \dtv(P_{J_1},P_{J_2}) \leq \rho \cdot \|J_1 - J_2\|_F
    \]
\end{lemma}
\begin{proof}
     First, using Pinsker's inequality, we get
    \[
    TV(P_{J_1}, P_{J_2}) \leq \sqrt{\frac{KL\lp(P_{J_1}\middle\| P_{J_2}\rp)}{2}}\enspace.
    \]
    Thus, it suffices to bound $KL\lp(P_{J_1}\middle\| P_{J_2}\rp)$. For this, we rely on a standard calculation for exponential families that connects the KL divergence to the variance of the sufficient statistic. In particular, following the derivation in Lemma 20 of \cite{dagan2021learning}, we get that for some matrix $J_\xi$ contained in the line segment connecting $J_1$ and $J_2$
    \[
    KL\lp(P_{J_1}\middle\| P_{J_2}\rp) = \frac{1}{2}\Var_{X \sim J_\xi}\lp[X^\top \lp(J_1 - J_2\rp)X\rp]\enspace,
    \]
    where in the above, the notation $X$ is sampled from an Ising model with interaction matrix $J_\xi$. Since $J_\xi$ satisfy MLSI($\rho$), it also satisfies
    Po$(2\rho)$), which yields
    \begin{align*}
    \Var_{X \sim S}\lp[X^\top \lp(J_1 - J_2\rp)X\rp] &\leq 2\rho \E_{X \sim S} \lp[\|(J_1 - J_2)X\|_2^2\rp]\\
    &= 2\rho  \sum_{i=1}^n \Var_{X \sim S}\lp[\lp(J_1 - J_2\rp)_i^\top X\rp]\\
    &\leq 4\rho^2\sum_{i=1}^n \lp\|\lp(J_1 - J_2\rp)_i\rp\|_2^2\\
    &= 4\rho^2 \lp\|J_1 - J_2\rp\|_F^2
    \end{align*}
    This concludes the proof.
   
\end{proof}

We are now ready to present the proof of the applications.

\subsection{Application: SK/diluted SK model}

First, using Lemma~\ref{lem:TV_from_frobenius} and Theorem~\ref{thm:frobenius_estimation_high_temp}, we can derive the following corollary:

\begin{corollary}\label{cor_informal:n4}
    Let $X_1,\ldots,X_l$ be independent samples drawn from an Ising model with interaction matrix $J^*$ satisfying Assumption~\ref{ass:spectral_bounded}. Let $\hat{J} \in \R^{n \times n}$ be the estimate of $J^*$ that is obtained by maximizing the pseudo-likelihood function \eqref{eq:pseudolikelihood_function} for a set of matrices $\mathcal{R} \subseteq \mathcal{S}_0^n$, which is convex and admits efficient projections in Frobenius norm. 
    Then, for any $\varepsilon \in (0,1)$, if 
    $l = \tilde{O}\lp(\frac{n^4 \log (1  / (\delta\varepsilon))}{\varepsilon^2}\rp)$
    then with probability at least $1-\delta$, $TV(P_{\hat{J}}, P_{J^*}) \leq \varepsilon$, 
    where $\tilde O$ hides sub-polynomial factors of $n$, and other terms involving $\lambda, \rho$ and $h_{\max}$. Moreover, we can implement MPLE in polynomial time.
\end{corollary}
\begin{proof}
    The proof follows immediately by setting $\epsilon = \epsilon'/n$ in the guarantees of Theorem~\ref{thm:frobenius_estimation_high_temp}. The optimal value of the pseudolikelihood can be found in polynomial time using projected gradient descent.
\end{proof}

Corollary~\ref{cor_informal:n4} says that if we can find a set $\mathcal{R}$ that is convex and admits efficient projections, we can solve the pseudolikelihood estimation problem in polynomial time. We instantiate this fact in the cases of SK/diluted SK model now.

\begin{corollary}\label{cor_informal:sk}
    Suppose we are given $l$ independent samples $X^{(1)},\ldots,X^{(l)} \sim \Pr_{J^*}$, where $J^*$ is sampled according to the SK-model with $\beta < C$, where $C \approx 0.295$. That is, every $J_{ij}=J_{ji}$ is chosen i.i.d. from $\mathcal N(0,\beta^2/n)$. Then, there is a polynomial time algorithm (MPLE) that produces an estimate $\hat{J} \in \mathcal{S}_0^n$ , such that with probability $1-o(1)$ over the choice of samples and the choice of matrix $J^*$ we have $\dtv(\Pr_{\hat{J}},\Pr_{J^*}) \leq \epsilon$, as long as $l = \tilde{\Omega}(n^4/\epsilon^2)$.
\end{corollary}

\begin{proof}
    By \cite{anari2024trickle}, when $\beta<C$, with $1-o(1)$ probability over the choice of random matrix $J^*$ it satisfies ATE, which by \cite{chen2021optimal} means it also satisfies MLSI. Therefore, we can run MPLE on the set of matrices $\mathcal{R} = \{J \in \mathcal{S}_0^n: \|J\|_{op} \leq 4\}$. We know that 
    $J^* \in \mathcal{R}$ with probability $1-o(1)$. Also, it is clear we can efficiently project to $\mathcal{R}$ by eigenvalue clipping.
    Thus, applying Corollary~\ref{cor_informal:n4} finishes the proof.
\end{proof}

Similar guarantees can be obtained for the diluted SK model.

\begin{corollary}\label{cor_informal:diluted_sk}
    Suppose we are given $l$ independent samples $X^{(1)},\ldots,X^{(l)} \sim \Pr_{J^*}$, where $J^*$ is sampled according to the diluted SK-model with $\beta < C$, where $C \approx 0.295$. That is, consider a random $d$-regular graph $G$. If $i$ and $j$ are not connected, $J_{ij}=0$. If $i$ and $j$ are connected, then $J_{ij}=J_{ji}$ is chosen from $\{\frac{\beta}{\sqrt{d-1}},-\frac{\beta}{\sqrt{d-1}}\}$ with the same probability, independently for all such pairs. Then, there is a polynomial-time algorithm (MPLE) that produces an estimate $\hat{J} \in \mathcal{S}_0^n$ , such that with probability $1-o(1)$ over the choice of samples and the choice of matrix $J^*$ we have $\dtv(\Pr_{\hat{J}},\Pr_{J^*}) \leq \epsilon$, as long as $l = \tilde{\Omega}(n^4/\epsilon^2)$.
\end{corollary}

\begin{proof}
    The proof is similar to Corollary \ref{cor_informal:sk}. By \cite{anari2024trickle}, when $\beta<C$, with $1-o(1)$ probability it satisfies ATE, thus also MLSI with a constant that depends on the distance of $\beta$ from $C$. Thus, we can use Corollary~\ref{cor_informal:n4} to obtain the result.
\end{proof}

\subsection{Application: Spectrally Bounded Models}
For spectrally bounded models, the set $\mathcal{R}$ can be naturally restricted to only include matrices that satisfy MLSI. We thus get optimal guarantees for TV learning.

\begin{corollary}\label{cor_informal:Jop<1}
    Let $X_1,\ldots,X_l$ be independent samples drawn from an Ising model with interaction matrix $J^*\in \mathcal{S}_0^n$ that belongs in the set $\mathcal{R}= \{J \in \mathcal{S}_0^n: \lambda_{max}(J)- \lambda_{min}(J) \le 1 - \alpha\}$, for some $\alpha \in (0,1)$. 
    Let $\hat{J} \in \R^{n \times n}$ be the estimate of $J^*$ that is obtained by maximizing the pseudo-likelihood function \eqref{eq:pseudolikelihood_function} over $\mathcal{R}^{(l)}$, given $l$ independent samples from $J^*$. Then, for any $\varepsilon < 1$, with probability at least $1- \delta$
    \[
    TV(P_{\hat{J}}, P_{J^*}) \leq \varepsilon \enspace,
    \]
    whenever 
    \[
    l \geq \Tilde{O}\lp(\frac{n^2 + \log(1/\delta)}{\varepsilon^2}\rp)\enspace,
    \]
    where $\tilde O$ hides sub-polynomial factors of $n$ and constant that depends on $\alpha,\lambda$. Moreover, the algorithm can be implemented in polynomial time.
\end{corollary}

\begin{proof}
    First, let us argue that $\mathcal{R}$ is a convex set. We know that for two symmetric matrices, using the Rayleigh quotient, we can derive that for any $t \in (0,1)$
    \begin{gather*}
    \lambda_{\max}(tA_1+ (1-t)A_2)\le t \lambda_{\max}(A_1) + (1-t) \lambda_{\max}(A_2)\\
    \lambda_{\min}(tA_1+ (1-t)A_2)\ge t \lambda_{\min}(A_1) + (1-t) \lambda_{\min}(A_2)
    \end{gather*}
    so if $A_1,A_2\in \mathcal R$, $tA_1 + (1-t)A_2 \in \mathcal R$. 
    
    By Lemma~\ref{lem:mlsi}, every $J \in \mathcal{R}$ satisfies MLSI($1/\alpha$). Also, clearly $\|J\|_{op} < 1$ for every $J \in \mathcal{R}$, since the zero diagonal implies $J$ will have both positive and negative eigenvalues. Thus, combining Theorem~\ref{thm:frobenius_estimation_high_temp} and Lemma~\ref{lem:tv_to_frob_mlsi} gives the desired guarantee in TV.
    
    Now, let us argue about the computational efficiency of the method. Since the pseudolikelihood function is convex, 
    to find $\hat J$ we can use the projected gradient descent algorithm as in Theorem 3.2 of \cite{bubeck2015convex}. Thus, we only need to argue that at every step we can efficiently project on the set $\mathcal{R}$.
    
    We have shown that $\mathcal R$ is a convex set, and the distance from it, which is measured in Frobenius norm, is a convex function. Thus, one way of computing the projection would be using the result in \cite{lee2018efficient}, where they optimize the distance in polynomial time using only a membership oracle for $\mathcal{R}$, which is, of course, easy to implement in our case. Alternatively, one could write the projection as a semi-definite program (SDP) \cite{boyd2004convex} and then solve it efficiently using \cite{jiang2020faster}. For convenience, we spell out the details of the SDP approach below.

    The SDP in standard form minimizes $\langle C,X\rangle$ subject to $\langle A_i,X\rangle \le b_i$ and $X\succeq 0$. We can use a diagonal block matrix $X$ so that $X\succeq 0$ is equivalent to $X_1,X_2,\dots,X_k\succeq 0$, and furthermore, using the linear constraints $\langle A_i,X\rangle \le b_i$ (both $\langle A,X\rangle \le b$ and $\langle -A,X\rangle \le b$ to build an equation) we can encode linear relationships between matrices $X_1,X_2,\dots,X_k$. Therefore, we can actually write several, rather than one, positive semi-definite constraints of matrices, where each element has a linear relation to the other.

    Therefore, we can first encode two matrices $J$ (the matrix that needs to be projected) and $J'$ (the projected matrix), and three values $\lambda_1,\lambda_2,t$ in $X_1$'s off-diagonal and put free variables on the diagonal of $X_1$ so that $X_1\succeq 0$ if we can make the diagonal large enough. In the linear constraints, we encode the information that $J$ and  $J'$ are both zero-diagonal symmetric matrices. Then, we encode $\lambda_1 I-J'\succeq0$ and $J'-\lambda_2I\succeq0$ and $\lambda_1-\lambda_2\le 1-\alpha$ to make the constraint $\lambda_{\max}(J')-\lambda_{\min}(J')\le1-\alpha$. Finally, we flatten the elements in $J-J'$ into a column vector $v$, and we add the constraint $\begin{pmatrix} I&v\\v^{\top}&t\end{pmatrix}\succeq 0$ which is equivalent to $t\ge\|J-J'\|^2_F$. Finally, we optimize $t$ to make it as small as possible. We solve this SDP to find $J'$, which is the projection from $J$ to $\mathcal R$. This SDP has a polynomial size in $n$ and a polynomial number of constraints. By \cite{jiang2020faster} (the main result), we can get an efficient algorithm solving the SDP, thus we can project on $\mathcal R$ efficiently.
\end{proof}

\subsection{Application: Antiferromagnetic Expanders}

\begin{corollary}\label{cor_informal:negative_spike}
    Given $0<\alpha<1$ and $c>0$. Let $\mathbf 1$ be the all-one vector. Let $\mathcal A$ be the set of matrices such that for all $A\in \mathcal A$, $A\mathbf 1=0$ and $0\preceq A\preceq (1-\alpha)I$. Suppose we are given $l$ independent samples $X^{(1)},\ldots,X^{(l)} \sim \Pr_{J^*}$, where $J^*$ is from $$\mathcal R=\mathcal S_0^n\cap\{\exists r\in\mathbb R, 0\le t\le c, A\in\mathcal A,\text{ s.t. } J^*+rI=-\frac{t}{n}\mathbf 1\mathbf 1^{\top}+A\}.$$
    Then, there is a polynomial time algorithm (MPLE) that produces an estimate $\hat{J} \in \mathcal R$, such that with high probability over the choice of samples and the choice of matrix $J^*$ we have $\dtv(\Pr_{\hat{J}},\Pr_{J^*}) \leq \epsilon$, as long as $l = \tilde{\Omega}(n^2/\epsilon^2)$.
\end{corollary}

\begin{proof}
    First, we know that, from \cite{anari2024trickle}, any matrix in $\mathcal{R} $ satisfies ATE, and thus satisfies modified-LSI. Following the previous recipe, we just need to prove that $\mathcal R$ is convex. First, we know that $\mathcal A$ is convex. Since the average of two positive semi-definite (PSD) matrices is PSD, for any $A_1, A_2\in\mathcal A$, $(A_1+A_2)/2\succeq 0$, and $(1-\alpha)I-(A_1+A_2)/2\succeq 0$. Therefore, $(A_1+A_2)/2\in \mathcal A$. Therefore, for $J_1, J_2 \in \mathcal R$, there exists real numbers $r_1,r_2,t_1,t_2,A_1,A_2$, such that $J_i+r_iI=\frac{t_i}{n}\mathbf 1\mathbf 1^{\top}+A_i$ for $i=1,2$. Therefore, we have $(J_1+J_2)/2+\frac{r_1+r_2}{2}I=\frac{(t_1+t_2)/2}{n}\mathbf 1\mathbf 1^{\top}+(A_1+A_2)/2$, and thus $(J_1+J_2)/2\in\mathcal R$, and thus $\mathcal R$ is convex. 

    Again, we have shown that $\mathcal R$ is a convex set, and we are optimizing a convex function. And the constraints of being a matrix in $\mathcal R$ are spectral or linear, as in the proof of Corollary~\ref{cor_informal:Jop<1}. Thus,
    we can formulate the problem as an SDP again and solve it efficiently \cite{jiang2020faster}.
    As in the proof of Corollary~\ref{cor_informal:Jop<1}, we encode matrices $J$ (matrix need to be projected) and $J'$ (the projected matrix) and scalars $\lambda_1,\lambda_2,t/n,v$. We also encode $M=J'-t/n\cdot \mathbf 1\mathbf 1^{\top}$ and $0\le t\le c$ in the linear constraint. We make $M=M^{\top}$, $M$ has equal-diagonal, and $M\mathbf 1=0$ in the linear constraints as well. In addition, we have $\lambda_1I-M\succeq 0,M-\lambda_2I\succeq 0$ as a positive semi-definite constraint and $\lambda_1-\lambda_2\le 1-\alpha$ as a linear constraint. Finally, we repeat the steps in Corollary \ref{cor_informal:Jop<1} for the final step of the Frobenius norm. As we can write the SDP like that, we can get an efficient projection algorithm.
\end{proof}

\begin{remark}\label{rmk:counter_for_bdd_w}
    We should mention that in this antiferromagnetic expanders model, despite the constraints, it could still be that $\|J^*\|_{\infty}$ is unbounded. Indeed, consider a Paley graph $G$ of a $4k+1$ type prime $p$ (that is, $i$ and $j$ are connected if and only if $i-j$ is a quadratic residue), and consider $J-I-2G$, where $J$ is the all-one matrix. We know that $G$ has all the eigenvalues $0$ once with the all-one vector, and $\pm \sqrt{p}$ appears $(p-1)/2$ times. Then, we know that if we consider $J^*=\frac{1-\alpha}{2\sqrt{p}}(J-I-2G)-t\cdot (J-I)$ for $0<t<c-1$, it will be in $\mathcal R$.
\end{remark}

\section{Bounding the TV distance by the Frobenius norm (Proof of Lemma~\ref{lem:TV_from_frobenius})}\label{sec:tvfrob}

Suppose we have two Ising models with interaction matrices $J,J^* \in \mathcal{S}_0^n$ respectively and zero external fields. The following lemma bounds their TV distance in terms of $\|J - J^*\|_F$. 

\tvtofrob*

\begin{proof}
We consider the TV distance of the $X_i$ conditioning on $\mu,\mu^*$.  By Pinsker's inequality, we have

\[2\|\mu-\mu^*\|_{TV}^2\le KL(\mu\|\mu^*)\]

Consider the symmetric KL distance, we have

\[4\|\mu-\mu^*\|_{TV}^2\le KL(\mu\|\mu^*)+KL(\mu^*\|\mu)=\E_{X\sim\mu}(X^{\top}(J-J^*)X)+\E_{X\sim\mu^*}(X^{\top}(J^*-J)X)\]

We expand one of them, and the second one is analogous. 

We can write
\[\E_{X\sim\mu}(X^{\top}(J-J^*)X)=2\sum_{i<j}(J_{ij}-J^*_{ij})\E_{X\sim\mu}(X_iX_j)=2\sum_{i<j}(J_{ij}-J^*_{ij})(2\Pr_{X\sim\mu}(X_iX_j=1)-1).\]

Therefore, after we group them, we have
\[
\E_{X\sim\mu}(X^{\top}(J-J^*)X)+\E_{X\sim\mu^*}(X^{\top}(J^*-J)X)=4\sum_{i<j}(J_{ij}-J^*_{ij})(\Pr_{X\sim\mu}(X_iX_j=1)-\Pr_{X\sim\mu^*}(X_iX_j=1))
\]

The probability of $X_iX_j=1$ can be upper-bounded by coupling. First, we consider the distribution of $X$ and $X'$ according to $\mu$ and $\mu^*$, respectively. We couple $X$ and $X'$ to achieve the probability $X_{-i-j}\ne X'_{-i-j}$ as small as possible. We call the event $X_{-i-j}\ne X_{-i-j}'$ to be $E_1$, and $X_{-i-j}=X_{-i-j}'$ to be $E_2$. Therefore, we can split the difference in probability to be
\begin{align*}
&\E_{X\sim\mu}(X^{\top}(J-J^*)X)+\E_{X\sim\mu^*}(X^{\top}(J^*-J)X)\\
=&4\sum_{i<j}(J_{ij}-J^*_{ij})(\Pr_{X\sim\mu}(X_iX_j=1)-\Pr_{X\sim\mu^*}(X_iX_j=1))\\
\le&4\sum_{i<j}|J_{ij}-J^*_{ij}|\cdot|\Pr_{X\sim\mu}(X_iX_j=1)-\Pr_{X\sim\mu^*}(X_iX_j=1)|\\
\le&4\sum_{i<j}|J_{ij}-J^*_{ij}|\cdot\left(\Pr(E_1)\cdot \left|\Pr_{X\sim\mu}(X_iX_j=1|E_1)-\Pr_{X\sim\mu^*}(X_iX_j=1|E_1)\right|\right.\\
&~~~~~~~~~~~~~~~~~~~~\left.+\Pr(E_2)\cdot \left|\Pr_{X\sim\mu}(X_iX_j=1|E_2)-\Pr_{X\sim\mu^*}(X_iX_j=1|E_2)\right|\right)
\end{align*}

For $E_1$ case, We have $\Pr(E_1)\le\|\mu-\mu^*\|_{TV}$ along with the naive bound 
$$|\Pr(X_iX_j=1|E_1)-\Pr(X_iX_j=1|E_1)|\le 1.$$

For the $E_2$ case, we bound $\Pr(E_2)\le 1$, and we upper bound the probability difference by the largest possible difference of the conditional probability. That is:
\begin{align*}
    \left|\Pr_{X\sim\mu}(X_iX_j=1|E_2)-\Pr_{X\sim\mu^*}(X_iX_j=1|E_2)\right|\le \max_{X_{-i-j}}\left|\Pr_{X\sim\mu}(X_iX_j=1|X_{-i-j})-\Pr_{X\sim\mu^*}(X_iX_j=1|X_{-i-j})\right|.
\end{align*}

To sum up, we have the following inequality:
\begin{align*}
&\E_{X\sim\mu}(X^{\top}(J-J^*)X)+\E_{X\sim\mu^*}(X^{\top}(J^*-J)X)\\
\le&4\sum_{i<j}|J_{ij}-J^*_{ij}|\cdot\left(\Pr(E_1)\cdot \left|\Pr_{X\sim\mu}(X_iX_j=1|E_1)-\Pr_{X\sim\mu^*}(X_iX_j=1|E_1)\right|\right.\\
&~~~~~~~~~~~~~~~~~~~~\left.+\Pr(E_2)\cdot \left|\Pr_{X\sim\mu}(X_iX_j=1|E_2)-\Pr_{X\sim\mu^*}(X_iX_j=1|E_2)\right|\right)\\
\le&4\sum_{i<j}|J_{ij}-J^*_{ij}|\cdot\left(\|\mu-\mu^*\|_{TV}+ \max_{X_{-i-j}}\left|\Pr_{X\sim\mu}(X_iX_j=1|X_{-i-j})-\Pr_{X\sim\mu^*}(X_iX_j=1|X_{-i-j})\right|\right)
\end{align*}

For Ising model $\mu$ with interaction matrix $J$, denote $h_i=\sum_{k\ne i,j}J_{ik}$, and $h_j=\sum_{k\ne i,j}J_{jk}$ we can calculate that
\begin{align*}
    &\Pr(X_iX_j=1)\\
    =&\frac{\exp(J_{ij}+h_i+h_j)+\exp(J_{ij}-h_i-h_j)}{\exp(J_{ij}+h_i+h_j)+\exp(J_{ij}-h_i-h_j)+\exp(-J_{ij}+h_i-h_j)+\exp(-J_{ij}-h_i+h_j)}.
\end{align*}

Denote $v=h_i+h_j$ and $w=h_i-h_j$. So, we have the following
\[
\Pr(X_iX_j=1)
    =\frac{\exp(J_{ij})\cosh(v)}{\exp(J_{ij})\cosh(v)+\exp(-J_{ij})\cosh(w)}=:F(J_{ij},u,w).
\]

Denote $v^*=\sum_{k\ne i,j}J_{ik}^*+\sum_{k\ne i,j}J_{jk}^*$ and $w^*=\sum_{k\ne i,j}J_{ik}^*-\sum_{k\ne i,j}J_{jk}^*$. Our target is to upper bound $|F(J_{ij},u,w)-F(J_{ij}^*,u^*,w^*)|$. We take the derivative to show the Lipschitzness of $F$ to achieve this. By taking the derivative directly, we have

\[\frac{\partial F}{\partial J_{ij}}=\frac{2\cosh(v)\cosh(w)}{(\exp(J_{ij})\cosh(v)+\exp(-J_{ij})\cosh(w))^2}.\]
\[\frac{\partial F}{\partial v}=\frac{\sinh(v)\cosh(w)}{(\exp(J_{ij})\cosh(v)+\exp(-J_{ij})\cosh(w))^2}.\]
\[\frac{\partial F}{\partial v}=-\frac{\sinh(w)\cosh(v)}{(\exp(J_{ij})\cosh(v)+\exp(-J_{ij})\cosh(w))^2}.\]

By AM-GM ineqiality, we have $|\frac{\partial F}{\partial J_{ij}}|\le \frac{1}{2}$. By furthermore $|\sinh(x)|\le |\cosh(x)|$, we have both $|\frac{\partial F}{\partial v}|$ and $|\frac{\partial F}{\partial w}|\le\frac{1}{4}$. Therefore, we have the bound:
\[
|F(J_{ij},u,w)-F(J_{ij}^*,u^*,w^*)|\le \frac{1}{2}|J_{ij}-J^*_{ij}|+\frac{1}{4}|v-v^*|+\frac{1}{4}|w-w^*|
\]

Plugging back the expression of $v,w,v^*,w^*$, we use the triangle inequalty, we have
\begin{align*}
    |v-v^*|&=\left|\sum_{k\ne i,j}J_{ik}+\sum_{k\ne i,j}J_{jk}-\sum_{k\ne i,j}J_{ik}^*-\sum_{k\ne i,j}J_{jk}^*\right|\\
    &=\left|\sum_{k\ne i,j}J_{ik}-J_{ik}+J_{jk}-J_{jk}^*\right|\le \sum_{k\ne i,j}|J_{ik}-J_{ik}|+|J_{jk}-J_{jk}^*|.
\end{align*}

And, we can get the same for $w-w^*$. Therefore, we can deduce that
\[
|F(J_{ij},u,w)-F(J_{ij}^*,u^*,w^*)|\le \frac{1}{2}|J_{ij}-J^*_{ij}|+\frac{1}{2}\sum_{k\ne i,j}|J_{ik}-J_{ik}^*|+|J_{ik}-J_{jk}^*|.
\]

Therefore, we can have the following inequality:
\begin{align*}
&\E_{X\sim\mu}(X^{\top}(J-J^*)X)+\E_{X\sim\mu^*}(X^{\top}(J^*-J)X)\\
\le&4\sum_{i<j}|J_{ij}-J^*_{ij}|\cdot\left(\|\mu-\mu^*\|_{TV}+ \max_{X_{-i-j}}\left|\Pr_{X\sim\mu}(X_iX_j=1|X_{-i-j})-\Pr_{X\sim\mu^*}(X_iX_j=1|X_{-i-j})\right|\right)\\
\le&4\sum_{i<j}|J_{ij}-J^*_{ij}|\cdot\left(\|\mu-\mu^*\|_{TV}+ \frac{1}{2}|J_{ij}-J^*_{ij}|+\frac{1}{2}\sum_{k\ne i,j}|J_{ik}-J_{ik}^*|+|J_{ik}-J_{jk}^*|\right)
\end{align*}

By the Cauchy-Schwarz inequality, we have
\[\sum_{i<j}|J_{ij}-J^*_{ij}|=\frac{1}{2}\sum_{i\ne j}|J_{ij}-J^*_{ij}|\le \frac{n}{2}\|J-J^*\|_F.\]

Also, the rest of that is
\begin{align*}
  &\sum_{i<j}|J_{ij}-J^*_{ij}|\cdot\left(\frac{1}{2}|J_{ij}-J^*_{ij}|+\frac{1}{2}\sum_{k\ne i,j}|J_{ik}-J_{ik}^*|+|J_{ik}-J_{jk}^*|\right)\\
=&\sum_{i=1}^n\sum_{j \ne i}|J_{ij}-J^*_{ij}| \cdot\left(\frac{1}{4}|J_{ij}-J^*_{ij}|+\frac{1}{2}\sum_{k\ne i,j}|J_{ik}-J_{ik}^*|\right)\\
<&\sum_{i=1}^n\sum_{j \ne i}|J_{ij}-J^*_{ij}| \cdot\left(\frac{1}{2}|J_{ij}-J^*_{ij}|+\frac{1}{2}\sum_{k\ne i,j}|J_{ik}-J_{ik}^*|\right)=\sum_{i=1}^n\frac{1}{2}\left(\sum_{j \ne i}|J_{ij}-J^*_{ij}|\right)^2.
\end{align*}

By the Cauchy-Schwarzartz inequality, we know that
\[\left(\sum_{j \ne i}|J_{ij}-J^*_{ij}|\right)^2\le n\sum_{j \ne i}|J_{ij}-J^*_{ij}|^2\]

Thus, we can have
\begin{align*}
  &\sum_{i<j}|J_{ij}-J^*_{ij}|\cdot\left(\frac{1}{2}|J_{ij}-J^*_{ij}|+\frac{1}{2}\sum_{k\ne i,j}|J_{ik}-J_{ik}^*|+|J_{ik}-J_{jk}^*|\right)\\
<&\sum_{i=1}^n\frac{1}{2}\left(\sum_{j \ne i}|J_{ij}-J^*_{ij}|\right)^2\le \sum_{i=1}^n\frac{n}{2}\sum_{j \ne i}|J_{ij}-J^*_{ij}|^2=\frac{n}{2}\|J-J^*\|_F^2.
\end{align*}
So, in summary, we have
\begin{align*}
&\E_{X\sim\mu}(X^{\top}(J-J^*)X)+\E_{X\sim\mu^*}(X^{\top}(J^*-J)X)\\
\le&4\sum_{i<j}|J_{ij}-J^*_{ij}|\cdot\left(\|\mu-\mu^*\|_{TV}+ \frac{1}{2}|J_{ij}-J^*_{ij}|+\frac{1}{2}\sum_{k\ne i,j}|J_{ik}-J_{ik}^*|+|J_{jk}-J_{jk}^*|\right)\\
\le &2n \|\mu-\mu^*\|_{TV}\cdot \|J-J^*\|_F+2n\|J-J^*\|_F^2
\end{align*}

Because in the beginning we have established 
\[4\|\mu-\mu^*\|_{TV}^2=\E_{X\sim\mu}(X^{\top}(J-J^*)X)+\E_{X\sim\mu^*}(X^{\top}(J^*-J)X),\]

We have the quadratic formula
\[2\|\mu-\mu^*\|_{TV}^2\le n \|\mu-\mu^*\|_{TV}\cdot \|J-J^*\|_F+n\|J-J^*\|_F^2,\]

which yields that $\|\mu-\mu^*\|_{TV}\le n\|J-J^*\|_F$.
\end{proof}

\section{Discussion of guarantee from Theorem~\ref{thm:mple_refined}}\label{sec: condition stricter}

Notice that for Ising model with bounded with, Lemma 6 in \cite{dagan2021learning} has established that (also in the proof of Theorem \ref{thm: regular Jinf}) there is a constant (depending only on $M$) that $\E_{J^*}[\|(J-J^*)X\|]\ge C_M\|J-J^*\|_F^2$ and thus we know that the condition $\E[\|(J-J^*)X\|^2]\le\varepsilon^2$ implies that $\|J-J^*\|_F^2\le\varepsilon^2/C_M$. Therefore, the Corollary \ref{cor:frobenius_informal} follows from Theorem \ref{thm:mple_improved_informal}.

Now, we use an example to show that in some cases, it is strictly stronger. Let $\beta_1,\beta_2>1$ be two numbers and $J^*=\frac{\beta_1}{n}\mathbf 1\mathbf 1^{\top}$, and $J=\frac{\beta_2}{n}\mathbf 1\mathbf 1^{\top}$. 

It is well known that (see e.g \cite{ellis2007entropy}) for Curie Weiss model with inverse temperature $\beta > 1$, we have with probability at least $1- \exp(-C\sqrt{n})$ that 
$$
\frac{1}{n}\sum_{i=1}^n X_i=\pm (2s(\beta)-1) + O(n^{-1/4})\enspace,
$$ 
where $s(\beta)$ is the (larger) maximum in the optimization problem \[s(\beta)=\argmax_s -s\log(s)-(1-s)\log(1-s)+2\beta s(1-s).\]

Now, the condition given in Theorem \ref{thm:mple_improved_informal} for matrices $J,J^*$ says that
$
\E_{J^*}[\|(J-J^*)X\|^2]\le\varepsilon^2 \enspace.
$
This implies that $\frac{(\beta_1-\beta_2)^2}{n}\E_{J^*}[(\sum_{i=1}^n X_i)^2]\le\varepsilon^2$, or, $(\beta_1-\beta_2)\lesssim \frac{\varepsilon}{\sqrt{n}\cdot (2s(\beta_1)-1)}$. Therefore, this implies $\|J-J^*\|_{F}\lesssim\frac{\varepsilon}{\sqrt{n}}$, stricter than $\|J-J^*\|_{F}\lesssim{\varepsilon}$.

\section{Verifying the regularity condition}\label{sec: regularity verification}

In this section, we will demonstrate how the regularity condition of Theorem~\ref{thm:mple_refined} can be verified whenever our model satisfies two natural high-temperature conditions. 
These include the following two cases: (1) Dobrushin's condition holds, (2) the model is spectrally bounded. We also show that it can be true even in low temperature settings, with the notable example being the Curie-Weiss model.

We start with the first claim.

\begin{theorem}\label{thm: regular Jinf}
    Suppose $J^* \in \mathcal{S}_0^n$ satisfies $\|J^*\|_\infty < 1 - \delta$, where $\delta>0$. Then, there is some $\epsilon_0 >0$ that depends on $\delta$, such that the Ising model with interaction matrix $J^*$ satisfies $(\epsilon_0^2,C)$-regularity, where $C$ depends on $\delta$ and $\epsilon$.  
\end{theorem}
    
\begin{proof}
    We compare $\E[\|(J-J^*)X\|^2]$ with $\|(J-J^*)X\|_F^2$. Specifically, because $\|J^*\|_{\infty}<1-\delta$, we can show that, by Lemma 6 in \cite{dagan2021learning}, there exists a universal constant $C_1$ such that for all $a\in \R^n$, $\E_{J^*}[\|a^\top X\|]=\Var[a^\top X]\ge C_1\|a\|^2$. Therefore, we know that $\E[\|(J-J^*)X\|^2]\le \varepsilon^2_0$ implies that 
    \[\|J-J^*\|^2_F\le C_1^{-1} \E_{J^*}[\|(J-J^*)X\|^2] \le C_1^{-1}\epsilon_0^2\enspace,\]
    which then implies that $\|J-J^*\|_\infty\le \sqrt{C_1^{-1}}{\varepsilon_0}$. So, we know that $\|J\|_{\infty}\le 1-\delta+\sqrt{C_1^{-1}}\varepsilon_0$. Thus, if we take $\varepsilon_0 \le \sqrt{C_1}\delta/2$, we know that $\|J\|_{\infty}\le 1-\delta/2$.

    Then, by Theorem 3.7 in \cite{adamczak2019note}, we know that for the Ising model with interaction matrix $J$, there exists a constant $C_2$ such that the Poincaré inequality holds with coefficient $\frac{C_2}{\delta}$. Therefore, we know that 
    \[\E_J[\|(J-J^*)X\|^2]=\sum_{i=1}^n\Var[(J-J^*)_i^{\top}X]\le\frac{C_2}{\delta}\sum_{i=1}^n\|(J-J^*)\|_2^2=\frac{C_2}{\delta}\|J-J^*\|_F^2\]\
    We conclude that 
    \[\E_J[\|(J-J^*)X\|^2]\le\frac{C_2}{\delta}\|J-J^*\|_F^2\le\frac{C_2}{\delta C_1}\cdot \E_{J^*}[\|(J-J^*)X\|^2].\]
\end{proof}

\begin{theorem}
    Let $J^* \in \mathcal{S}_0^n$ satisfy $\lambda_{\max}(J^*) - \lambda_{\min}(J^*) < 1-\delta$. Then, there exists $\epsilon_0 > 0$ that depends on $\delta$, such that the Ising model with interaction matrix $J^*$ satisfies $(\epsilon_0^2, C)$-regularity, where $C>0$ is a constant that depends on $\delta$. 
\end{theorem}
    
\begin{proof}
    Like the previous lemma, we compare $\E[\|(J-J^*)X\|^2]$ with $\|(J-J^*)X\|_F^2$. Specifically, because $\lambda_{\max}(J^*) - \lambda_{\min}(J^*)<1-\delta$, we can show that, by Lemma \ref{lem: lower bound Jop<1}, we have
    \[\|J-J^*\|^2_F\le 2\cosh\lp(4\sqrt{\frac{1}{1-\|J^*\|_{op}}}\rp)^2{\varepsilon_0^2}\le 2\cosh\lp(4\sqrt{\frac{1}{\delta}}\rp)^2{\varepsilon_0^2},\]
    which then implies that $\|J-J^*\|_{F}\le \sqrt{2}\cosh(4\sqrt{1/\delta}){\varepsilon_0}$. So, we know that 
    $$
    \lambda_{\max}(J^*) - \lambda_{\min}(J^*) \le 1-\delta+2\sqrt{2}\cosh(4\sqrt{1/\delta}){\varepsilon_0} \enspace.
    $$
    Then, if we take $\varepsilon_0\le \delta/(4\sqrt{2}\cosh(4\sqrt{1/\delta}))$, we know that $\lambda_{\max}(J^*) - \lambda_{\min}(J^*)\le 1-\delta/2$.

    We can then apply Lemma \ref{lem: poincare}, which says that there exists an absolute constant $C_2$, such that
    the Ising model with interaction matrix $J$ satisfies the Poincare inequality with coefficient $\frac{C_2}{\delta}$. Therefore, we know that 
    \[\E_J[\|(J-J^*)X\|^2]\le \frac{C_2}{\delta}\|J-J^*\|_F^2\]\
    We conclude that 
    \[\E_J[\|(J-J^*)X\|^2]\le\frac{C_2}{\delta}\|J-J^*\|_F^2\le \frac{2C_2}{\delta } \cosh\lp(4\sqrt{\frac{1}{1-\|J^*\|_{op}}}\rp)^2 \E_{J^*}[\|(J-J^*)X\|^2]\enspace.\]

    This finishes the proof.
\end{proof}

We conclude the section by showing that the Curie-Weiss model satisfies the regularity condition even at low temperatures.

\begin{theorem}
    Let $J^*$ be the matrix for the Curie-Weiss model. That is, $J^*=\frac{\beta}{n}(J-I)$ where $J$ is the all-one matrix and $\beta>1$. Then, we have $J^*$ satisfies $(c_1/n,c_2)$ regularity, where $c_1,c_2$ are constants that depends on $\beta$. 
\end{theorem}
\begin{proof}
    We know that the distribution is symmetric for all $X_i$. Therefore, the covariance matrix of $X_1,\dots, X_n$, denoted by $\Sigma$, is a matrix of the form $aJ+(1-a)I$ for some $0<a<1$. By \cite{deb2023fluctuations}, we can calculate that $a$ converges to $(2s-1)^2$ where $s>1/2$ is the solution of the equation $2(2s-1)=\log(\frac{s}{1-s})$. Therefore, for any $A$, we can calculate that
    $$
    \E_{J^*}[\|AX\|^2]=\mathrm{Tr}(A\E[XX^{\top}]A)=\mathrm{Tr}(A(aI+(1-a)J)A)=(1-a)\|A\|_F^2+a\|A\mathbf{1}\|^2
    $$
    
    Now, suppose $\E_{J^*}[\|AX\|^2]=\frac{\varepsilon}{n}$, where $\varepsilon < c_1$. We will prove that
    \[
    \E_{J}[\|AX\|^2] \leq c_2 \cdot \E_{J^*}[\|AX\|^2] \enspace,
    \]
    for some constant $c_2 > 0$.
    
    First of all, by assumption we have $a\cdot(\|A\mathbf 1\|^2)+(1-a)\|A\|_F^2=\frac{\varepsilon}{n}$. Therefore, we have $\|A\|_{op}\le\|A\|_{F}\le\sqrt{\frac{\varepsilon}{(1-a)n}}$, and $\|A\mathbf 1\|\le\sqrt\frac{\varepsilon}{an}$.

    Also, we can calculate the following:
    $$
    \E_{J^*}[X^{\top}AX]=\mathrm{Tr}(A\E[XX^{\top}])=\mathrm{Tr}(A(aI+(1-a)J))=(1-a)\mathbf 1^{\top}A\mathbf 1 \enspace,
    $$
    since $A$ has zero diagonal.

    Thus, we can derive that 
    $$|\mathbf 1^{\top}A\mathbf 1|\le\sqrt{n}\cdot \|A\mathbf 1\|\le\sqrt{n}\cdot\sqrt{\frac{\varepsilon}{an}}=\sqrt{\varepsilon/a}.$$
    The quantity $X^\top A X$ is important, since it represents the difference in Hamiltonians of $J$ and $J^*$. 
    
    First, we show an inequality: for some universal constant $c$ (without loss of generality $c\le 1$)
    $$\Pr_{J^*}\lp[\left|X^{\top}AX-|\mathbf1^{\top}A\mathbf1|\right|\ge t\rp]\le2\exp\lp(-c\min\lp(\frac{t^2}{\|A\|_F^2},\frac{t}{\|A\|_{op}}\rp)\rp)$$
    This means that the difference in Hamiltonians is upper-bounded with high probability. 
    
    To do that, we can use \cite{deb2023fluctuations} and the Hanson-Wright inequality. By Proposition 4.2 in \cite{deb2023fluctuations}, the Curie Weiss model can be written as a mixture of i.i.d. distributions of $X_i$. Importantly, each of these mixtures has the property that all variables have the same mean. Therefore, by Hanson-Wright inequality for each product measure $\mu$, there exists a universal constant $c$ (for all $\mu$) such that
    $$
    \Pr_{\mu}[|X^{\top}AX-\E_\mu[X^{\top}A X]|\ge t]\le2\exp\lp(-c\min\lp(\frac{t^2}{\|A\|_F^2+[\E_{\mu}\|AX\|]^2},\frac{t}{\|A\|_{op}}\rp)\rp).
    $$

    Therefore, if we denote $b_\mu = 1 - \E_\mu[X_i^2]$ (remember this value is the same for all $i$ by \cite{deb2023fluctuations}), then 
    $$
    \E_{\mu}[X^{\top}AX]=\mathrm{Tr}(A\E_{\mu}[XX^{\top}])=\mathrm{Tr}(A\E_{\mu}(b_\mu I+(1-b_\mu)J))\enspace.
    $$
    Therefore, we have for all i.i.d. measures $\mu$ in the decomposition 
    $$\E_{\mu}[X^{\top}A X]=\mathrm{Tr}(A(b_\mu I+(1-b_\mu)J))=(1-b_\mu)\mathrm{Tr}(AJ)=(1-b_\mu)(\mathbf1^{\top}A\mathbf1)\le|\mathbf1^{\top}A\mathbf1|,$$
    Also, we can calculate that for any i.i.d. $\mu$, $\E_{\mu} [XX^{\top}]=b_\mu J+(1-b_\mu)I\preceq I+J$ and therefore, 
    $$\max_{\mu}\E_{\mu}[|\|AX\|^2]\le \|A\mathbf 1\|^2+\|A\|_{F}^2\le\frac{\varepsilon/n}{a(1-a)}\enspace,$$
    where the maximum is with respect to all product measures in the decomposition.
    Let $u=|\mathbf1^{\top} A\mathbf 1|$ and  $v=\frac{\varepsilon/n}{a(1-a)}$. Then, for all $\mu$, by Hanson-Wright inequality
    $$
    \Pr_{\mu}[X^{\top}AX-|\mathbf1^{\top}A\mathbf1|\ge t]\le\Pr_{\mu}[X^{\top}AX-\E_{\mu}[X^{\top}AX]\ge t]\le2\exp\lp(-c\min\lp(\frac{t^2}{\|A\|_F^2+v},\frac{t}{\|A\|_{op}}\rp)\rp)\enspace.
    $$
    
    Taking expectation with respect to the random choice of measure $\mu$ in the decomposition, we can derive the tail bound $$
    \Pr_{J^*}[X^{\top}AX-|\mathbf1^{\top}A\mathbf1|\ge t]\le2\exp\lp(-c\min\lp(\frac{t^2}{\|A\|_F^2+v},\frac{t}{\|A\|_{op}}\rp)\rp)\enspace,
    $$ 
    which is what we wished to show.
    Similarly, we can prove for the lower tail,
    $$
    \Pr_{J^*}[X^{\top}AX+|\mathbf1^{\top}A\mathbf1|\le -t]\le2\exp\lp(-c\min\lp(\frac{t^2}{\|A\|_F^2+v},\frac{t}{\|A\|_{op}}\rp)\rp).
    $$
    
   Now let $A=J-J^*$, so we can write that 
    $$\E_{J}[\|AX\|^2]=\E_{J^*}\lp[\frac{Z(J^*)\exp(\frac{1}{2}X^{\top}AX)}{Z(J)}\|AX\|^2\rp].$$

    We first lower bound $\frac{Z(J)}{Z(J^*)}$. 
    We have
    $$\frac{Z(J)}{Z(J^*)}=\E_{J^*}\lp[\exp(\frac{1}{2}X^{\top}AX)\rp].$$

    We note that, for a bounded random variable $Z$ and a differentiable increasing function $f(x)$ such that $f(0)=1$, if we want to lower bound $\E[e^Z]$, we can do the following: let $Y=\min(Z,0)$. Therefore, we have  
    \begin{align*}
        \E[f(Z)]\ge \E[f(Y)]&=\E [1-\int_Y^{0} f'(t)\mathrm dt]=1-\E[\int_{-\infty}^{0} \mathbb 1[t\ge Y]f'(t)\mathrm dt]\\
        &=1-\int_{-\infty}^{0}f'(t)\E[ \mathbb 1[t\ge Y]]\mathrm dt=1-\int_{-\infty}^{0}f'(t)\Pr[Z\le t]\mathrm dt.
    \end{align*}
    The second line is by Fubini. By the above inequality and the tail bound, when we view $Z$ to be $X^{\top}AX+u$ and $f(x)=e^{x/2}$, we can calculate that
    \begin{align*}
        &\E_{J^*}\lp[\exp\lp(\frac{1}{2}X^{\top}AX\rp)\rp]=e^{-u/2}\E_{J^*}\lp[\exp(\frac{1}{2}(X^{\top}AX+u))\rp]\\
        \ge& e^{-u/2}\cdot \lp(1-\frac{1}{2}\int_{0}^{\infty}e^{-t/2}\Pr_{J^*}[X^{\top}AX+u\le -t]\mathrm dt\rp)\\
        \ge&e^{-u/2}\cdot \left(1-\int_{0}^{\infty}\exp\lp(-\frac{t}{2}-c\min\lp(\frac{t^2}{\|A\|_F^2+v},\frac{t}{\|A\|_{op}}\rp)\rp)\mathrm dt\right).
    \end{align*}
    Since we know $\|A\|_{op}\le\|A\|_F\le v$, we have that
    \begin{align*}
        &\E_{J^*}[\exp(\frac{1}{2}X^{\top}AX)]=e^{-u/2}\E_{J^*}\lp[\exp\lp(\frac{1}{2}(X^{\top}AX+u)\rp)\rp]\\
        \ge&e^{-u/2}\cdot \left(1-\int_{0}^{\infty}\exp\lp(-\frac{t}{2}-c\min\lp(\frac{t^2}{\|A\|_F^2+v},\frac{t}{\|A\|_{op}}\rp)\rp)\mathrm dt\right)\\
        \ge& e^{-u/2}\cdot \left(1-\int_{0}^{\infty}\exp\lp(-\frac{t}{2}-c\min\lp(\frac{t^2}{2v},\frac{t}{\sqrt v}\rp)\rp)\mathrm dt\right)\\
        =&e^{-u/2}\cdot \left(1-\int_{0}^{2\sqrt{v}}\exp\lp(-\frac{t}{2}-c\frac{t^2}{2v}\rp)\mathrm dt-\int_{2\sqrt{v}}^{\infty}\exp\lp(-\frac{t}{2}-c\frac{t}{\sqrt v}\rp)\mathrm dt\right)\\
        \ge &e^{-u/2}\cdot \left(1-\int_{0}^{2\sqrt{v}}1\mathrm dt-\int_{0}^{\infty}\exp\lp(-\frac{t}{2}-c\frac{t}{\sqrt v}\rp)\mathrm dt\right)\\
        = &e^{-u/2}\cdot \left(1-2\sqrt{v}-\frac{1}{c/\sqrt v+1/2}\right).
    \end{align*}
    We choose the $\varepsilon$ small enough such that $\sqrt{v}=\sqrt{\frac{\varepsilon}{na(1-a)}}\le c/8$, so we have
    \begin{align*}
        &\E_{J^*}[\exp(X^{\top}AX)]
        \ge e^{-u/2}\cdot \left(1-\|A\|_F-\frac{1}{c/\|A\|_F+1/2}\right)\\
        \ge& e^{-u/2}\cdot \left(1-1/4-1/4\right)\ge e^{-\sqrt{\varepsilon/a}/2}/2.
    \end{align*}
    
    Finally, we upper bound the expectation as follows
    $$\E_{J^*}\lp[\exp\lp(\frac{1}{2}X^{\top}AX\rp)\|AX\|^2\rp]\le\E_{J^*}\lp[\exp\lp(\frac{\sqrt{n}}{2}\|AX\|\rp)\|AX\|^2\rp]$$

    Using the Hanson-Wright inequality as before, replacing $A$ with $A^2$, we have
    \begin{align*}
        &\Pr_{J^*}[\|AX\|^2-\max_{\mu}\E_{\mu}[|\|AX\|^2]\ge t]\le2\exp\lp(-c\min(\frac{t^2}{\|A^2\|_F^2+\max_\mu\E_{\mu}[\|A^2X\|]^2},\frac{t}{\|A^2\|_{op}})\rp)\\
        \le&2\exp\lp(-c\min\lp(\frac{t^2}{\|A\|_F^4+\|A\|_F^2\max_\mu\E_{\mu}[\|AX\|^2]},\frac{t}{\|A\|_{F}^2}\rp)\rp).
    \end{align*}
    
    Here, as before, $\max_{\mu}$ is taking $\max$ function for all i.i.d. measure in the decomposition of the Curie Weiss model \cite{deb2023fluctuations}.

    Similar to the proof of the lower bound from before, if $Z$ is a bounded random variable and $f$ is an increasing differentiable function such that $f(0)=0$, if we want to upper bound $\E[f(Z)]$, we can let $W=\max(Z,0)$ and have
    \begin{align*}
        \E[f(Z)]\le \E[f(W)]&=\E [\int_{0}^W f'(t)\mathrm dt]=\E[\int_0^{\infty} \mathbb 1[t\le W]f'(t)\mathrm dt]\\
        &=\int_{0}^{\infty}f'(t)\E[ \mathbb 1[W\ge t]]\mathrm dt=\int_{0}^{\infty}f'(t)\Pr[Z\ge t]\mathrm dt.
    \end{align*}
    Let $f(t)=te^{\sqrt{nt}/2}$, and $f'(t)=\exp(\frac{\sqrt{nt}}{2})(1+\frac{\sqrt{nt}}{2})$. The random variable $Z$ will be $\|AX\|^2$. We will now use the concentration result for $\|AX\|^2$ to upper bound the expectation. Below, we will use the simple observation that we made before $\max_\mu \E_{\mu}[\|AX\|^2] \leq v$.
    Using the identity $\exp(\frac{\sqrt{nt}}{2})(1+\frac{\sqrt{nt}}{2})\le \exp({\sqrt{nt}})$ in the rest of the calculation, we can calculate that
    \begin{align*}
        &\E_{J^*}\lp[\exp\lp(\frac{\sqrt{n}}{2}\|AX\|\rp)\|AX\|^2\rp]\\
        \le&\int_{0}^{\infty} \exp\lp(\frac{\sqrt{nt}}{2}\rp)\lp(1+\frac{\sqrt{nt}}{2}\rp)\Pr(\|AX\|^2\ge t)\mathrm dt\\
        \le&\int_{0}^{\infty} \exp\lp({\sqrt{nt}}\rp)\Pr(\|AX\|^2\ge t)\mathrm dt\\
        \le&\int_{0}^{2v+\|A\|_F^2} \exp({\sqrt{nt}})\mathrm dt+\int_{2v+\|A\|_F^2}^{\infty} \exp({\sqrt{nt}})\cdot 2\exp\lp(-c\min\lp(\frac{(t-v)^2}{\|A\|_F^4+\|A\|_F^2v},\frac{t-v}{\|A\|_F^2}\rp)\rp)\mathrm dt\\
        =&\int_{0}^{2v+\|A\|_F^2} \exp({\sqrt{nt}})\mathrm dt+\int_{2v+\|A\|_F^2}^{\infty} \exp({\sqrt{nt}})\cdot 2\exp\lp(-c\frac{t-v}{\|A\|_F^2}\rp)\mathrm dt
    \end{align*}
    Here, the first inequality is by taking $Z$ to be $\|AX\|^2$ and $f(x)=xe^{\sqrt{nx}}$. The second one is because of the inequality we mentioned, the third one is by the following: if $t\le 2v+\|A\|_F^2$, the probability of $\|AX\|^2\ge t$ is at most $1$. Otherwise, we use the concentration tail bound and get the upper bound. 
    
    We know that $nv=\varepsilon/(a(1-a))$, and $\|A\|_F^2\le \varepsilon/(n(1-a))\le v$, so we have
    \begin{align*}
        \int_{0}^{2v+\|A\|_F^2} \exp({\sqrt{nt}})\mathrm dt\le(2v+\|A\|_F^2)e^{\sqrt{n(2v+\|A\|_F^2)}}\le \frac{3\varepsilon}{a(1-a)n}\exp\lp(\frac{3\varepsilon}{a(1-a)}\rp).
    \end{align*}

    Also, if we take $\varepsilon<1-a$, we have
    \begin{align*}
        &\int_{2v+\|A\|_F^2}^{\infty} \exp\lp({\sqrt{nt}}\rp)\cdot 2\exp\lp(-c\frac{t-v}{\|A\|_F^2}\rp)\mathrm dt\le \int_{v+\|A\|_F^2}^{\infty} \exp({\sqrt{nt}})\cdot 2\exp\lp(-c\frac{t-v}{\|A\|_F^2}\rp)\mathrm dt\\
        \le&\int_{v+\|A\|_F^2}^{\infty} \exp(1/(4c)+cnt)\cdot 2\exp\lp(-c\frac{t-v}{\|A\|_F^2}\rp)\mathrm dt\\
        =&\exp\lp(\frac{1}{4c}+(v+\|A\|_F^2)\cdot c\cdot n-\|A\|_F^2\cdot\frac{c}{\|A\|_F^2}\rp)\cdot\frac{1/c}{\frac{1}{\|A\|_F^2}-n}\\
        \le&\exp\lp(\frac{1}{4c}+(v+\|A\|^2_F)cn-c\rp)\cdot\frac{1/c}{{1-a}-\varepsilon}\cdot\frac{\varepsilon}{n}\le\exp\lp(\frac{1}{4c}+\frac{\varepsilon \cdot (2-a)}{a(1-a)}-c\rp)\cdot\frac{1/c}{{1-a}-\varepsilon}\cdot\frac{\varepsilon}{n}.
    \end{align*}

    Therefore, summing up all the things above, we have
    $$\E_{J}[\|AX\|^2]\le\frac{\varepsilon}{n}\left(\frac{3}{a(1-a)}\exp\lp(\frac{3\varepsilon}{a(1-a)}\rp)+\exp\lp(\frac{1}{4c}+\frac{\varepsilon \cdot (2-a)}{a(1-a)}-c\rp)\cdot\frac{1/c}{{1-a}-\varepsilon}\right)\cdot 2e^{\sqrt{\varepsilon/a}/2}.$$

    Since $a$ is a constant depending on $\beta$, and we can take $\varepsilon$ as a sufficiently small constant depending on $\beta$, the result follows.
\end{proof}

\bibliography{bib}

@inproceedings{bhattacharyya2021near,
  title={Near-optimal learning of tree-structured distributions by Chow-Liu},
  author={Bhattacharyya, Arnab and Gayen, Sutanu and Price, Eric and Vinodchandran, NV},
  booktitle={Proceedings of the 53rd annual acm SIGACT symposium on theory of computing},
  year={2021}
}

@article{bresler2020learning,
  title={Learning a tree-structured Ising model in order to make predictions},
  author={Bresler, Guy and Karzand, Mina},
  journal={The Annals of Statistics},
  volume={48},
  number={2},
  pages={713--737},
  year={2020},
  publisher={JSTOR}
}

@inproceedings{daskalakis2021sample,
  title={Sample-optimal and efficient learning of tree Ising models},
  author={Daskalakis, Constantinos and Pan, Qinxuan},
  booktitle={Proceedings of the 53rd Annual ACM SIGACT Symposium on Theory of Computing},
  pages={133--146},
  year={2021}
}

@inproceedings{caputo2015approximate,
  title={Approximate tensorization of entropy at high temperature},
  author={Caputo, Pietro and Menz, Georg and Tetali, Prasad},
  booktitle={Annales de la Facult{\'e} des sciences de Toulouse: Math{\'e}matiques},
  volume={24},
  number={4},
  pages={691--716},
  year={2015}
}

@inproceedings{chen2021optimal,
  title={Optimal mixing of Glauber dynamics: Entropy factorization via high-dimensional expansion},
  author={Chen, Zongchen and Liu, Kuikui and Vigoda, Eric},
  booktitle={Proceedings of the 53rd Annual ACM SIGACT Symposium on Theory of Computing},
  pages={1537--1550},
  year={2021}
}

@article{caputo2023lecture,
  title={Lecture notes on entropy and Markov chains},
  author={Caputo, Pietro},
  journal={Preprint, available from: http://www. mat. uniroma3. it/users/caputo/entropy. pdf},
  year={2023}
}

@inproceedings{blanca2022mixing,
  title={On mixing of markov chains: Coupling, spectral independence, and entropy factorization},
  author={Blanca, Antonio and Caputo, Pietro and Chen, Zongchen and Parisi, Daniel and {\v{S}}tefankovi{\v{c}}, Daniel and Vigoda, Eric},
  booktitle={Proceedings of the 2022 Annual ACM-SIAM Symposium on Discrete Algorithms (SODA)},
  pages={3670--3692},
  year={2022},
  organization={SIAM}
}

@article{montanari2010spread,
  title={The spread of innovations in social networks},
  author={Montanari, Andrea and Saberi, Amin},
  journal={Proceedings of the National Academy of Sciences},
  volume={107},
  number={47},
  pages={20196--20201},
  year={2010},
  publisher={National Academy of Sciences}
}

@article{ellison1993learning,
  title={Learning, local interaction, and coordination},
  author={Ellison, Glenn},
  journal={Econometrica: Journal of the Econometric Society},
  pages={1047--1071},
  year={1993},
  publisher={JSTOR}
}

@inproceedings{daskalakis2006optimal,
  title={Optimal phylogenetic reconstruction},
  author={Daskalakis, Constantinos and Mossel, Elchanan and Roch, S{\'e}bastien},
  booktitle={Proceedings of the thirty-eighth annual ACM symposium on Theory of computing},
  pages={159--168},
  year={2006}
}

@inproceedings{geman1986markov,
  title={Markov random field image models and their applications to computer vision},
  author={Geman, Stuart and Graffigne, Christine},
  booktitle={Proceedings of the international congress of mathematicians},
  volume={1},
  pages={2},
  year={1986},
  organization={Berkeley, CA}
}

@book{felsenstein2004,
title={Inferring phylogenies},
author={Felsenstein, Joseph},
publisher={Sunderland; Sinauer Associates},
year=2004
}

@inproceedings{dagan2021learning,
  title={Learning Ising models from one or multiple samples},
  author={Dagan, Yuval and Daskalakis, Constantinos and Dikkala, Nishanth and Kandiros, Anthimos Vardis},
  booktitle={Proceedings of the 53rd Annual ACM SIGACT Symposium on Theory of Computing},
  pages={161--168},
  year={2021}
}

@inproceedings{chen2022localization,
  title={Localization schemes: A framework for proving mixing bounds for Markov chains},
  author={Chen, Yuansi and Eldan, Ronen},
  booktitle={2022 IEEE 63rd Annual Symposium on Foundations of Computer Science (FOCS)},
  pages={110--122},
  year={2022},
  organization={IEEE}
}

@article{anari2021entropic,
  title={Entropic independence I: Modified log-Sobolev inequalities for fractionally log-concave distributions and high-temperature ising models},
  author={Anari, Nima and Jain, Vishesh and Koehler, Frederic and Pham, Huy Tuan and Vuong, Thuy-Duong},
  journal={arXiv preprint arXiv:2106.04105},
  year={2021}
}

@article{adamczak2019note,
  title={A note on concentration for polynomials in the Ising model},
  author={Adamczak, Rados{\l}aw and Kotowski, Micha{\l} and Polaczyk, Bart{\l}omiej and Strzelecki, Micha{\l}},
  year={2019}
}

@article{gotze2021concentration,
  title={Concentration inequalities for bounded functionals via log-Sobolev-type inequalities},
  author={G{\"o}tze, Friedrich and Sambale, Holger and Sinulis, Arthur},
  journal={Journal of Theoretical Probability},
  volume={34},
  pages={1623--1652},
  year={2021},
  publisher={Springer}
}

@article{marton2015logarithmic,
  title={Logarithmic Sobolev inequalities in discrete product spaces: a proof by a transportation cost distance},
  author={Marton, Katalin},
  journal={arXiv preprint arXiv:1507.02803},
  year={2015}
}

@inproceedings{anari2024universality,
  title={Universality of spectral independence with applications to fast mixing in spin glasses},
  author={Anari, Nima and Jain, Vishesh and Koehler, Frederic and Pham, Huy Tuan and Vuong, Thuy-Duong},
  booktitle={Proceedings of the 2024 Annual ACM-SIAM Symposium on Discrete Algorithms (SODA)},
  pages={5029--5056},
  year={2024},
  organization={SIAM}
}

@article{sambale2019modified,
  title={Modified log-Sobolev inequalities and two-level concentration},
  author={Sambale, Holger and Sinulis, Arthur},
  journal={arXiv preprint arXiv:1905.06137},
  year={2019}
}

@inproceedings{anari2024trickle,
  title={Trickle-Down in Localization Schemes and Applications},
  author={Anari, Nima and Koehler, Frederic and Vuong, Thuy-Duong},
  booktitle={Proceedings of the 56th Annual ACM Symposium on Theory of Computing},
  pages={1094--1105},
  year={2024}
}

@inproceedings{cryan2019modified,
  title={Modified log-Sobolev inequalities for strongly log-concave distributions},
  author={Cryan, Mary and Guo, Heng and Mousa, Giorgos},
  booktitle={2019 IEEE 60th Annual Symposium on Foundations of Computer Science (FOCS)},
  pages={1358--1370},
  year={2019},
  organization={IEEE}
}

@article{eldan2022spectral,
  title={A spectral condition for spectral gap: fast mixing in high-temperature Ising models},
  author={Eldan, Ronen and Koehler, Frederic and Zeitouni, Ofer},
  journal={Probability theory and related fields},
  volume={182},
  number={3},
  pages={1035--1051},
  year={2022},
  publisher={Springer}
}

@article{bobkov1999exponential,
  title={Exponential integrability and transportation cost related to logarithmic Sobolev inequalities},
  author={Bobkov, Sergej G and G{\"o}tze, Friedrich},
  journal={Journal of Functional Analysis},
  volume={163},
  number={1},
  pages={1--28},
  year={1999},
  publisher={Elsevier}
}

@book{artstein2021asymptotic,
  title={Asymptotic geometric analysis, Part II},
  author={Artstein-Avidan, Shiri and Giannopoulos, Apostolos and Milman, Vitali D},
  volume={261},
  year={2021},
  publisher={American Mathematical Society}
}

@inproceedings{klivans2017learning,
  title={Learning graphical models using multiplicative weights},
  author={Klivans, Adam and Meka, Raghu},
  booktitle={2017 IEEE 58th Annual Symposium on Foundations of Computer Science (FOCS)},
  pages={343--354},
  year={2017},
  organization={IEEE}
}

@inproceedings{gaitonde2024unified,
  title={A Unified Approach to Learning Ising Models: Beyond Independence and Bounded Width},
  author={Gaitonde, Jason and Mossel, Elchanan},
  booktitle={Proceedings of the 56th Annual ACM Symposium on Theory of Computing},
  pages={503--514},
  year={2024}
}

@article{wu2019sparse,
  title={Sparse logistic regression learns all discrete pairwise graphical models},
  author={Wu, Shanshan and Sanghavi, Sujay and Dimakis, Alexandros G},
  journal={Advances in Neural Information Processing Systems},
  volume={32},
  year={2019}
}

@article{chandrasekaran2024learning,
  title={Learning the Sherrington-Kirkpatrick Model Even at Low Temperature},
  author={Chandrasekaran, Gautam and Klivans, Adam},
  journal={arXiv preprint arXiv:2411.11174},
  year={2024}
}

@article{devroye2020minimax,
  title={The minimax learning rates of normal and Ising undirected graphical models},
  author={Devroye, Luc and Mehrabian, Abbas and Reddad, Tommy},
  year={2020}
}

@article{tropp2015introduction,
  title={An introduction to matrix concentration inequalities},
  author={Tropp, Joel A and others},
  journal={Foundations and Trends{\textregistered} in Machine Learning},
  volume={8},
  number={1-2},
  pages={1--230},
  year={2015},
  publisher={Now Publishers, Inc.}
}

@inproceedings{sly2012computational,
  title={The computational hardness of counting in two-spin models on d-regular graphs},
  author={Sly, Allan and Sun, Nike},
  booktitle={2012 IEEE 53rd Annual Symposium on Foundations of Computer Science},
  pages={361--369},
  year={2012},
  organization={IEEE}
}

@inproceedings{sly2010computational,
  title={Computational transition at the uniqueness threshold},
  author={Sly, Allan},
  booktitle={2010 IEEE 51st Annual Symposium on Foundations of Computer Science},
  pages={287--296},
  year={2010},
  organization={IEEE}
}

@article{galanis2015inapproximability,
  title={Inapproximability for antiferromagnetic spin systems in the tree nonuniqueness region},
  author={Galanis, Andreas and {\v{S}}tefankovi{\v{c}}, Daniel and Vigoda, Eric},
  journal={Journal of the ACM (JACM)},
  volume={62},
  number={6},
  pages={1--60},
  year={2015},
  publisher={ACM New York, NY, USA}
}

@article{galanis2024sampling,
  title={On sampling from Ising models with spectral constraints},
  author={Galanis, Andreas and Kalavasis, Alkis and Kandiros, Anthimos Vardis},
  journal={arXiv preprint arXiv:2407.07645},
  year={2024}
}

@article{besag1974spatial,
  title={Spatial interaction and the statistical analysis of lattice systems},
  author={Besag, Julian},
  journal={Journal of the Royal Statistical Society: Series B (Methodological)},
  volume={36},
  number={2},
  pages={192--225},
  year={1974},
  publisher={Wiley Online Library}
}

@article{dobruschin1968description,
  title={The description of a random field by means of conditional probabilities and conditions of its regularity},
  author={Dobruschin, PL},
  journal={Theory of Probability \& Its Applications},
  volume={13},
  number={2},
  pages={197--224},
  year={1968},
  publisher={SIAM}
}

@book{levin2017markov,
  title={Markov chains and mixing times},
  author={Levin, David A and Peres, Yuval},
  volume={107},
  year={2017},
  publisher={American Mathematical Soc.}
}

@article{lenz1920beitrvsge,
  title={Beitr{\v{s}}ge zum verst{\v{s}}ndnis der magnetischen eigenschaften in festen k{\v{s}}rpern},
  author={Lenz, Wilhelm},
  journal={Physikalische Z},
  volume={21},
  number={613-615},
  pages={1},
  year={1920}
}

@inproceedings{bresler2015efficiently,
  title={Efficiently learning Ising models on arbitrary graphs},
  author={Bresler, Guy},
  booktitle={Proceedings of the forty-seventh annual ACM symposium on Theory of computing},
  pages={771--782},
  year={2015}
}

@article{hamilton2017information,
  title={Information theoretic properties of Markov random fields, and their algorithmic applications},
  author={Hamilton, Linus and Koehler, Frederic and Moitra, Ankur},
  journal={Advances in Neural Information Processing Systems},
  volume={30},
  year={2017}
}

@article{vuffray2016interaction,
  title={Interaction screening: Efficient and sample-optimal learning of Ising models},
  author={Vuffray, Marc and Misra, Sidhant and Lokhov, Andrey and Chertkov, Michael},
  journal={Advances in neural information processing systems},
  volume={29},
  year={2016}
}

@article{lokhov2018optimal,
  title={Optimal structure and parameter learning of Ising models},
  author={Lokhov, Andrey Y and Vuffray, Marc and Misra, Sidhant and Chertkov, Michael},
  journal={Science advances},
  volume={4},
  number={3},
  pages={e1700791},
  year={2018},
  publisher={American Association for the Advancement of Science}
}

@article{santhanam2012information,
  title={Information-theoretic limits of selecting binary graphical models in high dimensions},
  author={Santhanam, Narayana P and Wainwright, Martin J},
  journal={IEEE Transactions on Information Theory},
  volume={58},
  number={7},
  pages={4117--4134},
  year={2012},
  publisher={IEEE}
}

@inproceedings{brustle2020multi,
  title={Multi-item mechanisms without item-independence: Learnability via robustness},
  author={Brustle, Johannes and Cai, Yang and Daskalakis, Constantinos},
  booktitle={Proceedings of the 21st ACM Conference on Economics and Computation},
  pages={715--761},
  year={2020}
}

@article{bhattacharya2018inference,
  title={Inference in Ising models},
  author={Bhattacharya, Bhaswar B and Mukherjee, Sumit},
  year={2018}
}

@article{lee2023parallelising,
  title={Parallelising glauber dynamics},
  author={Lee, Holden},
  journal={arXiv preprint arXiv:2307.07131},
  year={2023}
}

@inproceedings{daskalakis2019regression,
  title={Regression from dependent observations},
  author={Daskalakis, Constantinos and Dikkala, Nishanth and Panageas, Ioannis},
  booktitle={Proceedings of the 51st Annual ACM SIGACT Symposium on Theory of Computing},
  pages={881--889},
  year={2019}
}

@article{chatterjee2007estimation,
  title={Estimation in spin glasses: A first step},
  author={Chatterjee, Sourav},
  year={2007}
}

@article{mossel2009hardness,
  title={On the hardness of sampling independent sets beyond the tree threshold},
  author={Mossel, Elchanan and Weitz, Dror and Wormald, Nicholas},
  journal={Probability Theory and Related Fields},
  volume={143},
  number={3},
  pages={401--439},
  year={2009},
  publisher={Springer}
}

@article{chow1968approximating,
  title={Approximating discrete probability distributions with dependence trees},
  author={Chow, CKCN and Liu, Cong},
  journal={IEEE transactions on Information Theory},
  volume={14},
  number={3},
  pages={462--467},
  year={1968},
  publisher={IEEE}
}

@article{bresler2017learning,
  title={Learning graphical models from the Glauber dynamics},
  author={Bresler, Guy and Gamarnik, David and Shah, Devavrat},
  journal={IEEE Transactions on Information Theory},
  volume={64},
  number={6},
  pages={4072--4080},
  year={2017},
  publisher={IEEE}
}

@article{gaitonde2024bypassing,
  title={Bypassing the Noisy Parity Barrier: Learning Higher-Order Markov Random Fields from Dynamics},
  author={Gaitonde, Jason and Moitra, Ankur and Mossel, Elchanan},
  journal={arXiv preprint arXiv:2409.05284},
  year={2024}
}

@inproceedings{boix2022chow,
  title={Chow-liu++: Optimal prediction-centric learning of tree ising models},
  author={Boix-Adsera, Enric and Bresler, Guy and Koehler, Frederic},
  booktitle={2021 IEEE 62nd Annual Symposium on Foundations of Computer Science (FOCS)},
  pages={417--426},
  year={2022},
  organization={IEEE}
}

@inproceedings{kandiros2023learning,
  title={Learning and testing latent-tree ising models efficiently},
  author={Kandiros, Vardis and Daskalakis, Constantinos and Dagan, Yuval and Choo, Davin},
  booktitle={The Thirty Sixth Annual Conference on Learning Theory},
  pages={1666--1729},
  year={2023},
  organization={PMLR}
}

@inproceedings{bresler2019learning,
  title={Learning restricted Boltzmann machines via influence maximization},
  author={Bresler, Guy and Koehler, Frederic and Moitra, Ankur},
  booktitle={Proceedings of the 51st Annual ACM SIGACT Symposium on Theory of Computing},
  pages={828--839},
  year={2019}
}

@article{koehler2022statistical,
  title={Statistical efficiency of score matching: The view from isoperimetry},
  author={Koehler, Frederic and Heckett, Alexander and Risteski, Andrej},
  journal={arXiv preprint arXiv:2210.00726},
  year={2022}
}

@article{koehler2023sampling,
  title={Sampling multimodal distributions with the vanilla score: Benefits of data-based initialization},
  author={Koehler, Frederic and Vuong, Thuy-Duong},
  journal={arXiv preprint arXiv:2310.01762},
  year={2023}
}

@inproceedings{koehler2022sampling,
  title={Sampling approximately low-rank Ising models: MCMC meets variational methods},
  author={Koehler, Frederic and Lee, Holden and Risteski, Andrej},
  booktitle={Conference on Learning Theory},
  pages={4945--4988},
  year={2022},
  organization={PMLR}
}

@inproceedings{friedman2003proof,
  title={A proof of Alon's second eigenvalue conjecture},
  author={Friedman, Joel},
  booktitle={Proceedings of the thirty-fifth annual ACM symposium on Theory of computing},
  pages={720--724},
  year={2003}
}

@book{talagrand2010mean,
  title={Mean field models for spin glasses: Volume I: Basic examples},
  author={Talagrand, Michel},
  volume={54},
  year={2010},
  publisher={Springer Science \& Business Media}
}

@article{gaitonde2025better,
  title={Better Models and Algorithms for Learning Ising Models from Dynamics},
  author={Gaitonde, Jason and Moitra, Ankur and Mossel, Elchanan},
  journal={arXiv preprint arXiv:2507.15173},
  year={2025}
}

@article{panchenko2012sherrington,
  title={The Sherrington-Kirkpatrick model: an overview},
  author={Panchenko, Dmitry},
  journal={Journal of Statistical Physics},
  volume={149},
  number={2},
  pages={362--383},
  year={2012},
  publisher={Springer}
}

@article{koehler2024efficiently,
  title={Efficiently learning and sampling multimodal distributions with data-based initialization},
  author={Koehler, Frederic and Lee, Holden and Vuong, Thuy-Duong},
  journal={arXiv preprint arXiv:2411.09117},
  year={2024}
}

@article{ghosal2020joint,
  title={Joint estimation of parameters in Ising model},
  author={Ghosal, Promit and Mukherjee, Sumit},
  year={2020}
}

@inproceedings{kandiros2021statistical,
  title={Statistical estimation from dependent data},
  author={Kandiros, Vardis and Dagan, Yuval and Dikkala, Nishanth and Goel, Surbhi and Daskalakis, Constantinos},
  booktitle={International Conference on Machine Learning},
  pages={5269--5278},
  year={2021},
  organization={PMLR}
}

@article{mukherjee2021high,
  title={High dimensional logistic regression under network dependence},
  author={Mukherjee, Somabha and Niu, Ziang and Halder, Sagnik and Bhattacharya, Bhaswar B and Michailidis, George},
  journal={arXiv preprint arXiv:2110.03200},
  year={2021}
}

@article{mukherjee2022estimation,
  title={Estimation in tensor Ising models},
  author={Mukherjee, Somabha and Son, Jaesung and Bhattacharya, Bhaswar B},
  journal={Information and Inference: A Journal of the IMA},
  volume={11},
  number={4},
  pages={1457--1500},
  year={2022},
  publisher={Oxford University Press}
}

@inproceedings{daskalakis2020logistic,
  title={Logistic regression with peer-group effects via inference in higher-order Ising models},
  author={Daskalakis, Constantinos and Dikkala, Nishanth and Panageas, Ioannis},
  booktitle={International Conference on Artificial Intelligence and Statistics},
  pages={3653--3663},
  year={2020},
  organization={PMLR}
}

@inproceedings{galanis2024learning,
  title={Learning Hard-Constrained Models with One Sample},
  author={Galanis, Andreas and Kalavasis, Alkis and Kandiros, Anthimos Vardis},
  booktitle={Proceedings of the 2024 Annual ACM-SIAM Symposium on Discrete Algorithms (SODA)},
  pages={3184--3196},
  year={2024},
  organization={SIAM}
}

@article{bhattacharya2021parameter,
  title={Parameter estimation for undirected graphical models with hard constraints},
  author={Bhattacharya, Bhaswar B and Ramanan, Kavita},
  journal={IEEE Transactions on Information Theory},
  volume={67},
  number={10},
  pages={6790--6809},
  year={2021},
  publisher={IEEE}
}

@article{xu2023inference,
  title={Inference in Ising models on dense regular graphs},
  author={Xu, Yuanzhe and Mukherjee, Sumit},
  journal={The Annals of Statistics},
  volume={51},
  number={3},
  pages={1183--1206},
  year={2023},
  publisher={Institute of Mathematical Statistics}
}

@article{bauerschmidt2019very,
  title={A very simple proof of the LSI for high temperature spin systems},
  author={Bauerschmidt, Roland and Bodineau, Thierry},
  journal={Journal of Functional Analysis},
  volume={276},
  number={8},
  pages={2582--2588},
  year={2019},
  publisher={Elsevier}
}

@inproceedings{el2022sampling,
  title={Sampling from the Sherrington-Kirkpatrick Gibbs measure via algorithmic stochastic localization},
  author={El Alaoui, Ahmed and Montanari, Andrea and Sellke, Mark},
  booktitle={2022 IEEE 63rd Annual Symposium on Foundations of Computer Science (FOCS)},
  pages={323--334},
  year={2022},
  organization={IEEE}
}

@article{celentano2024sudakov,
  title={Sudakov--Fernique post-AMP, and a new proof of the local convexity of the TAP free energy},
  author={Celentano, Michael},
  journal={The Annals of Probability},
  volume={52},
  number={3},
  pages={923--954},
  year={2024},
  publisher={Institute of Mathematical Statistics}
}

@inproceedings{kunisky2024optimality,
  title={Optimality of Glauber dynamics for general-purpose Ising model sampling and free energy approximation},
  author={Kunisky, Dmitriy},
  booktitle={Proceedings of the 2024 Annual ACM-SIAM Symposium on Discrete Algorithms (SODA)},
  pages={5013--5028},
  year={2024},
  organization={SIAM}
}

@article{jayakumar2024discrete,
  title={Discrete distributions are learnable from metastable samples},
  author={Jayakumar, Abhijith and Lokhov, Andrey Y and Misra, Sidhant and Vuffray, Marc},
  journal={arXiv preprint arXiv:2410.13800},
  year={2024}
}

@article{bubeck2015convex,
  title={Convex optimization: Algorithms and complexity},
  author={Bubeck, S{\'e}bastien and others},
  journal={Foundations and Trends{\textregistered} in Machine Learning},
  volume={8},
  number={3-4},
  pages={231--357},
  year={2015},
  publisher={Now Publishers, Inc.}
}

@article{kunsch1982decay,
  title={Decay of correlations under Dobrushin's uniqueness condition and its applications},
  author={K{\"u}nsch, H},
  journal={Communications in Mathematical Physics},
  volume={84},
  pages={207--222},
  year={1982},
  publisher={Springer}
}

@book{koller2009probabilistic,
  title={Probabilistic graphical models: principles and techniques},
  author={Koller, Daphne and Friedman, Nir},
  year={2009},
  publisher={MIT press}
}

@article{wainwright2008graphical,
  title={Graphical models, exponential families, and variational inference},
  author={Wainwright, Martin J and Jordan, Michael I and others},
  journal={Foundations and Trends{\textregistered} in Machine Learning},
  volume={1},
  number={1--2},
  pages={1--305},
  year={2008},
  publisher={Now Publishers, Inc.}
}

@article{nikolakakis2021predictive,
  title={Predictive learning on hidden tree-structured Ising models},
  author={Nikolakakis, Konstantinos E and Kalogerias, Dionysios S and Sarwate, Anand D},
  journal={Journal of Machine Learning Research},
  volume={22},
  number={59},
  pages={1--82},
  year={2021}
}

@inproceedings{zhang2020privately,
  title={Privately learning Markov random fields},
  author={Zhang, Huanyu and Kamath, Gautam and Kulkarni, Janardhan and Wu, Steven},
  booktitle={International conference on machine learning},
  pages={11129--11140},
  year={2020},
  organization={PMLR}
}

@article{van2014probability,
  title={Probability in high dimension},
  author={Van Handel, Ramon},
  journal={Lecture Notes (Princeton University)},
  volume={2},
  number={3},
  pages={2--3},
  year={2014}
}

@book{ellis2007entropy,
  title={Entropy, large deviations, and statistical mechanics},
  author={Ellis, Richard S},
  year={2007},
  publisher={Springer}
}

@inproceedings{liu2024fast,
  title={Fast mixing in sparse random Ising models},
  author={Liu, Kuikui and Mohanty, Sidhanth and Rajaraman, Amit and Wu, David X},
  booktitle={2024 IEEE 65th Annual Symposium on Foundations of Computer Science (FOCS)},
  pages={120--128},
  year={2024},
  organization={IEEE}
}

@article{bobkov2006modified,
  title={Modified logarithmic Sobolev inequalities in discrete settings},
  author={Bobkov, Sergey G and Tetali, Prasad},
  journal={Journal of Theoretical Probability},
  volume={19},
  number={2},
  pages={289--336},
  year={2006},
  publisher={Springer}
}

@book{boyd2004convex,
  title={Convex optimization},
  author={Boyd, Stephen and Vandenberghe, Lieven},
  year={2004},
  publisher={Cambridge university press}
}

@inproceedings{lee2018efficient,
  title={Efficient convex optimization with membership oracles},
  author={Lee, Yin Tat and Sidford, Aaron and Vempala, Santosh S},
  booktitle={Conference On Learning Theory},
  pages={1292--1294},
  year={2018},
  organization={PMLR}
}

@inproceedings{jiang2020faster,
  title={A faster interior point method for semidefinite programming},
  author={Jiang, Haotian and Kathuria, Tarun and Lee, Yin Tat and Padmanabhan, Swati and Song, Zhao},
  booktitle={2020 IEEE 61st annual symposium on foundations of computer science (FOCS)},
  pages={910--918},
  year={2020},
  organization={IEEE}
}

@article{deb2023fluctuations,
  title={Fluctuations in mean-field Ising models},
  author={Deb, Nabarun and Mukherjee, Sumit},
  journal={The Annals of Applied Probability},
  volume={33},
  number={3},
  pages={1961--2003},
  year={2023},
  publisher={Institute of Mathematical Statistics}
}

\end{document}